%% file: main-fewshotdm.tex
\documentclass[sigconf,authorversion]{acmart}
\input{notation}

\newif\ifextversion
\extversiontrue
\AtBeginDocument{%
  }

\copyrightyear{2026}
\acmYear{2026}
\setcopyright{cc}
\setcctype{by-nc-nd}
\acmConference[KDD '26]{Proceedings of the 32nd ACM SIGKDD Conference on Knowledge Discovery and Data Mining V.2}{August 09--13, 2026}{Jeju Island, Republic of Korea}
\acmBooktitle{Proceedings of the 32nd ACM SIGKDD Conference on Knowledge Discovery and Data Mining V.2 (KDD '26), August 09--13, 2026, Jeju Island, Republic of Korea}
\acmDOI{10.1145/3770855.3817752}
\acmISBN{979-8-4007-2259-2/2026/08}
\begin{document}

%%
%% The "title" command has an optional parameter,
%% allowing the author to define a "short title" to be used in page headers.
\title[Few-Shot Resampling for Scalable Statistically-Sound Data Mining]{Few-Shot~Resampling~for~Scalable~Statistically-Sound~Data~Mining}

%%
%% The "author" command and its associated commands are used to define
%% the authors and their affiliations.
%% Of note is the shared affiliation of the first two authors, and the
%% "authornote" and "authornotemark" commands
%% used to denote shared contribution to the research.
\author{Leonardo Pellegrina}
\affiliation{%
  \institution{Department of Information Engineering, University~of~Padova}
  \streetaddress{Via Gradenigo 6b}
  \city{Padova}
  \country{Italy}
  \postcode{35129}
}
\email{leonardo.pellegrina@unipd.it}

\author{Fabio Vandin}
\affiliation{%
  \institution{Department of Information Engineering, University~of~Padova}
  \streetaddress{Via Gradenigo 6b}
  \city{Padova}
  \country{Italy}
  \postcode{35129}
}
\email{fabio.vandin@unipd.it}

%%
%% By default, the full list of authors will be used in the page
%% headers. Often, this list is too long, and will overlap
%% other information printed in the page headers. This command allows
%% the author to define a more concise list
%% of authors' names for this purpose.
\renewcommand{\shortauthors}{Leonardo Pellegrina and Fabio Vandin}

%%
%% The abstract is a short summary of the work to be presented in the
%% article.
\begin{abstract}
A key step in knowledge discovery is the evaluation of data mining results. In several applications, including pattern mining, graph analysis, and others, this step includes the evaluation of the statistical significance of the results, to avoid spurious discoveries due only to noise or random fluctuations in the data. While specialized procedures have been developed for some specific applications, resampling-based approaches are widely used, in particular for complex analyses where analytical results cannot be derived. However, current resampling-based approaches require the generation and analysis of thousands of resampled datasets, and are therefore impractical for large datasets or computationally intensive analyses.

In this paper, we introduce \algname, a simple and effective resam\-pling-based approach to assess the statistical significance of data mining results with rigorous guarantees on the probability of false discoveries. Our approach can be used in every situation where resampling-based approaches are applied. \algname\ builds on our derivation of a novel bound to the supremum deviation of test statistics representing the quality of data mining results. We prove that \algname\ needs to generate and analyze an extremely small number of resampled datasets, leading to a highly scalable approach with wide applicability. We test our approach on  common tasks such as pattern mining and network analysis. In all cases, our approach results in a reduction of up to two orders of magnitude in running time compared to the state of the art, 
while preserving high statistical power, 
%enabling the analysis of large-scale real-world datasets.
enabling the statistical validation of data mining results on large-scale real-world datasets.
\end{abstract}

%%
%% The code below is generated by the tool at http://dl.acm.org/ccs.cfm.
%% Please copy and paste the code instead of the example below.

%%
%% Keywords. The author(s) should pick words that accurately describe
%% the work being presented. Separate the keywords with commas.
%% A "teaser" image appears between the author and affiliation
%% information and the body of the document, and typically spans the
%% page.
%\begin{teaserfigure}
%  \includegraphics[width=\textwidth]{sampleteaser}
%  \caption{Seattle Mariners at Spring Training, 2010.}
%  \Description{Enjoying the baseball game from the third-base
%  seats. Ichiro Suzuki preparing to bat.}
%  \label{fig:teaser}
%\end{teaserfigure}

%\received{20 February 2007}
%\received[revised]{12 March 2009}
%\received[accepted]{5 June 2009}

%%
%% This command processes the author and affiliation and title
%% information and builds the first part of the formatted document.
\maketitle

\input{fewshotdm}

\appendix
\input{appendix}

\end{document}

%% file: notation.tex
\usepackage{graphicx}
\usepackage{bm}
\usepackage{amsmath}
\usepackage{amsfonts}
\usepackage{amsthm}
\usepackage{dsfont}
\usepackage{subcaption}
\usepackage{cleveref}
\usepackage{thmtools, thm-restate}
\usepackage[ruled,vlined,linesnumbered]{algorithm2e}
\usepackage[normalem]{ulem}
\usepackage{enumitem}
\usepackage{booktabs}

\definecolor{darkred}{RGB}{200,0,0}

\newcommand{\algname}{\textsc{FewRS}}

\newcommand{\pars}[1]{\left( #1 \right)}

\newcommand{\brpars}[1]{\left\lbrace #1 \right\rbrace}
\newcommand{\qt}[1]{\lq\lq#1\rq\rq}

\newcommand{\ind}[1]{\mathds{1} \left[ #1 \right] }

\let\L\relax
\newcommand{\L}{\mathcal{L}}
\newcommand{\R}{\mathbb{R}}
\newcommand{\dataset}{\mathcal{D}}

\newcommand{\tnulls}{\mathcal{H}_0}
\newcommand{\sample}{\mathcal{S}}
\newcommand{\A}{\mathcal{A}}
\newcommand{\lang}{\mathcal{L}}

\newcommand{\BT}[1]{\Theta\pars{ #1 }}

\newcommand{\E}{\mathop{\mathbb{E}}}
\newcommand{\Var}{\text{Var}}

\newcommand{\X}{\mathcal{X}}

\DeclareMathOperator*{\argmin}{arg\,min}

\newcommand{\probdist}{\gamma}

%% file: fewshotdm.tex
% !TEX root = main-fewshotdm.tex

\section{Introduction}
\label{sec:intro}

Knowledge discovery is a complex process consisting of several steps whose goal is to extract useful information from complex and large datasets. 
A crucial component of this process is the mining step~\cite{han2022data,leskovec2020mining}, where the dataset is analyzed to extract patterns or other characteristics (e.g., clusters) of the data which provide actionable information to the end user or inform subsequent analyses. 
A key issue in the mining step is to avoid uninteresting or spurious results~\cite{tan2016introduction} that do not provide useful knowledge and are instead due to factors such noise or random fluctuations of the data.

A common approach to avoid spurious discoveries is the evaluation of the statistical significance of the data mining results within the statistical hypothesis testing framework. In this framework, the data mining analyses are summarized by one or multiple \emph{statistics}, i.e., scores quantitatively capturing the aspects of interest. 
The analyst then formulates a \emph{null hypothesis} on the underlying process that generated the data, corresponding to the situation in which results are not interesting. 
The data mining results are evaluated by comparing the statistics with their distribution under the null hypothesis (i.e., when the null hypothesis holds): results with large deviations from the null distribution are unlikely to be due to random fluctuations, and are therefore flagged as \emph{significant}; results that are likely under the null hypothesis are instead discarded.

\textbf{Example: market basket analysis.} A common task in market basket analysis is to mine a set of transactions, each corresponding to the items bought by a customer, with the goal of identifying sets of  items, or itemsets, which are frequently bought together. 
To avoid uninteresting results, we want to discard itemsets whose selling frequency is explained by the frequency of the single items. 
Therefore, we define the \emph{null hypothesis} in which single items are placed into transactions independently with probability equal to their frequency in the dataset. 
(Note that under the null hypothesis no itemset is interesting, since items are placed independently in transactions.) 
As test statistic, we choose the number of itemsets of size $c$ that appear in a fraction at least $\theta$ of the transactions of the dataset. 
Assume that the value of the statistic is $s$. 
The significance of the results is then given by the probability that, under the null hypothesis, $\ge s$  itemsets of size $c$ appear in a fraction at least $\theta$ of the transactions. If such probability, called $p$-value, is small, then the itemsets are significant, otherwise they are not.

While the statistical hypothesis testing framework cannot certify the \emph{validity} of data mining outcomes, it provides formal guarantees on the avoidance of spurious discoveries, thus reducing the risk of investing resources on invalid or not meaningful results. 
More in detail, for a given statistic $A(\mathcal{D})$ computed by analyzing the dataset $\mathcal{D}$, let $p_A$ be its \emph{$p$-value}, that is, the probability that the statistic $A(\mathcal{D}^0)$ computed by analyzing a dataset $\mathcal{D}^0$ generated under the null hypothesis is \emph{more extreme}
(e.g., higher) 
than $A(\mathcal{D})$. For a given \emph{significance threshold} $\alpha \in (0,1]$, flagging the results as significant when $p_A \le \alpha$ guarantees that the probability of reporting a false discovery (i.e., flagging the results as significant when they are not) is bounded by $\alpha$. 
When multiple analyses, described by statistics $A_1(\mathcal{D}),A_2(\mathcal{D}),\dots,A_k(\mathcal{D})$, are performed, often one wants to bound the \emph{Family-Wise Error Rate} (FWER), that is the probability of even a single false discovery. 
Several approaches~\cite{bonferroni1936teoria,holm1979simple,hochberg1988sharper} have been developed for such a \emph{multiple hypothesis testing} scenario. 
In essence, such approaches require to compute an \emph{adjusted} significant threshold $\alpha^*$, which depends on the desired bound $\alpha$ on the FWER and the number $k$ of hypotheses. 
Results with $p$-value below $\alpha^*$ are then flagged as significant.

For some statistic and null hypotheses, the $p$-value $p_A$ for a statistic $A(\mathcal{D})$ can be computed analytically. 
For example, in market basket analysis, if $A(\mathcal{D})$ is the frequency of a given itemset in $\mathcal{D}$ and the null hypothesis is the one described in our example above (where items are placed independently in each transaction), then it is easy to see that the frequency $A(\mathcal{D}^0)$ of the itemset in a dataset $\mathcal{D}^0$ from the null hypothesis follows a binomial distribution (whose parameters are the product of the frequency of the single items in the itemset and the number of transactions in the dataset). 
Several approaches have been designed to identify significant results for specific data mining tasks, such as significant pattern mining~\cite{webb2006discovering,webb2007discovering,webb2008layered,TeradaOHTS13,minato2014fast}, when the null hypothesis allows the analytical computation of $p$-values. 
However, analytical results cannot be exploited for more complex analyses or null hypotheses. In such scenarios, resampling-based approaches~\cite{westfall1993resampling} are widely used. 
Such methods assess the statistical significance of  the results by comparing $A(\mathcal{D})$ with the \emph{empirical distribution} of the statistic $A(\mathcal{D}^0)$ under the null hypothesis. 
The latter is obtained by generating a large number of \emph{resampled datasets}, or resamples, which are datasets generated according to the null hypothesis. 

\textbf{Example: network polarization.} Given a graph $G$ representing a social network, we are interested in studying its polarization~\cite{conover2011political,cinelli2021echo,garimella2018quantifying,gonzalez2023asymmetric,preti2025dsp} using one of the many measures available as our statistic $A(G)$. 
These measures depend on the graph and on the colors of the vertices, where the color of a vertex represent the different sides or groups in an argument or debate. 
The null hypothesis is that the edges in the graph are independent of the colors of the vertices, but the total number of edges that connect vertices of different colors are fixed; moreover, to capture the highly-skewed degree distribution of real networks, in the null  hypothesis the degree sequence of vertices is fixed. 
For such an example, the $p$-value cannot be computed analytically, even for simple statistics such as the average, over the vertices, of the fraction of neighbors of a vertex that have the same color of the vertex. 
However, one can use recent methods~\cite{fosdick2018configuring,preti2025polaris} to generate (independent) resampled networks from the null distribution.

Resampling-based approaches are also widely used in the multiple hypothesis testing scenario, even when the $p$-values 
can be computed analytically. 
In fact, when the analytical computation of  $p$-values is possible, 
the \emph{adjusted} significant threshold $\alpha^*$ can be 
accurately estimated from resampled datasets~\cite{westfall1993resampling}. 
Compared to analytical corrections, 
resampling-based approaches 
result in higher \emph{statistical power} (i.e., the ability to correctly flag significant hypotheses).
This is a consequence of the fact that they 
inherently consider the correlation structure among hypotheses, 
while analytical adjustments are affected by worst-case assumptions.

Given their usefulness and wide applicability, several methods leverage a resampling scheme (see Section~\ref{sec:relworks}). 
However, they all require to generate \emph{thousands} of resampled datasets for common values of $\alpha$ (i.e., $\alpha = 0.1$ or $\alpha = 0.05$). 
For complex analyses and/or large datasets,  mining or generating even a \emph{single} resampled dataset may take hours or longer. 
Indeed, this is the case for the network analysis example described above. 
Therefore, currently available resampling-based approaches are impractical for modern-sized datasets whenever the data mining step or the resampling procedure is computationally intensive.

\textbf{Our contributions.} In this paper, we focus on resampling-based approaches for assessing the statistical significance of data mining results. Our contributions are fourfold. 
Firstly, we introduce \algname, a simple and effective resampling-based approach to assess the statistical significance of data mining results  while controlling the FWER. 
\algname\ uses a few-shot resampling approach, that is, it analyzes a small number of resampled datasets, leading to a highly scalable method. 
\algname\ can leverage \emph{any} procedure to generate resampled datasets and can be applied to any data mining analysis where the results are summarized by a statistic. 
Therefore, our approach is applicable to a wide range of situations and data mining tasks. 
Second, we derive a novel bound to the supremum deviation of test statistics representing the quality of data mining results.  
This result is crucial in making our method practical, since it implies that mining a small number of resampled datasets is enough to identify statistically significant results.  
Third, we prove that \algname\ also provides guarantees on its statistical power (i.e., the ability to correctly report significant hypotheses). Remarkably, our analysis does not require any assumption on the distribution of significant patterns.
Fourth, we perform an extensive empirical evaluation on several analyses for common data mining tasks such as pattern mining and network analysis. In all cases, \algname\ results in a reduction of up to two orders of magnitude in running time compared to the state of the art, 
while preserving high statistical power, enabling the statistical validation of data mining results on large-scale real-world datasets.

\section{Related Works}
\label{sec:relworks}

Our work focuses on efficient resampling-based approaches to evaluate the statistical significance of data mining results. 
We now describe the works most related to ours, and point the reader to recent reviews and tutorials~\cite{hamalainen2019tutorial,PellegrinaRV19b} for a wider introduction to statistically-sound data mining algorithms, and to the book by~\citet{westfall1993resampling} for an introduction to resampling-based approaches.

\citet{gionis2007assessing} introduced the use of permutation strategies to assess data mining results. 
Their approach focuses on datasets with binary features, represented as $0$-$1$ matrices, and on a specific null hypothesis, where  every column and every row in a random dataset must have the same number of $1$'s as the original dataset. 
\citet{kirsch2012efficient} introduced a method to identify a set of statistically significant patterns with a small false discovery rate (FDR)~\cite{benjamini1995controlling}. 
Their method is specific for itemsets mining and for the null hypothesis where items are placed independently into transactions while preserving the expected frequency of each item. \citet{dalleiger2022discovering} focused on mining non-redundant sets of statistically significant patterns while controlling the FWER or the FDR. Their method is specific for itemsets and a maximum entropy null-distribution.

General resampling-based approaches~\citep{westfall1993resampling} may be used to assess the statistical significance of data mining results, but their straightforward application requires to generate and mine a large number of resampled datasets.  
A large body of recent work has focused on providing computationally efficient procedures for implementing some permutation-based approaches in specific scenarios, or for the efficient generation of resampled datasets.  
Several works~\cite{TeradaOHTS13,llinares2015fast,PellegrinaRV19a,pellegrina2020efficient,terada2015high} have focused on \emph{significant pattern mining}, where the goal is to identify patterns significantly associated with the value of a given binary feature (the target) in a transactional dataset. \cite{llinares2015fast,pellegrina2020efficient,terada2015high} use the Westfall-Young (WY) permutation test~\cite{westfall1993resampling}. 
The WY permutation procedure requires to estimate the quantile of the distribution of the minimum (over all patterns) $p$-value  (or, equivalently, of the distribution of the maximum deviation, over all patterns, of the measure of significance). 
The use of such quantiles makes WY permutation testing more powerful than previous analytical approaches (e.g., LAMP~\cite{TeradaOHTS13}), but also more computationally expensive, since the estimation of such quantiles requires to mine a large number of permuted datasets.  
The work most related to ours is \citep{pellegrina2024efficient}, which introduced FSR, a few-shot resampling approach for significant pattern mining with rigorous guarantees on the FWER. 
While FSR can be applied to several pattern types (e.g., itemsets, sequential patterns, subgroups) and to two different null hypotheses, it is still limited to the task of significant pattern mining described above, and cannot be applied to other data mining tasks (e.g., more general pattern mining analyses or to graph mining, such as the example in Section~\ref{sec:intro}). 
Our approach instead directly supports any scenario where the results are summarized by a statistic and resampling is possible.
 
Different lines of work have  studied i) efficient procedures to generate resampled datasets for various data types, including sequential data~\citep{tonon2019permutation,jenkins2022speck}, graphs~\cite{preti2025polaris}, and general binary transactional and sequence datasets~\citep{preti2024alice,abuissa2023rohan}, and ii) pattern sampling or mining with constraints~\citep{al2009output,boley2011direct,diop2018sequential,giacometti2018dense,diop2022high} (to reduce the cost of mining a single dataset).
These procedures are orthogonal to our method, which can be used to reduce the number of resampled datasets required in any application and can therefore exploit any method for the efficient generation of resampled datasets and any method to compute/bound the maximum test statistic.

\section{Preliminaries}
\label{sec:prelims} 

We are given a dataset $\dataset \in \X$ from a domain $\X$. 
We assume that $\dataset$ comes from an (unknown) distribution $\probdist$, that is, $\dataset \sim \probdist$. 
We define a set of analyses $\A = \{ A_1 , A_2 , \dots , A_k \}$ as bounded real-valued functions 
$A_i : \X \rightarrow \R$. 
Every analysis $A_i$ corresponds to a data mining task or analysis, and $A_i(\dataset)$ is the \emph{statistic} corresponding to the result of analysis $A_i$ on the dataset $\dataset$. 
For example, if $A_i$ corresponds to frequent pattern mining,  $A_i(\dataset)$ may be defined as the number of frequent patterns mined from $\dataset$.

Our goal is to find statistically significant results from analyses $\A = \{ A_1 , A_2 , \dots , A_k \}$. 
To this end, we need to define a \emph{null hypothesis} corresponding to the situation in which results are not interesting. 
We define the null hypothesis through a corresponding \emph{null distribution} $\probdist_0$ on datasets in the domain $\X$. 
$\probdist_0$ defines the probability with which datasets under the null hypothesis (i.e., with no significant results) are generated. 
We define the set $\tnulls \subseteq  \A = \{ A_1 , A_2 , \dots , A_k \} $ of \emph{true null hypotheses}  as the subset of  $\{ A_1 , A_2 , \dots , A_k \}$ for which the null hypothesis holds. 
We denote with $\dataset^0$ a (random) dataset sampled from the null distribution $\probdist_0$. 

Given the dataset $\dataset$ generated from the (unknown) distribution $\gamma$, we then define the \emph{null hypotheses} for the analyses $A_i \in \A$ as the hypotheses that observed results $A_i(\dataset)$ \emph{well conform} to the distribution of the statistics $A_i(\dataset^0)$, where the dataset $\dataset^0$ comes from the null distribution $\probdist_0$. 
We test the null hypothesis for $A_i$ using the dataset $\dataset$; 
for simplicity, we assume that larger values of $A_i(\dataset)$ corresponds to higher evidence of significance for the $i$-th hypothesis, but our approach easily extends  (e.g., by changing the sign of test statistic) to the case where lower values of $A_i(\dataset)$ provide higher evidence of significance. 
More precisely,  we reject the $i$-th null hypothesis if we are able to conclude that the observed value $A_i(\dataset)$ is \emph{unlikely} under the distribution $\probdist_0$, that is
$\Pr_{\dataset^0 \sim \probdist_0} \left( A_i(\dataset^0) \geq A_i(\dataset) \right) \leq t $,
for some small value of $t$. Such a probability is the \emph{$p$-value} of $A_i$.

More generally, we want to identify a subset of $\A$ that is significant, i.e., the elements of $\A$ that significantly deviate from what is expected under the null distribution $\probdist_0$.
Ideally, we want to reject the largest set of hypotheses while controlling the \emph{Family-Wise Error Rate} (FWER) below a value $\alpha$ defined by the user, where the FWER  is the probability that at least one element of $\tnulls$ is reported as significant (i.e., of at least one false discovery).

A procedure to assess the statistical significance while controlling the FWER is the Bonferroni correction;
this approach 
computes the $p$-value $p_i = \Pr_{\dataset^0 \sim \probdist_0} \left( A_i(\dataset^0) \geq A_i(\dataset) \right)$ 
for each $i$, 
and then 
outputs the set
$\left\{ A_i : p_i \leq \alpha/k \right\} $. 
While the $\alpha/k$ multiplicity correction is straightforward, 
this procedure features two key issues.
The first issue is that, as discussed previously, in most data mining tasks   
the null distribution $\probdist_0$ 
is quite complex, and the probabilities $p_i$ do not admit a closed-form expression;
in these settings, $p_i$ cannot be computed, nor estimated efficiently. 
This is the case for all the real-world settings we consider in this work. 
The second issue is that, even when the $p_i$'s are available, this analytic correction is in almost all situations too conservative and restrictive, resulting in a loss of \emph{statistical power}, that is, few or no significant patterns are reported in output. 
This is due to the inherent worst-case assumption of independence among the analyses of $\A$  
made by the Bonferroni correction, 
which is typically not met in practice. 
To address such issues, state-of-the-art methods adopt resampling to assess the statistical significance of their results.

\subsection{WY Resampling} 
\label{sec:wyresampling}

We now describe the Westfall-Young (WY) resampling approach~\cite{westfall1993resampling} to identify significant elements of $\A$ while controlling the FWER below $\alpha$. The WY resampling procedure inherently considers the correlation structure among analyses $\A = \{ A_1 , A_2 , \dots , A_k \}$, resulting in (asymptotically) optimal statistical power~\cite{meinshausen2011asymptotic}. 
The analysis of~\cite{westfall1993resampling} leverages a technical condition, called \emph{subset pivotality}, that is, the assumption that joint distribution of the statistics for which the null hypothesis is true (i.e., analyses $A_i$ in $\tnulls$) is the same when \emph{all} hypotheses are true null (i.e., $\tnulls = \A$). 
While this assumption is sufficient to establish guarantees, it is still open whether it is necessary 
~\cite{westfall2008multiple}. 
Analogously to previous work on significant pattern mining~\cite{llinares2015fast,pellegrina2020efficient,pellegrina2024efficient}, we assume subset pivotality. 
Note that this assumption is useful only when more than one analysis is performed ($k > 1$), and not needed when $k=1$. 
Define $\delta(\alpha)$ as
\begin{align*}
\delta(\alpha) = \argmin_{x \in \R} \left\{ \Pr_{\dataset^0 \sim \probdist_0} \Bigl( \max_{A \in \A} A(\dataset^0) > x \Bigr) \leq \alpha \right\} .
\end{align*}
The following result, from~\cite{westfall1993resampling},   
guarantees that, if one first identifies $\delta(\alpha)$ and then reports 
in output all analyses $A_i$ with $A_i(\dataset) > \delta(\alpha)$, 
then the FWER is bounded by $\alpha$. That is, the following holds.

\begin{lemma}
\label{lemma:fwerguaranteesres}
Let $R = \{ A_i: A_i(\dataset) > \delta(\alpha) \}$. 
It holds
\begin{align*}
\Pr_{\dataset \sim \probdist}( R \cap \tnulls \neq \emptyset ) = 
\Pr_{\dataset \sim \probdist} \Bigl( \max_{A_i \in \tnulls} \{ A_i( \dataset ) \} > \delta(\alpha) \Bigr)  \leq \alpha .
\end{align*}
\end{lemma}

Despite the interesting theoretical properties of $\delta(\alpha)$, i.e., that it allows reporting a set of hypotheses $R$ as significant with control of the FWER, computing $\delta(\alpha)$ is in almost all practical cases not possible, since no closed form of the joint distribution of $\A$ is typically available.
Therefore, $\delta(\alpha)$ is estimated using the following resampling procedure. For an integer $r \geq 1$, define $\sample = \{ \dataset^0_1 , \dots , \dataset^0_r \}$
as a set of $r$ resamples taken i.i.d. from the null distribution $\probdist_0$.
For all $i \in [1 , r]$, 
define $d_i$ as the maximum test statistic, over the set $\A$, evaluated on the resample $\dataset^0_i$:
$d_i = \max_{A \in \A} \{ A(\dataset^0_i) \}$.
Assume, w.l.o.g., that the values $\{ d_1 , \dots , d_r \}$ are sorted in non-increasing order, such that $d_i \leq d_j , \forall i > j$.
The estimate $\tilde{\delta}(\alpha , \sample)$ of $\delta(\alpha)$ 
is defined as the $\alpha$-quantile of the set $\{ d_1 , \dots , d_r \}$:
$\tilde{\delta}(\alpha , \sample) = d_{ \lceil \alpha r \rceil } $.
From the definition of $\delta(\alpha)$, the estimate $\tilde{\delta}(\alpha , \sample)$ converges to $\delta(\alpha)$ as $r \rightarrow + \infty$, i.e., it holds 
$\lim_{r \rightarrow + \infty} \E_\sample [ \tilde{\delta}(\alpha , \sample) ] = \delta(\alpha) $. 

The WY resampling approach provides a powerful technique to identify significant results whenever one or more analyses are considered, and it is widely applicable since it only requires the ability of resampling datasets from the null distribution $\probdist_0$. 
However, in practice its application can be extremely computationally intensive, in particular for large datasets or complex data mining tasks, for three reasons. First, in order to obtain accurate estimates of $\delta(\alpha)$, the number $r$ of resampled datasets must be very large even for reasonably large values of $\alpha$. In fact, for commonly used values of $\alpha$ (i.e., $\alpha = 0.1$ or $\alpha=0.05$), $r$ should be of the order of $10^4$. Second, mining a dataset to compute $d_i = \max_{A \in \A} \{ A(\dataset^0_i) \}$ can be very expensive for large datasets, even when only one complex analysis is performed. Third, in many applications generating even one resample from the null distribution $\probdist_0$ is very expensive (e.g., for generating networks where the degrees of vertices are fixed).

\section{\algname\ Algorithm}
\label{sec:algo}

In this section we describe our algorithm \algname\ to identify statistically significant results with rigorous guarantees on the FWER. \algname\ uses a \emph{few-shot} resampling approach, and  can be used in every situation where resampling-based approaches apply, providing a significant advantage even when a single analysis is performed.
\algname\ builds on a novel probabilistic bound to the supremum deviation of statistics that is estimated on an extremely small number of resampled datasets (see Section~\ref{sec:analysis}).

\algname\ is described in Algorithm~\ref{algo:main}. At a high level, \algname\ computes an upper bound $\hat{\delta}(\sample)$ to the ideal threshold $\delta(\alpha)$ defined by the WY resampling approach, and reports as significant all analyses with values (in the real dataset) above $\hat{\delta}(\sample)$. The number $m$ of resampled datasets used to compute $\hat{\delta}(\sample)$ is set to the minimum value which guarantees that $\hat{\delta}(\sample) \ge \delta(\alpha)$ with probability at least $1-\alpha$, which leads to the required FWER guarantees (see Section~\ref{sec:analysis}).

More in details, \algname\ starts by computing the number $m$ of resampled datasets, that depends on the FWER upper bound $\alpha$ provided in input by the user (line~\ref{alg:numresdef}). It then generates a set $\sample = \{ \dataset^0_1 , \dataset^0_2 , \dots , \dataset^0_m \}$ of $m$ resampled datasets, 
 that are drawn i.i.d. from the null distribution $\probdist_0$ (line~\ref{alg:resample}). 
 The resampled datasets are generated with the procedure \textsc{Resample($\probdist_0 , m$)}. 
 Then, \algname\ computes a threshold $\hat{\delta}(\sample)$ by evaluating the analyses $\A$ on the $m$ resampled datasets $\sample$ (line~\ref{alg:deltahat}). The definition of $\hat{\delta}(\sample)$ is extremely simple, since $\hat{\delta}(\sample)$ is the maximum value observed in the resampled datasets over the performed analyses. The set $R$ of results is then computed by finding the analyses with values in the real dataset that exceed the threshold $\hat{\delta}(\sample)$ (line~\ref{alg:rsetdef}).

\begin{algorithm}[htb]
\SetNoFillComment%
%\DontPrintSemicolon% Old LaTeX installations require \dontprintsemicolon
  \KwIn{Set of analyses $\A$; dataset $\dataset$; null distr. $\probdist_0$; $\alpha \in (0,1)$.}
  \KwOut{\mbox{Set $R \subseteq \A$ of significant analyses with FWER $\le \alpha$.}} 
  $m \gets \left\lceil \frac{ \ln \bigl( \frac{1}{\alpha} \bigr) }{ \ln \bigl( \frac{1}{1 - \alpha} \bigr) } \right\rceil$\; \label{alg:numresdef}
  $\sample \gets $ \textsc{Resample($\probdist_0 , m$)}\; \label{alg:resample}
  $\hat{\delta}(\sample) \gets \max_{\dataset^0 \in \sample , A \in \A}  A(\dataset^0) $\; \label{alg:deltahat}
  $R \gets \bigl\{ A \in \A : A(\dataset) > \hat{\delta}(\sample) \bigr\} $\label{alg:rsetdef}\;
  \textbf{return} $R$\label{alg:rsetout}\;
  \caption{\algname}\label{algo:main}
\end{algorithm}

Note that the number $m$ of resampled dataset required by \algname\ is much smaller than current approaches, which is of the order of $10^4$. For example, \algname\ requires only $m=22$ resampled datasets for $\alpha = 0.1$, and $m=59$ resampled datasets for $\alpha = 0.05$, providing a significant reduction over the state-of-the-art.

\algname\ makes use of the procedure \textsc{Resample($\probdist_0 , m$)} to generate the $m$ resampled datasets from the null distribution $\probdist_0$. The specific resampling procedure depends on the datasets, data mining tasks, and null hypothesis of interest. These must be specified by the user and are captured by the null distribution $\gamma_0$ provided in input to \algname. However, \algname\ can employ any resampling procedure, with no restriction, and is therefore applicable to any scenario where resampling approaches are employed.

\subsection{Analysis}
\label{sec:analysis}

We now prove that \algname\ provides the required guarantees on the FWER. Due to space constraints, all proofs are in the Appendix.

We  start by proving the following technical result, which shows that, under the null hypothesis $\gamma_0$, the maximum deviation has probability at least $\alpha$ to have a value greater \emph{or equal} than the ideal threshold $\delta(\alpha)$ defined by the WY resampling approach.  Intuitively, this is trivially true when $\max_{A \in \A} \{ A( \dataset^0 ) \}$ is a continuous random variable, while it follows from the definition of $\delta(\alpha)$ (see Section~\ref{sec:wyresampling}) when $\max_{A \in \A} \{ A( \dataset^0 ) \}$ is a discrete random variable.

\begin{lemma}
\label{lemma:oneprobbound}
Let $\dataset^0$ be a resample taken from $\probdist_0$. 
Then, it holds
\[
\Pr_{\dataset^0 \sim \probdist_0} \Bigl( \max_{A \in \A} \{ A( \dataset^0 ) \} \geq \delta(\alpha) \Bigr) \geq \alpha .
\]
\end{lemma}

Now, consider again the $m \geq 1$ resamples $\sample = \{ \dataset^0_1 , \dataset^0_2 , \dots , \dataset^0_m \}$ drawn i.i.d. from the null distribution $\probdist_0$ by \algname.
Recall that 
$\hat{\delta}(\sample) = \max_{\dataset^0 \in \sample , A \in \A} \{ A(\dataset^0) \}$
is the maximum test statistic computed over all resamples. 
The following result proves that, if $m$ is large enough, then $\hat{\delta}(\sample)$ provides an upper bound to $\delta(\alpha)$ with probability at least $1 - \alpha$.

\begin{lemma}
\label{lemma:numresamples}
Let $m \ge \Bigl \lceil \frac{\ln\left(\frac{1}{\alpha}\right)}{\ln \left( \frac{1}{1 - \alpha} \right)} \Bigr \rceil $.
Then  
$\Pr_{\sample} \bigl( \hat{\delta}(\sample) \geq \delta(\alpha) \bigr) \geq 1 - \alpha$.
\end{lemma}

Lemma~\ref{lemma:numresamples} implies that, when $m$ is set as in \algname, the set  
$R = \{ A_i : A_i(\dataset) > \hat{\delta}(\sample) \}$
can be flagged as significant with FWER $\leq \alpha$, i.e., it holds 
$\Pr( R \cap \tnulls \neq \emptyset ) \leq \alpha $ 
(since $R \subseteq \{ A_i : A_i(\dataset) > \delta(\alpha) \}$). 
This directly implies the guarantees of \algname.

\begin{corollary}
The output of \algname\ has FWER $\leq \alpha$.
\end{corollary}
To better interpret the bound to $m$ given in Lemma~\ref{lemma:numresamples}, 
we observe that, from the fact $e^x \geq 1+x$, it holds $1/\ln(\frac{1}{1-\alpha}) \leq 1/\alpha$; therefore, $m \leq \left\lceil \ln(1/\alpha) / \alpha \right\rceil$. 
This implies that the number $m$ of resamples used by \algname\ 
scales nicely w.r.t. $\alpha$. 
Finally, we remark that the value $m$ used by \algname\ is the \emph{minimum} number of resamples to control the FWER; 
this property is clearly useful for computational reasons, but also to maximize the statistical power, as increasing the size of $\sample$ leads to higher values of the threshold $\hat{\delta}(\sample)$.

\textbf{Normalization for Multiple Analyses.} 
While the guarantees of \algname\ hold for any set $\A$ of test statistics, 
a common issue of resampling-based approaches is that the statistical power may be reduced when the functions $A \in \A$ have drastically different scale (and/or variance). 
For instance, consider two functions $A_1 , A_2$ with 
$\E_{\dataset_0}[ A_1(\dataset^0) ] \gg A_2(\dataset)$:
 it is highly unlikely that $A_2$ will be returned as significant, regardless of how far $A_2(\dataset)$ is from what expected under the null distribution, since 
$\delta(\alpha) \geq \E_{\dataset_0}[ A_1(\dataset^0) ]$. 
To improve the power in the case of heterogeneous test statistics, we show that they can be normalized without affecting any theoretical guarantee. 
For any $A \in \A$, if its expected value $\mu_A = \E_{\dataset_0}[ A(\dataset^0) ]$ and its variance $\sigma^2_A = \Var_{\dataset_0}[ A(\dataset^0) ]$ are known, we normalize $A$ as
$A^\odot(\dataset) = (A(\dataset) -  \mu_A)/\sigma_A$. 
($A^\odot$ 
is typically known as the \emph{standard score} for $A$.)
Alternatively, we may replace $\sigma_A$ with the range $\Delta_A$ of $A$, i.e., $\Delta_A = \max_{\dataset^0 \in \X} A(\dataset^0) - \min_{\dataset^0 \in \X} A(\dataset^0)$.
When none of the moments of $A$ are known, we propose the following simple and practical solution, also used in our experiments.
Let $\sample^\prime$ be a set of $m^\prime$ i.i.d. resamples taken from $\probdist_0$. 
Then, $\mu_A$ and $\sigma^2_A$ can be replaced by empirical estimates computed on $\sample^\prime$ (details are in Appendix~\ref{sec:appendixnorm}). 
Note that one may be tempted to estimate the moments of $A$ using the resamples $\sample$ generated by \algname; 
however,
this choice may violate the FWER guarantees due to the correlation between the definition of $A^\odot$ and the bound $\hat{\delta}(\sample)$ to its $\alpha$-quantile. 
In our experiments, when normalization was needed, we use an extremely small resampled set $\sample^\prime$ ($m^\prime = 10$), which always provided accurate results. 
\ifextversion
In Appendix~\ref{sec:appxnormalization} we test other choices of $m^\prime$.
\else
In~\cite{fewrsextended} we test other choices of $m^\prime$.
\fi

Given the set $\A^\odot = \{ A_1^\odot , \dots , A_k^\odot \}$ of normalized test statistics, \algname\ can be applied to $\A^\odot$ instead of $\A$, without affecting the FWER guarantees, as stated in the following result. 
\begin{corollary}
\mbox{\!\!The output of \algname$(\A^\odot , \dataset, \probdist_0 , \alpha)$ has FWER $\leq \alpha$.}
\end{corollary}

\textbf{Power of \algname.} 
The previous section focuses on the false positive guarantees of \algname. 
We also present several guarantees on the \emph{power} of \algname. 
Remarkably, our results do not require any assumption on the distribution of significant patterns. 
In particular, we show that the threshold $\hat{\delta}(\sample)$ computed by \algname\ is close to the (ideal) threshold $\delta(\alpha)$;
more formally,  
$\hat{\delta}(\sample)$ does not exceed $\delta(\alpha / k)$ with probability $\geq 1 - t$, where $k$ depends on $\alpha$ and $t$. 
\begin{lemma}
\label{thm:powerguarantees}
For any $t \in (0,1)$, let $k = \frac{m \alpha}{ \ln(\frac{1}{1-t}) }$.
Then, it holds
\begin{align*}
\Pr ( \hat{\delta}(\sample) > \delta(\alpha / k) ) \leq t .
\end{align*} 
\end{lemma}
Interestingly, by setting $t=1/2$,  
we obtain that the \emph{median} value of the threshold $\hat{\delta}(\sample)$ computed by \algname\
will be within the ideal threshold $\delta(\alpha)$ by the modest factor $k \approx 1.44 \cdot \ln(1/\alpha)$, which is $k = 3.17$ when $\alpha = 0.1$, and $k = 4.26$ when $\alpha = 0.05$. 
This guarantees that \algname\ will robustly identify a significance threshold with guarantees on false discoveries, without a significant decrease in terms of power. 
We prove Lemma~\ref{thm:powerguarantees} in 
\ifextversion
Appendix~\ref{sec:app_power}, where 
\else
Appendix~\ref{sec:app_power}. 
In~\cite{fewrsextended} 
\fi
we also
describe and analyze a generalization of \algname:
such approach computes $\hat{\delta}(\sample)$ (defined as the maximum test statistic over the resamples in Algorithm~\ref{algo:main}) as the empirical quantile over an higher number of resamples.
This generalization improves power (i.e., by reducing the factor $k$) 
at the cost of a higher running time,
so it is useful when the generation of a higher number of resamples is possible, and it allows a more flexible exploration of the trade-off between computational resources and accuracy of the procedure.

\section{Experimental evaluation}
\label{sec:experiments}

This section presents the results of our experiments.
The goals of our experimental evaluation are to assess the statistical power and scalability of \algname\
on synthetic datasets, where a ground truth is known, and to compare it with state-of-the-art methods;
we also apply \algname\ to three realistic data mining tasks: 
frequent pattern mining from transactional datasets, 
the analysis of the diversity of interactions in labelled networks,
and mining patterns associated to a target label~\cite{pellegrina2024efficient}. 
For the first two scenarios, our resampling scheme considers expressive null models for transactional datasets~\cite{gionis2007assessing,preti2024alice} 
and random graphs~\cite{fosdick2018configuring,preti2025polaris}. 

\textbf{Experimental setup.} 
We implemented \algname\ in \texttt{Python}.\footnote{The code is available online \url{https://github.com/leonardopellegrina/FewRS}} 
All the experiments were run on a machine with 2.30 GHz Intel Xeon CPU, using up to $32$ cores, 1 TB of RAM, on Ubuntu~22.04. 
To generate resamples for transactional datasets and labelled graphs, we used efficient implementations from previous works~\cite{preti2024alice,preti2025polaris}.

\textbf{Synthetic experiments.} 
In this set of experiments we asses \algname\ on synthetic data, evaluating its statistical power and scalability over a wide choice of parameters. 
We consider a null distribution $\probdist_0$ that generates datasets $\dataset^0 \in \{ 0,1 \}^{n \times d}$ with $n=10^4$ samples and $d$ features, where each entry $\dataset^0_{j,i}$ of $\dataset^0$ is a Bernoulli random variable with parameter $1/2$, i.e., $\Pr( \dataset^0_{j,i} = 1 ) = \Pr( \dataset^0_{j,i} = 0 ) = 1/2 $, for all $j,i$. 
While the entries of the same feature are independent, 
we control the degree of correlation among the $d$ features by first setting all features $i>1$ as copies of the first feature, and then randomly flipping each entry of $\dataset^0$ independently at random with probability $\rho \in [0,1/2]$;
note that when $\rho = 0$ all the features are identical, and when $\rho = 1/2$ they are mutually independent. 
For each $i \in [1,d]$, we define 
$A_i$ as the average value of the $i$-th feature, i.e., $A_i(\dataset^0) = \frac{1}{n}\sum_{j=1}^n \dataset^0_{j,i}$. 
Thus, for each $i$ we test the hypothesis that the $i$-th feature has mean $\leq 1/2$. 
We fix $\alpha = 0.1$, and vary the correlation, by varying  $\rho$, and the number $d$ of features. 
For WY, we use $10^3$ resamples. 
Instead, \algname\ only requires $22$ resamples. 

First, we evaluate the statistical power of each method by estimating its empirical FWER with the following procedure: for each method, we compute its corrected significance threshold $\delta$; then, we draw $100$ resamples from $\probdist_0$ and compute the fraction of times that the maximum statistic $\max_i A_i$ is $> \delta$. 
We repeat these operations $10$ times, and report the average empirical FWER and its standard deviation. 
For the Bonferroni correction, we consider two variants that differ in the computation of the threshold $\delta$. 
The first variant, Bonferroni-E, sets $\delta$ using the exact quantile of the Binomial distribution; note that this approach assumes full knowledge of the distribution of each test statistic, i.e., a way to exactly compute its quantile, which is usually not available. 
The threshold $\delta$ is defined using the exact Binomial distribution, and a union bound over $d$ events:  
$\delta = \min \brpars{ x \in [1 , n] : \sum_{j=x+1}^n \binom{n}{i} \frac{1}{2}^n \leq \frac{\alpha}{d} } $. 
The second variant, Bonferroni-H, uses Hoeffding's bound~\cite{mitzenmacher2017probability}, which only assumes that each test statistic is an average of $n$ random variables bounded $\in [0,1]$, and a union bound over $d$ events, thus is more widely applicable. 
We have $\delta = \frac{1}{2} + \sqrt{ \ln(d / \alpha) / (2n) } $.

\ifextversion
Figure~\ref{fig:empiricalfwerfull} 
shows the results for this experiment, for four values of $\rho \in [0,1/2]$. 
\else
Figure~\ref{fig:empiricalfwerappendix} \textcolor{red}{(in the Appendix)}
shows the \textcolor{red}{results for $\rho = 0.05$} (results for other choices in~\cite{fewrsextended}). 
\fi
Note that higher values of the empirical FWER (below the nominal target $\alpha$) denote that the method is less conservative, thus more powerful.
The results clearly show that the empirical FWER of the WY resampling procedure is always very close to its nominal value $\alpha$, confirming that it is (asymptotically) optimal. 
Furthermore, the empirical FWER of \algname\ is always $\leq \alpha$, confirming the theoretical guarantees from Section~\ref{sec:analysis}.
While the power of \algname\ is slightly lower than WY, it is still comparable and stable over all choices of $\rho$ and $d$.
Remarkably, we observe that the average empirical FWER of \algname\ is $\geq \alpha/3$, confirming the bound to the power of \algname\ discussed in our theoretical analysis. 
For Bonferroni-E, while its empirical FWER is very close to $\alpha$ when $d=1$, or when all test statistics are independent ($\rho = 0.5$), it is overly conservative when they are correlated, resulting in subpar statistical power, particularly when the number $d$ of tests grows. 
Bonferroni-H, which uses less precise information on the distribution of each test statistic, is in all cases the most conservative approach.
We also vary $\alpha \in [ 0.1 , 0.01 ]$ 
\ifextversion
(Figures~\ref{fig:empiricalfwervsalpha} and~\ref{fig:empiricalfwervstimealpha}); 
\else
(Figures in~\cite{fewrsextended}); 
\fi
the results are analogous to the case $\alpha = 0.1$, confirming that, as expected, \algname\ preserves power and guarantees on the FWER in all cases.

Regarding computational resources, we note that we ran the experiments for the WY procedure only for $d \leq 10^3$, due to high running times (as running all the experiments would need $\approx$ a week to complete, we only report estimates of the running time of WY for $d > 10^3$), while both Bonferroni and \algname\ run very fast. 
We remark that the WY procedure poorly scales to higher values of the datasets size $n$, or to a higher number of resamples (i.e., to achieve a more accurate estimate of the quantile for lower values of $\alpha$). 
\ifextversion
Figure~\ref{fig:empiricalfwervstime} shows 
\else
In~\cite{fewrsextended} we show 
\fi
the running time of each method for several $d$, and we compare it with the achieved statistical power (i.e., the empirical FWER). 
As expected, \algname\ is $50 \times$ faster than WY, as it drastically reduces the number of resamples. 
Overall, these experiments show that \algname\ achieves a substantially better trade-off between scalability, allowing to handle large high-dimensional datasets using reasonable computational resources, and statistical power, that is comparable to the optimal WY resampling and significantly better than Bonferroni in realistic settings ($\rho < 0.5$).

\textbf{Analysis of transactional datasets.} 
In this experiment, we apply \algname\ to evaluate the statistical significance of collections of frequent patterns from transactional datasets~\cite{agrawal1993mining}.
We consider $14$ real-world datasets\footnote{From \url{https://archive.ics.uci.edu} and \url{http://fimi.uantwerpen.be}}
often used as benchmark in previous works;
the statistics are shown in Table~\ref{tab:datasets} (in the Appendix). 
We represent each dataset $\dataset \in \{ 0,1 \}^{n \times d}$ as a binary matrix, where each entry $\dataset_{i,j}$ is $1$ when the $j$-th items is contained in the $i$-th transaction, $0$ otherwise.
An itemset $X \subseteq [1 , d]$ is contained in the $i$-th transaction of $\dataset$ if all the items of $X$ are contained in the $i$-th row of $\dataset$, i.e., when $\prod_{j \in X} \dataset_{i,j} = 1$.
The frequency $f(X)$ of $X$ is defined as $f(X) = \frac{1}{n} \sum_{i=1}^n \prod_{j \in X} \dataset_{i,j}$, i.e., is the fraction of transactions of $\dataset$ that contains $X$.
Given a frequency threshold $\theta$, the set $\textsc{FI}(\dataset , \theta)$ contains all itemsets of $\dataset$ with $f(X) \geq \theta$: 
$\textsc{FI}(\dataset , \theta) = \{ X \subseteq \mathcal{I} : f(X) \geq \theta \}$. 
For all $c \in [1,d]$ we define the set $\textsc{FI}(\dataset , \theta , c)$ of frequent itemsets with cardinality $c$ as  
$\textsc{FI}(\dataset , \theta , c) = \{ X \subseteq \mathcal{I} : f(X) \geq \theta , |X| = c \}$. 
In this experiment we evaluate the significance of the set of frequent patterns $\textsc{FI}(\dataset , \theta)$ w.r.t. two different null models for the dataset $\dataset$.
The first model, called GMMT~\cite{gionis2007assessing}, defines $\probdist_0$ as the uniform distribution over the set of all datasets with the same marginals of $\dataset$:
for each $\dataset^0$ sampled from $\probdist_0$, the number of $1$'s in each row and column is the same as in $\dataset$. 
The second model \cite{preti2024alice}, denoted as ALICE, in addition to the column and row marginals, preserves the Bipartite Joint Degree Matrix of the graph representation of $\dataset$.
To evaluate the statistical significance of collections of frequent patterns, we define the test statistic $A(\dataset) = | \textsc{FI}(\dataset , \theta) |$, and we employ \algname\ to generate resamples from $\probdist_0$ to evaluate if the number of frequent patterns in the data is higher than what expected under the null distribution. 
To do so, \algname\
generates resamples $\sample = \{ \dataset^0_1 , \dots , \dataset^0_m \}$ from $\probdist_0$ and computes the significance threshold $\hat{\delta}(\sample) = \max_i A(\dataset^0_i)$. 
We fix $\alpha = 0.05$. 
Recall that, when $A(\dataset) > \hat{\delta}(\sample)$, we conclude that this measure is significant with FWER $\leq \alpha$. 
For both GMMT and ALICE, sampling from $\probdist_0$ is performed using algorithms based on the Markov Chain Monte Carlo (MCMC) and Metropolis-Hastings methods~\cite{peskun1973optimum,mitzenmacher2017probability,preti2024alice}.
For the MCMC, we use default setting for the number of iterations (which is in the order of $\mathcal{O}(|\dataset|)$, see Table~\ref{tab:datasets}), and impose a time limit of $4$ days to traverse the chain; when the time limit is exceeded, the current dataset is returned in output. 
We use the values of $\theta$ shown in Table~\ref{tab:datasets}. 
We remark that for this task there are no analytical procedures to compute the significance of the results, since no closed form of the distribution of $| \textsc{FI}(\dataset , \theta) |$ under these complex null models is available. 

Figure~\ref{fig:numfiandgraphs} shows the results for this experiment. 
The plots compare the number of frequent patterns $| \textsc{FI}(\dataset , \theta) |$ in the dataset $\dataset$
and the values of $| \textsc{FI}(\dataset^0_i , \theta) |$
computed from the resamples drawn from the null distributions, using both the GMMT and ALICE models. 
\ifextversion
The table of Figure~\ref{fig:restables} in Appendix reports 
\else
In the extended version~\cite{fewrsextended} we report 
\fi
the significance thresholds $\hat{\delta}(\sample)$ computed by \algname\ and the average $\overline{\text{FI}} = \frac{1}{m} \sum_{i=1}^m A(\dataset^0_i)$ of the test statistics over the resamples. 
First, we observe that in both models the values of the test statistic measured on the resamples are very concentrated and stable across different random draws. 
Regarding GMMT, from Figure~\ref{fig:numfiandgraphs} we clearly conclude that almost all datasets have a significant number of frequent patterns (since $A(\dataset) > \hat{\delta}(\sample)$). 
The only exception is the dataset retail, which contains a lower number of frequent patterns than what expected under the null. 
As observed by previous work \cite{gionis2007assessing}, this dataset contains structure that may be explained by the row and column marginals alone. 
Considering the ALICE model, we observed similar results, with some interesting differences. 
First, the distribution of $| \textsc{FI}(\dataset^0_i , \theta) |$ changes in some cases, e.g., for the dataset kosarak. 
Furthermore, the number of frequent patterns in retail is significant w.r.t. ALICE; 
in such a case, we may conclude that preserving additional properties of the input dataset in the null distribution $\probdist_0$ highlights additional structure that could not be detected with GMMT. 
Finally, note that \algname\ is in perfect agreement with the WY procedure; all hypotheses that WY would reject are also rejected by \algname, since their statistics $A(\dataset)$ are sensibly higher than the treshold $\hat{\delta}(\sample)$, which provides a guaranteed upper bound to the quantile $\delta(\alpha)$. 
Moreover, all hypotheses that would not be rejected by WY are also not provided in output by our approach, since they fall well below the average resampled values $\overline{\text{FI}}$ (thus also well below the $\alpha$-quantile). 

We then analyze the distribution of frequent patterns with a higher granularity by considering different itemsets' cardinalities. 
For each $c \in [1,10]$, we define the statistic $A_c(\dataset) = | \textsc{FI}(\dataset , \theta , c) |$ which measures the number of frequent itemsets with size $c$, i.e., containing $c$ items. 
As discussed in Section~\ref{sec:analysis}, 
computing $\hat{\delta}(\sample)$ using the test statistics $A_c$ directly may not be very effective, as we expect such functions to have different scale and range;
therefore, we consider the centered and range-normalized test statistics $A_c^\odot$ from Equation~\eqref{eq:teststatnormalization} (Appendix~\ref{sec:appendixnorm}), 
where $\sample^\prime$ is an independent set of $m^\prime = 10$ resamples from $\probdist_0$.
Thus, we define $\hat{\delta}(\sample) = \max_{i , c} A_c^\odot(\dataset_i^0)$
and report as significant all $c$ with $A_c^\odot(\dataset) > \hat{\delta}(\sample)$.

\begin{figure*}[ht]
\begin{subfigure}{.16\textwidth}
  \centering
  \includegraphics[width=\textwidth]{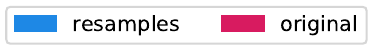}
\end{subfigure} \\
\begin{subfigure}{.2465\textwidth}
  \centering
  \includegraphics[width=\textwidth]{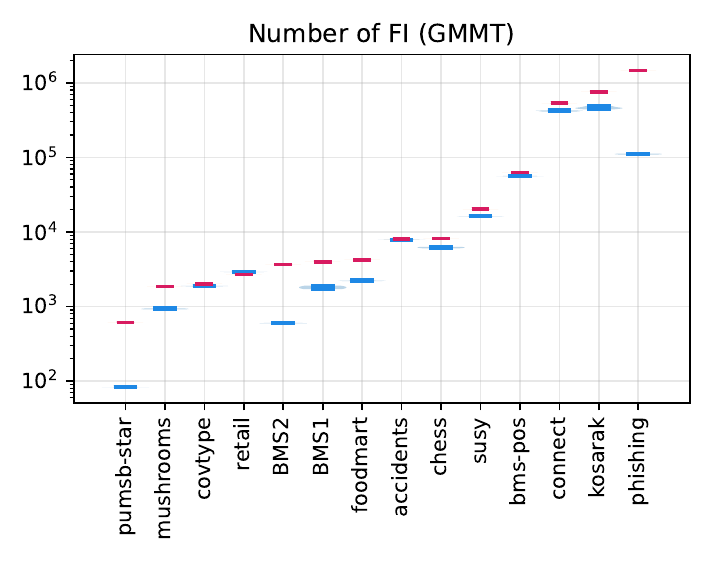} 
    %\caption{}
\end{subfigure}
\begin{subfigure}{.2465\textwidth}
  \includegraphics[width=\textwidth]{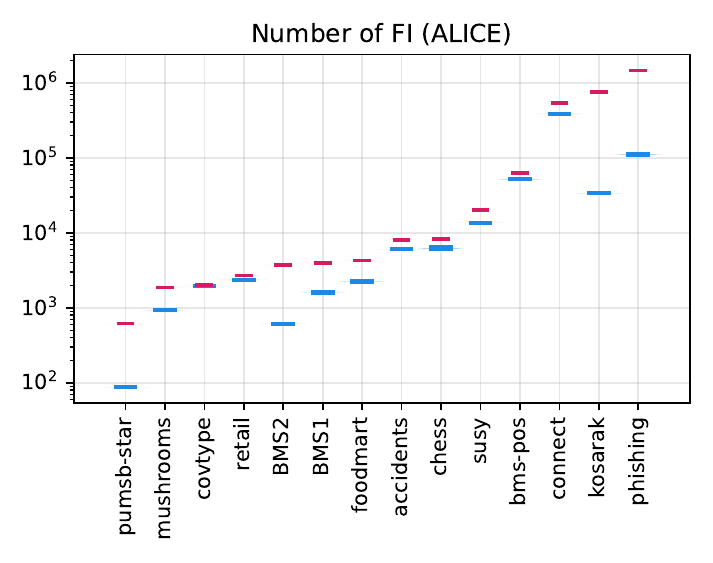}
  %\caption{}
\end{subfigure}
\begin{subfigure}{.2465\textwidth}
  \centering
  \includegraphics[width=\textwidth]{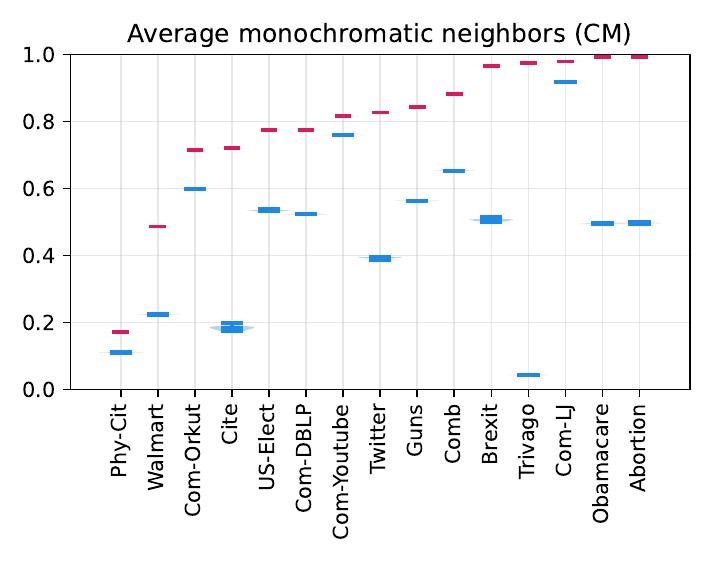} 
\end{subfigure}
\begin{subfigure}{.2465\textwidth}
  \centering
  \includegraphics[width=\textwidth]{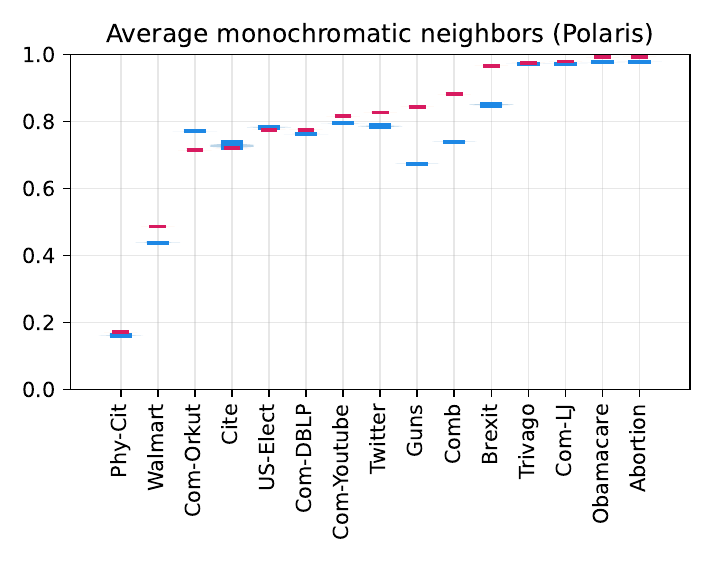}
  %\caption{}
\end{subfigure}
\caption{ 
Left: testing the significance of the number of FI under the GMMT and ALICE models.
Right: testing the significance of $M(G)$ under the CM and Polaris models.
}
\Description{Left: testing the significance of the number of FI under the GMMT and ALICE models.
Right: testing the significance of $M(G)$ under the CM and Polaris models.}
\label{fig:numfiandgraphs}
\end{figure*}

\begin{figure*}[ht]
\begin{subfigure}{.16\textwidth}
  \centering
  \includegraphics[width=\textwidth]{./figures/power-legend.pdf}
\end{subfigure} \\
\begin{subfigure}{.2465\textwidth}
  \centering
  \includegraphics[width=\textwidth]{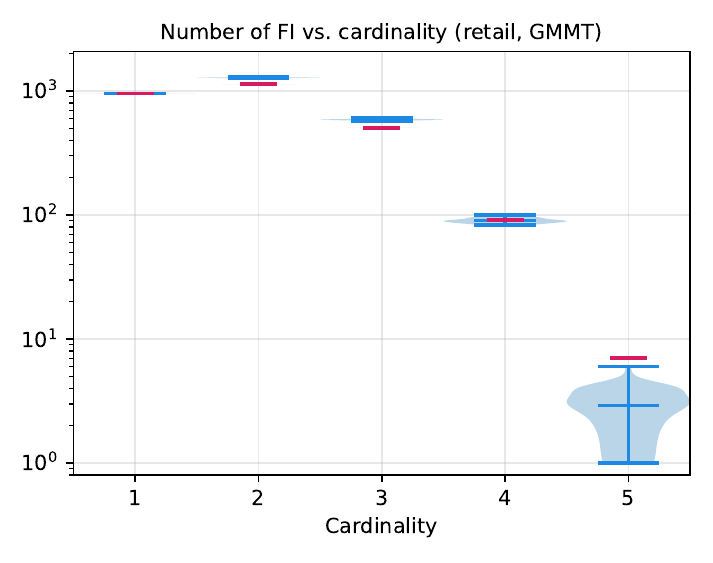} 
    %\caption{}
\end{subfigure}
\begin{subfigure}{.2465\textwidth}
  \includegraphics[width=\textwidth]{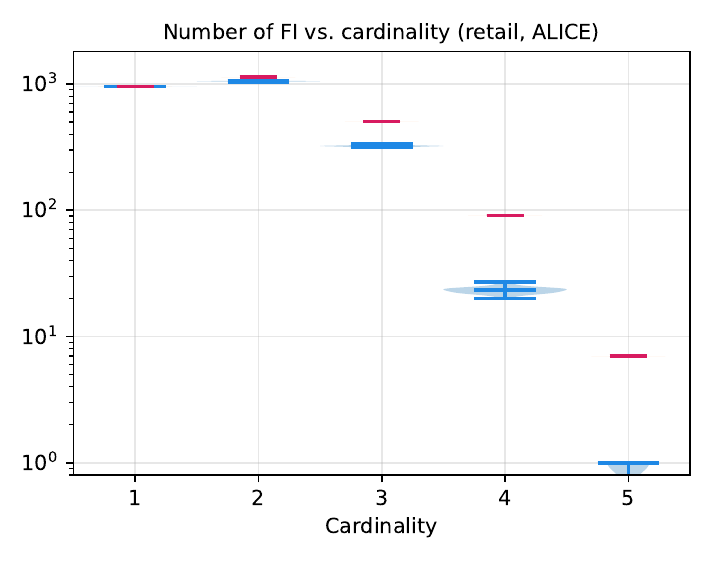}
  %\caption{}
\end{subfigure}
\begin{subfigure}{.2465\textwidth}
  \includegraphics[width=\textwidth]{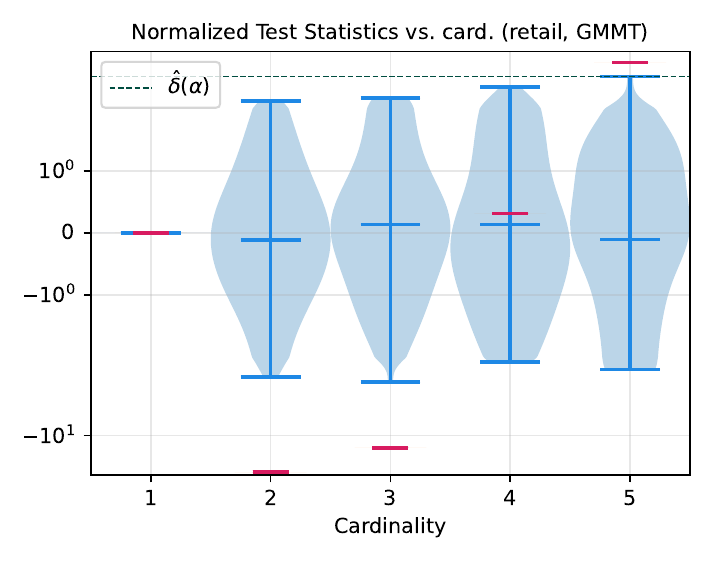}
  %\caption{}
\end{subfigure}
\begin{subfigure}{.2465\textwidth}
  \includegraphics[width=\textwidth]{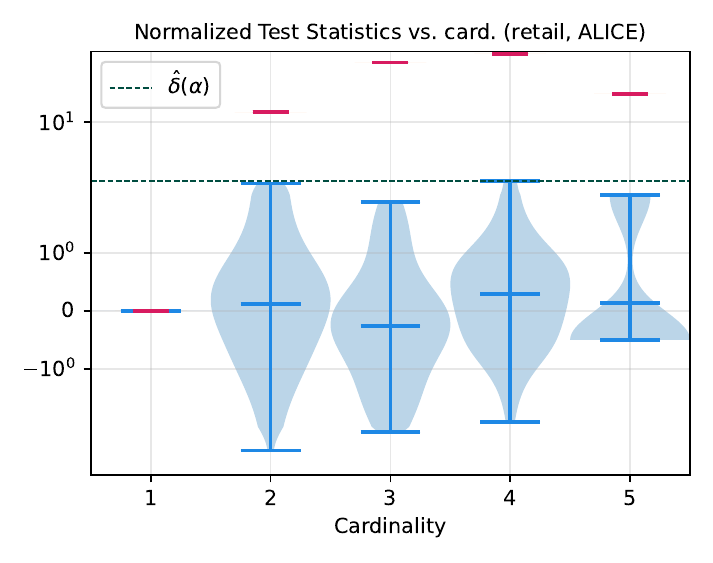}
  %\caption{}
\end{subfigure}
\caption{ 
Testing the significance of the number of FI for different cardinalities under the GMMT and ALICE models.
}
\Description{
Testing the significance of the number of FI for different cardinalities under the GMMT and ALICE models.}
\label{fig:numfilen}
\end{figure*}

\begin{figure*}%[ht]
\begin{subfigure}{.52\textwidth}
  \centering
 % \vspace{-2px}
  \includegraphics[width=.99\textwidth]{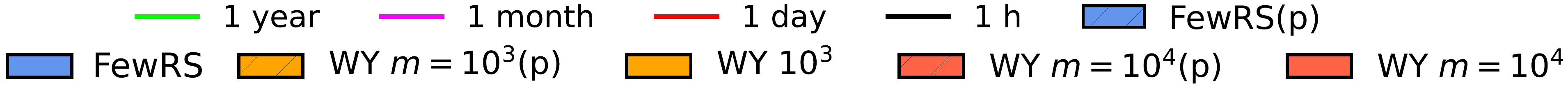}
\end{subfigure} \\
\begin{subfigure}{.455\textwidth}
  \centering
  \includegraphics[width=\textwidth]{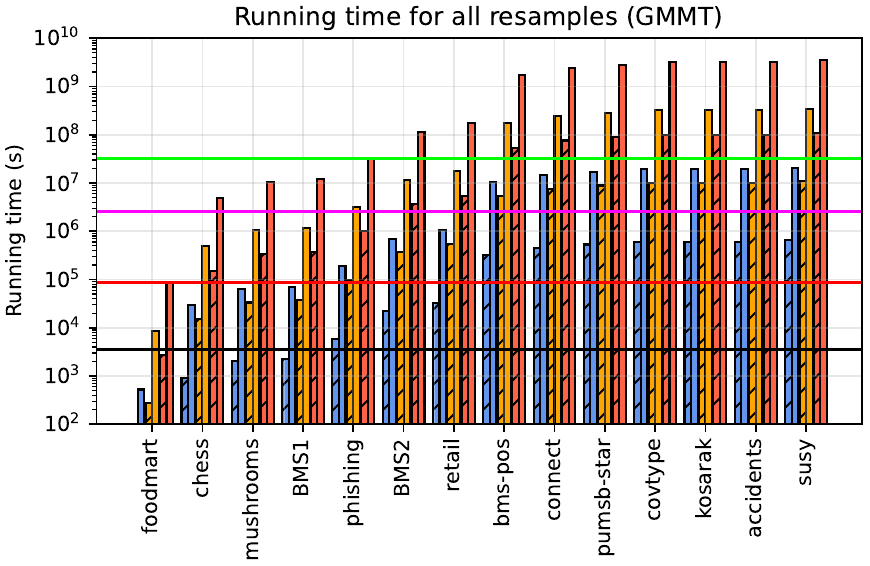}
  %\caption{}
\end{subfigure}
\begin{subfigure}{.455\textwidth}
  \centering
  \includegraphics[width=\textwidth]{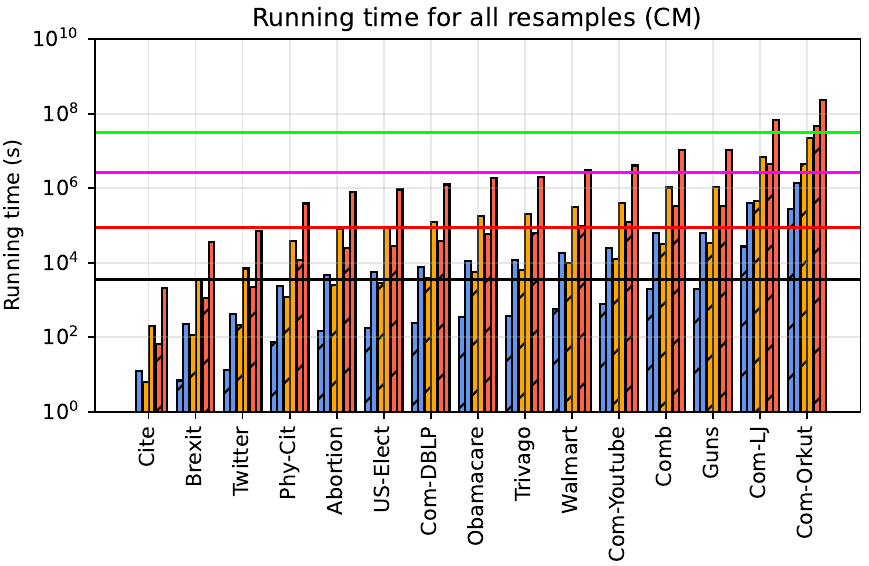}
  %\caption{}
\end{subfigure}
\caption{ 
Running times to generate and analyze all samples with \algname\ and WY for the GMMT and CM models.
The label (p) denotes parallelized approaches with $32$ cores.
}
\Description{Running times to generate and analyze all samples with \algname\ and WY for the GMMT and CM models.
The label (p) denotes parallelized approaches with $32$ cores.}
\label{fig:runningtimegmmtcm}
\end{figure*}

Figure~\ref{fig:numfilen}~reports the results for this experiment for the dataset retail 
\ifextversion
(all other datasets are shown in Figures \ref{fig:numfilengmmtappendix}-\ref{fig:numfilenalicenormappendix} in Appendix due to space constraints).
\else
(all other datasets are shown in the extended version~\cite{fewrsextended} due to space constraints).
\fi
The plots compare the distribution of the number $A_c(\dataset)$ of frequent patterns with cardinality $c$ in the dataset $\dataset$ with the values $A_c(\dataset^0_i)$ observed from the resamples, under the GMMT and ALICE models. 
Under GMMT, itemsets with cardinality $5$ are significant for retail,
an observation that may be overlooked by only considering the global number of frequent patterns.
Then, we observe that the structure of the frequent patterns for lower cardinalities is more interesting under the ALICE model. 
We also show the values of the normalized test statistics $A_c^\odot$ 
and the values of the significance threshold $\hat{\delta}(\sample)$;
the plots confirm that the employed normalization is very effective in detecting deviations from the null distribution that are comparable across different scales and ranges.

For running times, 
Figure~\ref{fig:runningtimegmmtcm} shows the resources required by \algname\ and (an estimate of) WY to generate all resamples and 
assess the statistical significance of the FI under GMMT 
\ifextversion
(ALICE is similar, see Figure~\ref{fig:runningtimealice} in the Appendix). 
\else
(ALICE is similar, shown in~\cite{fewrsextended}). 
\fi
Generating even one resample from the null models is quite costly 
\ifextversion
(see Figure~\ref{fig:runningtimealice} in the Appendix), 
\else
(see Figure~\ref{fig:runningtimegmmtsingle} in the Appendix), 
\fi
in particular for large datasets, due to a high number of MCMC iterations needed to reach convergence and each iteration being fairly expensive.
The immediate consequence is that generating and mining a high number of resamples, as done by the WY procedure, is extremely expensive, and infeasible for larger instances.
In fact, WY would need several months to analyze $7$ out of the $14$ datasets we considered, using $10^3$ resamples (and parallelizing over $32$ cores). 
Using the recommended amount of resamples ($10^4$) would require more than a year of computation.
Instead, \algname\ is faster by at least one order of magnitude, up to two orders, thus it enables the analysis of these instances by requiring at most a few days of computation. 
Therefore, this experiment clearly shows that \algname\ is very effective and powerful in assessing significance for frequent pattern mining tasks, while greatly reducing the computational burden of state-of-the-art resampling procedures.

\textbf{Analysis of labelled networks.}
Next, we apply \algname\ to analyze large real-world labelled networks. 
We consider $15$ graphs, where each node $v$ has a categorial label $c(v)$, i.e., a color. 
For instance, in graphs from online social networks the color of a node denotes the view of the corresponding profile w.r.t. a polarizing topic (e.g., the graph Guns encodes the opinion of each user w.r.t. gun control~\cite{garimella2018quantifying}). Table~\ref{tab:graphs} in the Appendix shows their statistics. 

\begin{figure*}[ht]
\begin{subfigure}{.16\textwidth}
  \centering
  \includegraphics[width=\textwidth]{./figures/power-legend.pdf}
\end{subfigure} \\
\begin{subfigure}{.2465\textwidth}
  \centering
  \includegraphics[width=\textwidth]{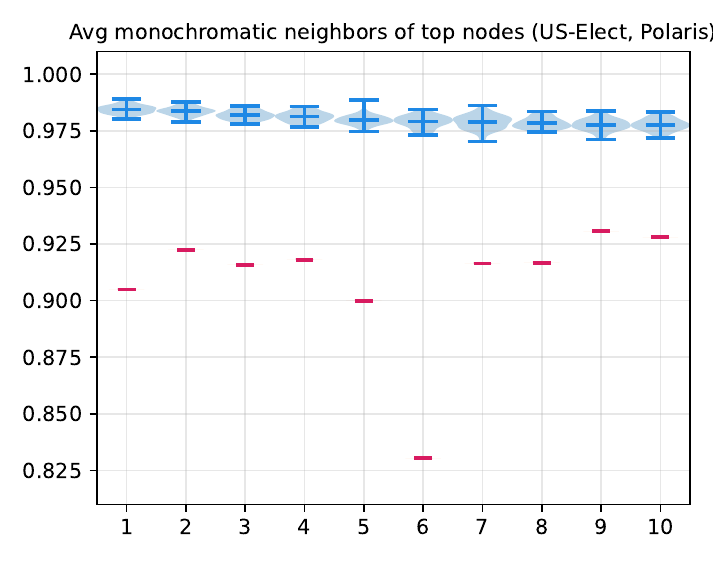} 
\end{subfigure}
\begin{subfigure}{.2465\textwidth}
  \centering
  \includegraphics[width=\textwidth]{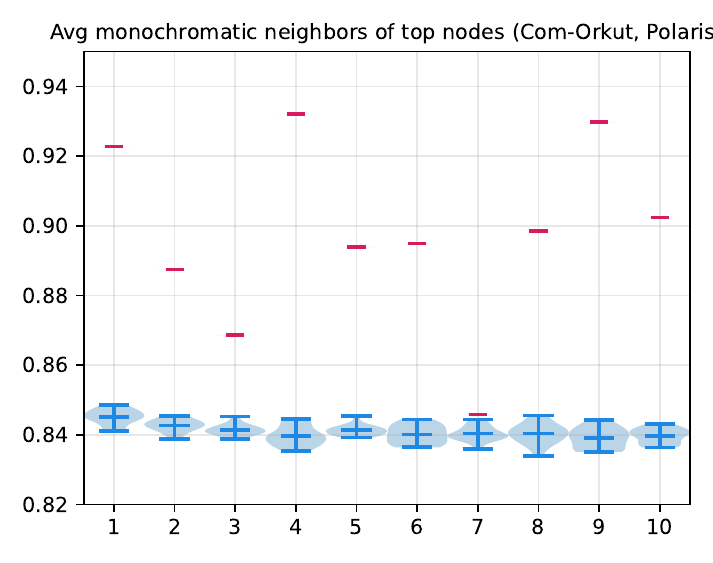} 
\end{subfigure}
\begin{subfigure}{.2465\textwidth}
  \centering
  \includegraphics[width=\textwidth]{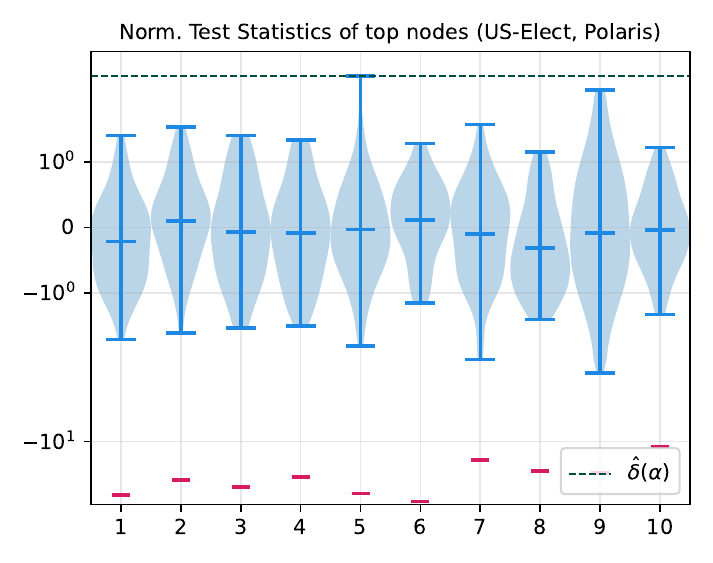} 
\end{subfigure}
\begin{subfigure}{.2465\textwidth}
  \centering
  \includegraphics[width=\textwidth]{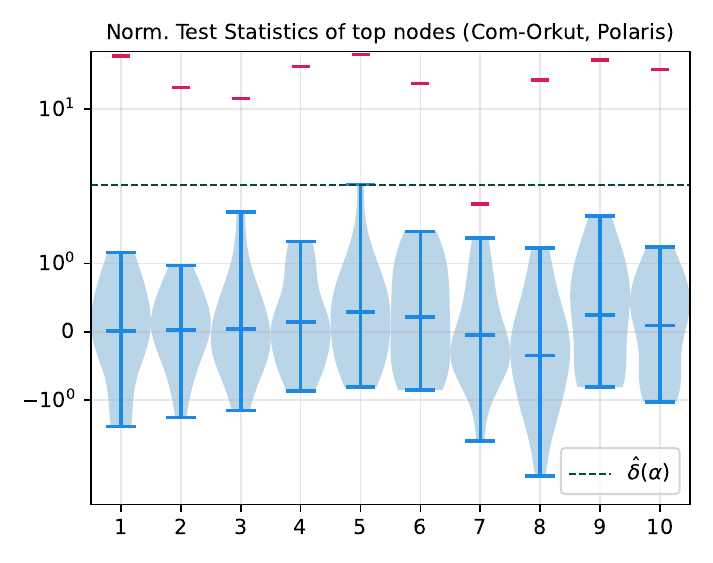} 
\end{subfigure}
\caption{ 
Testing the significance of $M_v(G)$ of the $10$ nodes with highest degree under the Polaris model.
}
\Description{
Testing the significance of $M_v(G)$ of the $10$ nodes with highest degree under the Polaris model.}
\label{fig:avgcolpolaristopnodes}
\end{figure*}

For each graph $G = (V , E)$ we consider a relatively simple measure $M(G)$ of the level of polarization of the network, which computes the average fraction of neighbors of a node $v$ that has the same color of $v$. 
More precisely, let $N(v)$ be the multiset of neighbors of the node $v \in V$. 
We define   
$M(G) = \frac{1}{|V|} \sum_{v \in V} \frac{ | \{ u \in N(v) : c(u) = c(v) \} | }{|N(v)|} $. 
A large value of $M(G)$ denotes that most nodes of the network are connected to nodes of the same color, e.g., that have the same side on a polarizing topic, meaning that they have a low exposure to diverse opinions. 
We also evaluate a more local measure of neighbor diversity by testing the connections of the most relevant nodes of the graph. 
We consider the $10$ nodes with highest degree in the graph $G$; for each such a $v$, we define the measure $M_v(G) = | \{ u \in N(v) : c(u) = c(v) \} | / |N(v)|$ which only considers the colors of the neighbors of $v$. 
We evaluate the statistical significance of these measures using resamples from two null distributions $\probdist_0$.
First, we consider the Configuration Model (CM) \cite{bollobas1980probabilistic,fosdick2018configuring}, where a random graph $G^0 \sim \probdist_0$ is taken uniformly at random from the set of all graphs with the same degree sequence of the input graph $G$.
Then, we consider Polaris~\cite{preti2025polaris}, that is an extension of CM to labelled graphs which preserves also the color assortativity of the graph, that is the total number of edges between all pairs of colors~\cite{newman2003mixing}. 
As for the previous experiment, we sample from the null models with state-of-the-art MCMC algorithms~\cite{peskun1973optimum,mitzenmacher2017probability,fosdick2018configuring,preti2025polaris}.

To evaluate the statistical significance of $M(G)$, for each graph $G$, we apply \algname\ with $\alpha = 0.05$, which generates resamples $\sample = \{ G^0_1 , \dots , G^0_m \}$ from $\probdist_0$ and computes the significance threshold $\hat{\delta}(\sample) = \max_i M(G^0_i)$. 
When $M(G) > \hat{\delta}(\sample)$, we conclude that this measure is significant with FWER $\leq \alpha$. 
Regarding $M_v(G)$, we define centralized and range-normalized statistics $A_v^\odot(G)$ using an independent set of resamples $\sample^\prime$ as done for the frequent patterns setting. 
While we fix $m^\prime = 10$, we test other choices  
\ifextversion
in Appendix~\ref{sec:appxnormalization}. 
\else
in~\cite{fewrsextended}. 
\fi
Note that also for this task we can't use any analytical correction procedure, since no closed form of the distribution of $M(G^0)$ and $M_v(G)$ under these complex null models is available.

Figure~\ref{fig:numfiandgraphs} shows the results for this experiment. 
The plots show a comparison between the values of $M(G)$ and the values of $M(G^0)$ from the resamples $G^0$ generated from $\probdist_0$, for both the CM and Polaris models.
\ifextversion
The table in Figure~\ref{fig:restables} in Appendix reports 
\else
In the extended version~\cite{fewrsextended} we report  
\fi
the values of $M(G)$, the means $\overline{M} = \frac{1}{m} \sum_{i=1}^m M(G^0_i)$, and the thresholds $\hat{\delta}(\sample)$ computed by \algname.
Note that, for all graphs and under CM, it holds $M(G) > \hat{\delta}(\sample)$; in all networks $M(G)$ is significantly higher compared to random graphs from CM. 
This may be 
a sign that the fraction of neighbors with the same color is not explained by the degree sequence of the nodes.
Then, we clearly observe that preserving the color assortativity with Polaris has a significant impact on the distribution of $M(G^0)$: the values of $M(G^0)$ are considerably higher in Polaris compared to the CM model.
Nevertheless, in all but $3$ graphs (Cite, US-Elect, and Com-Orkut) we observe $M(G) > \hat{\delta}(\sample)$ also under the Polaris model. 
In such $3$ cases, 
we observed $M(G) < \overline{M}$, i.e., a fraction of mono-chromatic neighbors lower than what expended under the Polaris null model.
These findings imply that the users in most of the networks we considered are exposed to less diverse content than expected under the null;
on the other hand, in some cases the color assortativity of the network may explain the observed overall fraction of mono-chromatic neighbors. 
 Figure~\ref{fig:avgcolpolaristopnodes} reports the values $M_v(G)$, that quantify the 
the fraction of neighbors with the same color for the nodes with highest degree in the graphs,
and compare them with $M_v(G^0_i)$ computed from the resampled graphs,
for US-Elect and Com-Orkut under Polaris 
\ifextversion
(all other results are in Figures~\ref{fig:avgcolcmappendix}-\ref{fig:avgcolpolarisnormappendix} in Appendix due to space constraints). 
\else
(all other results are in~\cite{fewrsextended}). 
\fi
Interestingly, the fraction of mono-chromatic neighbors of the top nodes of US-Elect aligns with the corresponding global average, i.e., it is lower than expected; 
on the other hand, for Com-Orkut the top nodes have a significantly higher fraction of mono-chromatic neighbors compared to the overall network. 
These observations provide additional insights on the structure of the overall network, and the role played by the most relevant nodes. 
Finally, note that, as in the previous setting, \algname\ leads to the same results of the WY procedure: all hypotheses that WY would reject are also rejected by \algname, and all hypotheses that would not be rejected by WY are also not provided in output by \algname. 

\balance

For computational resources, 
\ifextversion
Figures~\ref{fig:runningtimegmmtcm}, \ref{fig:runningtimecm},  and~\ref{fig:runningtimepol} (in the Appendix) 
show the time to generate one resample under CM and Polaris models, the time needed by \algname, and the time for the WY procedure. 
\else
Figure~\ref{fig:runningtimegmmtcm} shows the total time needed by \algname\ and the WY procedure for the CM model (the time to generate one resample, and the time for Polaris are shown in~\cite{fewrsextended}). 
\fi
Similarly to previous experiments, sampling from the null distribution is costly  even for a single resample, e.g., $\approx 3$ hours for the largest graph. 
This implies that WY does not scale to the most challenging instances, even when parallelized.
In fact, for the two largest graphs, generating $10^3$ resamples requires almost a week and more than one month, respectively.
Drawing $10^4$ resamples, which is the suggested amount to accurately estimate the $\alpha$-quantile, is clearly infeasible: for the largest graph more than $1$ year of computation would be necessary.
Instead, \algname\ analyzes such instances in at most a couple of days, with a speedup of at least $1$ order of magnitude, up to $2$ orders. 
It is clear that considering larger instances, or more expensive analysis (e.g., evaluating measures more complex than $M(G)$), is not possible with WY. 
Therefore, we conclude that \algname\ significantly accelerates the statistical assessment of graph analysis tasks, while yielding rigorous guarantees and high statistical power; moreover, it enables the analysis of challenging instances, that are out of reach for state-of-the-art methods.

\textbf{Comparison with FSR.} 
Finally, we compare our method with FSR~\cite{pellegrina2024efficient}, a few-shot resampling approach tailored to discover significant patterns from transactional datasets, where each transaction is labelled with a binary target.
The set of analyses $\A = \{ A_P : P \in \L \}$ contains one test statistic $A_P$ for each pattern $P$ from the language $\L$ of \emph{subgroups}; 
each $P$ is defined as (at most $z$) conjunctions of logical conditions over the $d$ features of the dataset $\dataset$ (statistics shown in Table~\ref{tab:datasetsfsr}). 
Note that the size of $\A$ is $\BT{d^z}$. 
Each test statistic $A_P$ corresponds to the $1$-quality of the subgroup $P$, which quantifies its association with the target variable~\cite{pellegrina2024efficient} 
\ifextversion
(more details in Appendix~\ref{sec:appxfsr}). 
\else
(more details in~\cite{fewrsextended}). 
\fi
We fix $\alpha=0.05$, and use $m$ resamples for both \algname\ and FSR; each resample is 
obtained by randomly permuting the target labels, 
while the maximum test statistic is computed with a depth-first search. 
\ifextversion
Figure~\ref{fig:comparisonfsrlarge} (in the Appendix) 
displays the respective thresholds, the number of significant patterns reported in output, and the running time. 
\else
Figure~\ref{fig:comparisonfsrsmall} (in the Appendix) 
displays the respective thresholds and the number of significant patterns reported in output (the running time, shown in~\cite{fewrsextended}, was always very similar, as both methods use the same number $m$ of resamples). 
\fi
We observed that the thresholds $\hat{\delta}(\sample)$ obtained by \algname\ are sensibly smaller than the thresholds $\varepsilon$ computed by FSR, regardless of the fact that FSR is tailored to the considered task, while \algname\ has wide applicability. 
Consequently, \algname\ reports in output a much higher number of significant patterns, for the same work (and running time). 
While FSR works with any number of resamples, as $\varepsilon$ is computed with advanced concentration bounds, 
our results indicate that it is more conservative than \algname. 
Overall, \algname\ achieves an excellent trade-off between power and computation resources, even in a high-dimensional setting where a large number of analyses is performed.

\section{Conclusions}
We introduced \algname, a few-shot resampling approach to assess the statistical significance of data mining results with rigorous guarantees on the probability of false discoveries. 
\algname\ builds on a novel probabilistic bound to the supremum deviation of statistics representing the quality of data mining results, and 
can be used in every scenario where resampling-based approaches are applied. \algname\  generates and mines an extremely small number of resamples, providing a highly scalable approach with wide applicability.  
Our experiments show that \algname\ results in a reduction of up to two orders of magnitude in running time compared to state-of-the-art approaches, 
enabling the identification of statistically-sound data mining results on large-scale real-world datasets.

\begin{acks}
This work is supported by 
the STARS@UNIPD 2025 program, 
project "AtHeNA: Algorithms for Heterogeneous Network Analysis", 
with the support of the University of Padova and Fondazione Cassa di Risparmio di Padova e Rovigo,
and by 
the Italian Ministry of University and Research (MUR), project “National Center for HPC, Big Data, and Quantum Computing” CN00000013. 
\end{acks}

\clearpage
\newpage

\bibliographystyle{ACM-Reference-Format}
\balance
\bibliography{bibliography}

%% file: appendix.tex
% !TEX root = main-fewshotdm.tex

\section{Appendix}
In this appendix we provide additional theoretical results, proofs, and experimental results that could not fit in the main manuscript.

\begin{proof}[Proof of Lemma \ref{lemma:fwerguaranteesres}~\cite{westfall1993resampling}]
From the definition of the null hypotheses, 
and from subset pivotality, 
we have that 
\begin{align*}
 \Pr_{\dataset \sim \probdist} \left( \max_{A \in \tnulls} \{ A( \dataset ) \}  > \delta(\alpha) \right)  
 \leq \Pr_{\dataset^0 \sim \probdist_0} \left( \max_{A \in \tnulls} \{ A( \dataset^0 ) \} > \delta(\alpha) \right) .
\end{align*}
Thus, from the fact $\tnulls \subseteq \A$, and the definition of $\delta(\alpha)$, 
it holds 
\begin{align*}
 \Pr_{\dataset^0 \sim \probdist_0} \! \left( \max_{A \in \tnulls} \{ A( \dataset^0 ) \} \! > \! \delta(\alpha) \right) \! 
 \leq \!\!\Pr_{\dataset^0 \sim \probdist_0} \! \left( \max_{A \in \A} \{ A( \dataset^0 ) \} \! > \! \delta(\alpha) \right) 
  \leq  \alpha ,
\end{align*}
which concludes the proof.
\end{proof}

\begin{proof}[Proof of Lemma \ref{lemma:oneprobbound}]
By definition of $\delta(\alpha)$, for any $\varepsilon > 0$ it holds  
\begin{align*}
\Pr_{\dataset^0 \sim \probdist_0} \left( \max_{A \in \A} \{ A( \dataset^0 ) \} > \delta(\alpha) - \varepsilon \right) > \alpha . 
\end{align*}
Taking the limit $\varepsilon \rightarrow 0^+$ on both sides, we have
\begin{align*}
\alpha \leq
& \lim_{\varepsilon \rightarrow 0^+} \Pr_{\dataset^0 \sim \probdist_0} \left( \max_{A \in \A} \{ A( \dataset^0 ) \} > \delta(\alpha) - \varepsilon \right) \\
& = 1 -  \lim_{\varepsilon \rightarrow 0^+} \Pr_{\dataset^0 \sim \probdist_0} \left( \max_{A \in \A} \{ A( \dataset^0 ) \} \leq \delta(\alpha) - \varepsilon \right) . 
\end{align*}
Let $\varepsilon_1 , \varepsilon_2 , \dots$ be a decreasing sequence $\varepsilon_i > \varepsilon_{i+1}$ converging to $0$. 
Then, let  
$X_i = \text{``} \max_{A \in \A} \{ A( \dataset^0 ) \} \leq \delta(\alpha) - \varepsilon_i \text{''}$. 
These 
form an increasing sequence of measurable sets $X_i \subseteq X_{i+1}$, whose limit is 
\begin{align*}
\lim_{i \rightarrow + \infty} X_i =  \text{``} \max_{A \in \A} \{ A( \dataset^0 ) \} < \delta(\alpha) \text{''} .
\end{align*}
Moreover, since $X_i$ is an increasing sequence, it holds 
\begin{align*}
\lim_{i \rightarrow + \infty} \Pr_{\dataset^0 \sim \probdist_0}( X_i ) 
= \Pr_{\dataset^0 \sim \probdist_0} \left( \lim_{i \rightarrow + \infty} X_i \right)  .
\end{align*}
Therefore, 
\begin{align*}
& \alpha  \leq 1 -  \lim_{\varepsilon \rightarrow 0^+} \Pr_{\dataset^0 \sim \probdist_0} \left( \max_{A \in \A} \{ A( \dataset^0 ) \} \leq \delta(\alpha) - \varepsilon \right) \\
& =
 1 - \!\! \Pr_{\dataset^0 \sim \probdist_0} \! \left( \max_{A \in \A} \{ A( \dataset^0 ) \} \! < \! \delta(\alpha) \right) 
 \! = \!\! \Pr_{\dataset^0 \sim \probdist_0} \! \left( \max_{A \in \A} \{ A( \dataset^0 ) \} \! \geq \! \delta(\alpha) \right)  ,
\end{align*}
obtaining the statement. 
\end{proof}

\begin{proof}[Proof of Lemma \ref{lemma:numresamples}]
We prove that $\Pr_{\sample} ( \hat{\delta}(\sample) < \delta(\alpha) ) \leq \alpha$. 
First, note that the event 
\qt{$\hat{\delta}(\sample) < \delta(\alpha)$} 
is equivalent to the intersection of the $m$ events 
\qt{$\max_{A \in \A} \dataset^0_i < \delta(\alpha)$}, $\forall i \in [1,m]$; 
moreover, these events are mutually independent,  
since each resample $\dataset^0_i$ is taken i.i.d. from $\probdist_0$. 
Then, from Lemma~\ref{lemma:oneprobbound}, it holds
\[
\Pr_{\dataset^0_i} \left( \max_{A \in \A} \dataset^0_i < \delta(\alpha) \right) \leq 1-\alpha , 
\]
for all $i \in [1,m]$. 
Therefore, we have 
\begin{align*}
 \Pr_{\sample} \left( \hat{\delta}(\sample) < \delta(\alpha) \right)   
 = \prod_{i=1}^m \Pr_{\dataset^0_i} \left( \max_{A \in \A} \dataset^0_i < \delta(\alpha) \right) 
 \leq (1-\alpha)^m .
\end{align*}
We now show that, if $m$ is set as in the statement, then $(1-\alpha)^m \leq \alpha$. 
This follow from the fact  
\begin{align*}
 (1-\alpha)^m \leq \alpha \iff  m \geq \frac{ \ln \bigl( \frac{1}{\alpha} \bigr) }{ \ln \bigl( \frac{1}{1 - \alpha} \bigr) } ,
\end{align*}
completing the proof. 
\end{proof}

\subsection{Power Analysis of \algname} 
\label{sec:app_power}
This section presents additional results on the guarantees to the statistical power of \algname.

\begin{proof}[Proof of Lemma \ref{thm:powerguarantees}]
From the definition of $\delta(\alpha)$, it holds 
\begin{align*}
\Pr \left( \max_{A \in \A} A(\dataset^0) > \delta(\alpha / k) \right) \leq \alpha / k .
\end{align*} 
Therefore, we have
\begin{align*}
 & \Pr \left( \hat{\delta}(\sample) \leq \delta(\alpha / k) \right) 
= \prod_{i=1}^m \Pr \left( \max_{A \in \A} A(\dataset^0_i) \leq \delta(\alpha / k) \right) \\
& = \! \prod_{i=1}^m \! \left( 1 - \Pr \left( \max_{A \in \A} A(\dataset^0_i) \! > \! \delta(\alpha / k) \right) \right) 
\geq \prod_{i=1}^m \! \left( 1 - \frac{\alpha}{k}  \right) 
\! = \! \left( 1 - \frac{\alpha}{k} \right)^m \! .
\end{align*} 
Setting $( 1 - \frac{\alpha}{k} )^m \geq 1 - t$ yields 
\begin{align*}
m \leq \frac{\ln \left( \frac{1}{1-t} \right) }{ \ln \left( \frac{k}{k-\alpha} \right)  } 
\leq \frac{k \ln \left( \frac{1}{1-t} \right) }{ \alpha } ,
\end{align*} 
which always holds when $k = \frac{m \alpha}{ \ln(\frac{1}{1-t}) }$, proving the statement.
\end{proof}

\ifextversion
The following results provide a generalization of \algname\ to improve the statistical power when the generation of a higher number of random resamples is possible.
Recall that $\sample = \{ \dataset^0_1 , \dataset^0_2 , \dots , \dataset^0_m  \}$, and define 
$d_i = \max_{A \in \A} \dataset^0_i$. 
For any $\ell \in [1,m]$, let $r(\sample , \ell)$ be the $\ell$-th largest value of $\{ d_i \}$, i.e., $r(\sample , \ell) = \max \{ x : \sum_{i=1}^m \ind{d_i \geq x} \geq \ell \}$.
Finally, define $B(i, m , p)$ as the (left) tail of the Binomial distribution:
\begin{align*}
B(i , m ,  p) = \sum_{j=0}^{i} \binom{m}{j} p^j (1-p)^{m-j} .
\end{align*}

\begin{lemma}
\label{lemma:numresamplesgen}
Let $m \geq \left\lceil \ln(\frac{1}{\alpha})/\ln ( \frac{1}{1 - \alpha} ) \right\rceil$,
and define 
\begin{align*}
\ell = \max \left\{ i \in [1,m] : B(i-1 , m ,  \alpha) \leq \alpha \right\} .
\end{align*} 
If line~\ref{alg:deltahat} of \algname\ is replaced by ``$\hat{\delta}(\sample) \gets r(\sample , \ell)$'',
then the output of \algname\ has FWER $\leq \alpha$. 
\end{lemma}
\begin{proof}
First, note that the set $\left\{ i \in [1,m] : B(i-1 , m ,  \alpha) \leq \alpha \right\}$ is always not empty since $B(0 , m ,  \alpha) \leq \alpha$ given the choice of $m$.

Recall that, from Lemma~\ref{lemma:oneprobbound}, it holds 
\[
\Pr_{\dataset^0 \sim \probdist_0} \Bigl( \max_{A \in \A} \{ A( \dataset^0 ) \} \geq \delta(\alpha) \Bigr) \geq \alpha .
\]
For $i \in [1,m]$, define the i.i.d. Bernoulli random variables $Y_i$
with $\E[Y_i] = \alpha$.

We have 
\begin{align*}
& \Pr \Bigl( r(\sample , \ell) < \delta(\alpha) \Bigr) 
= \Pr \Bigl( \sum_{i=1}^m \ind{ d_i \geq \delta(\alpha) } < \ell \Bigr) \\
& \leq \Pr \Bigl( \sum_{i=1}^m Y_i < \ell \Bigr) 
= \sum_{j=0}^{\ell-1} \binom{m}{j} \alpha^j (1-\alpha)^{m-j} 
= B( \ell-1 , m , \alpha)
\leq \alpha,
\end{align*} 
where the first inequality follows from observing that 
$\ind{ d_i \geq \delta(\alpha) }$ is a Bernoulli random variable with  
$\E[\ind{ d_i \geq \delta(\alpha) }] \geq \E[Y_i] =\alpha$, 
obtaining the statement. 
\end{proof}

The following result shows how to bound the quantile $\delta(\alpha)$ 
by the estimate $r(\sample , \ell)$, 
while simultaneously guaranteeing that $r(\sample , \ell)$ does not exceed $\delta(\alpha/k)$ for any $k > 1$. 

\begin{lemma}
\label{lemma:powertimetradeoff}
Given $k > 1$, let $m , \ell \geq 1$, such that 
\begin{align*}
 B( \ell - 1 , m , \alpha)  \leq \alpha/2 , \text{\hspace{0.3cm}}
 B( \ell - 1 , m , \alpha/k)  \geq 1 - \alpha/2 .
\end{align*}
Then, it holds
\begin{align*}
\Pr ( \delta(\alpha) \leq r(\sample , \ell) \leq \delta(\alpha / k) ) \geq 1- \alpha .
\end{align*} 
\end{lemma}
\begin{proof}
For $i \in [1,m]$, 
define $Y_i, X_i$ 
as mutually independent Bernoulli random variables with 
$\E[Y_i] = \alpha$ and $\E[X_i] = \alpha/k$. 
Recall that 
$\Pr(d_i \geq \delta(\alpha)) \geq \alpha$
and that 
$\Pr(d_i > \delta(\alpha/k)) \leq \alpha/k$. 
It holds 
\begin{align*}
&\Pr ( r(\sample , \ell) < \delta(\alpha) \vee r(\sample , \ell) > \delta(\alpha / k) ) \\
&= \Pr ( r(\sample , \ell) < \delta(\alpha) ) + \Pr( r(\sample , \ell) > \delta(\alpha / k) ) \\
&= \Pr \Bigl( \sum_{i=1}^m \ind{ d_i \geq \delta(\alpha) } < \ell \Bigr) + \Pr \Bigl( \sum_{i=1}^m \ind{ d_i > \delta(\alpha/k) } \geq \ell \Bigr) \\
&\leq \Pr ( \sum_{i=1}^m Y_i < \ell ) + \Pr ( \sum_{i=1}^m X_i \geq \ell ) \\
&= B(\ell - 1 , m , \alpha) + 1 - B(\ell - 1 , m ,  \alpha/k)
\leq \alpha ,
\end{align*}
obtaining the statement. 
\end{proof}

We remark that, for any choice of $k > 1$, there always exist a sufficiently large value of $m$ (and thus, a valid choice of $\ell$), such that the conditions in the statement of Lemma~\ref{lemma:powertimetradeoff} holds.
\fi

\begin{table}%[ht]% h asks to places the floating element [h]ere.
  \caption{Statistics of the datasets considered in our experiments.  
  $n$ is the number of transactions, $d$ is the number of items, $|\dataset|$ is the sum of the transaction lengths, $\ell$ is the average transaction length, $\theta$ is the frequency threshold used to mine frequent itemsets (set similarly to previous works~\cite{preti2024alice}). 
  }
\label{tab:datasets}
\center
  \begin{tabular}{lrrrrc}
    \toprule
    $\dataset$              & $n$      & $d$ & $|\dataset|$  & $\ell$  & $\theta$   \\
    \midrule
    foodmart & 4141 & 1559 & 18319 & 4.4 & $3 \cdot 10^{-4}$ \\
    chess & 3196 & 75 & 118252 & 37.0 & 0.8 \\
	BMS1 & 59602 & 497 & 149639 & 2.5 & $1 \cdot 10^{-3}$ \\
	mushroom     & 8124 & 117 & 178728 & 22.0 & 0.3  \\
	phishing & 11055 & 68 & 331650 & 30.0 & 0.3 \\
	BMS2 & 77512 & 3340 & 358278 & 4.6 & $2 \cdot 10^{-3}$ \\
	retail & 88162 & 16470 & 908576 & 10.3 & $2 \cdot 10^{-3}$ \\
	pumsb-star & 49046 & 2087 & 2448577 & 49.9 & 0.5 \\
    connect & 67557 & 128 & 2870784 & 42.5 & 0.8 \\
    bms-pos & 515420 & 1656 & 3058363 & 5.9 & $1 \cdot 10^{-3}$ \\
    covtype & 581012 & 64 & 6940438 & 11.9 & 0.1 \\
    kosarak & 990002 & 41270 & 8018988 & 8.1 & $1 \cdot 10^{-3}$ \\
    accidents & 340183 & 467 & 11327561 & 33.3 & 0.5 \\
	susy & 5000000 & 178 & 90363253 & 18.1 & 0.1 \\
  \bottomrule
\end{tabular}
\end{table}

\begin{table}%[ht]% h asks to places the floating element [h]ere.
  \caption{Statistics of the graphs considered in our experiments.  
  $|V|$ is the number of vertices, $|E|$ is the number of edges, $|\mathcal{L}|$ is the number of distinct nodes' colors. 
  }
\label{tab:graphs}
\center
  \begin{tabular}{lrrrrrr}
    \toprule
    $G$              & $|V|$      & $|E|$ & $|\mathcal{L}|$    \\
    \midrule
    Cite & 3264 & 4611 & 6 \\
    Brexit & 22745 & 48830 & 2 \\
    Twitter & 22405 & 77920 & 3 \\
    Phy-Cit & 30501 & 347268 & 6 \\
    Abortion & 279505 & 671144 & 2 \\
    US-Elect & 23832 & 845152 & 3 \\
    Com-DBLP & 317080 & 1049866 & 2 \\
    Trivago & 172738 & 1327092 & 160 \\
    Obamacare & 334617 & 1511670 & 2 \\
    Walmart & 88860 & 2267396 & 11 \\
    Com-Youtube & 1134890  & 2987624 & 2 \\
    Comb & 677753 & 6666018 & 2 \\
    Guns & 632659 & 7478993 & 2 \\
    Com-LJ & 3997962  & 34681189 & 2 \\
    Com-Orkut & 3072441  & 117185083 & 2 \\
  \bottomrule
\end{tabular}
\end{table}

\begin{table}%[ht]% h asks to places the floating element [h]ere.
  \caption{Statistics of the labelled datasets considered in our experiments.  
  $n$ is the number of transactions, $d$ is the number of features (categorical/continuous), $\mu(\dataset)$ is the fraction of transactions with target equal to $1$, $z$ is the maximum number of conjunction terms in the language $\lang$. 
  }
\label{tab:datasetsfsr}
\center
  \begin{tabular}{lrrrrrr}
    \toprule
    $\dataset$              & $n$      & $d$ & $\mu(\dataset)$  & $z$     \\
    \midrule
    brain-cancer     & 862 & 22/1  & 0.421 & 5 \\
    theorem-prover     & 3059 & 0/51 & 0.420 & 3  \\
    abalone     & 4177 & 1/7 & 0.663 & 5 \\
    gisette     & 7000 & 0/5000 & 0.500 &  2 \\
    mushroom     & 8124 & 22/0 & 0.482 & 5  \\
    adult     & 32561 & 8/6  & 0.241 & 5 \\
    bank     & 41188 &  10/10 & 0.113 & 3 \\
    kdd-cup & 95370 & 73/405 & 0.050 &  2 \\
    covtype     & 581012 & 0/54  & 0.365 & 3 \\
  \bottomrule
\end{tabular}
\end{table}

\subsection{Normalization}
\label{sec:appendixnorm}
We describe two types of normalizations based on the holdout set of resampled datasets $\sample^\prime$, with $m^\prime = |\sample^\prime|$.
The first type of normalization, that we use in our experiments, is based on the empirical standard score for $A$; 
define $\tilde{\mu}_A(\sample^\prime)$ and $\tilde{\sigma}_A(\sample^\prime)$ as 
\begin{align*}
\tilde{\mu}_A(\sample^\prime) &= \frac{1}{m^\prime} \sum_{\dataset^\prime \in \sample^\prime} A(\dataset^\prime) , \\
\tilde{\sigma}^2_A(\sample^\prime) &= \frac{1}{m^\prime} \sum_{\dataset^\prime \in \sample^\prime} \left( A(\dataset^\prime) - \tilde{\mu}_A(\sample^\prime) \right)^2 ,
\end{align*}
as, respectively, empirical estimates of the expectation and variance of $A$.
The normalized statistic $A^\odot(\dataset)$ is defined as:
\begin{align}
A^\odot(\dataset) = \frac{ A(\dataset) - \tilde{\mu}_A(\sample^\prime) }{ \sqrt{\tilde{\sigma}^2_A(\sample^\prime)} + \tau } , \label{eq:teststatnormalization}
\end{align}
where $\tau$ is a small constant (e.g., $\tau = 0.01$).
The second type uses an estimate of the range of $A$:
\begin{align*}
A^\odot(\dataset) = \frac{ A(\dataset) - \tilde{\mu}_A(\sample^\prime) }{ \max_{\dataset^\prime \in \sample^\prime} A(\dataset^\prime) - \min_{\dataset^\prime \in \sample^\prime} A(\dataset^\prime) + \tau } . 
\end{align*}

\ifextversion
\subsection{Effect of normalization} 
\label{sec:appxnormalization}
In this experiment we evaluated the accuracy of the normalization employed in our experiments. 
Note that this step is needed by both \algname\ and WY only when the number of analyses is $>1$.
To do so, we vary the size $m^\prime$ of the set $\sample^\prime$ of independently generated resamples, showing the behavior of the estimates $\tilde{\mu}_A(\sample^\prime)$ and $\tilde{\sigma}^2_A(\sample^\prime)$, that are, respectively, empirical estimates of the expectation and variance of the statistic $A$ (see Section~\ref{sec:analysis} and Appendix~\ref{sec:appendixnorm}).
We consider the labelled network experiment, in particular the test statistic $M_v$ used for evaluating the fraction of same-color neighbors of the highest degree node $v$, under the Polaris model. 
In Figure~\ref{fig:normalizationparameterspolappendix}, 
we show the results of this experiment.
Each plot shows the behavior of $\tilde{\mu}$ and $\tilde{\sigma}^2$ for $m^\prime \in [5 , 25]$, displaying the mean and std over $10$ runs. 
Overall, we observe that both empirical estimates were very stable; 
therefore, our choice of $m^\prime = 10$ is a good trade-off between accuracy of the estimates and computational overhead. 
\fi

\ifextversion
\subsection{Comparison with FSR}
\label{sec:appxfsr}
This section provides additional details regarding the experimental comparison of \algname\ with FSR.

We are given a dataset $\dataset = \{ (t_1 , \ell_1) , (t_2 , \ell_2) , \dots , (t_n , \ell_n) \}$ with $n$ labelled transactions, 
where each $t_i$ is a realization of $d$ features (each either categorical or continuous), and $\ell_i \in \{ 0 , 1 \}$ is a binary label. 
A pattern $P$ is defined as (at most $z$) conjunctions of logical conditions,
where each condition is over one the features of the dataset. 
We say that a pattern $P$ is observed on a transaction $t_i$ if $P \in t_i$.
The pattern language $\L$ contains all patterns that may be observed on a dataset $\dataset$. 
For each pattern $P \in \L$
we define the test statistic $A_P$, that corresponds to the $1$-quality of the subgroup $P$, which quantifies its association with the target variable~\cite{pellegrina2024efficient}. 
Let $\mu(\dataset) = \frac{1}{n} \sum_{i=1}^n \ell_i$ be the mean target value. 
The test statistic $A_P(\dataset)$ is defined as
\begin{align*}
A_P(\dataset) = \frac{1}{n} \sum_{i=1}^n \ind{ P \in t_i }(\ell_i - \mu(\dataset)) .
\end{align*}
The set of test statistics is $\A = \{ A_P : P \in \L \}$. 
A resampled dataset $\dataset^0$ is drawn from the null distribution $\probdist_0$ by randomly permuting the labels of $\dataset$, i.e., $\dataset^0 = \{ (t_1 , \ell_{\sigma_1}) , (t_2 , \ell_{\sigma_2}) , \dots , (t_n , \ell_{\sigma_n}) \}$, where $\sigma_1 , \sigma_2 , \dots , \sigma_n$ is a uniform random permutation of $[1 , n]$. 

In our experiment we focus on the set $T(k , \dataset)$ of $k$ patterns with highest test statistic on $\dataset$, with $k= 10^5$. 
We report as significant all patterns $P \in T(k , \dataset)$ with $A_P(\dataset) > \hat{\delta}(\sample)$ for \algname,
and all patterns $P \in T(k , \dataset)$ with $A_P(\dataset) > \varepsilon$ for FSR.
\fi

\ifextversion
\begin{figure*}[ht]
\begin{subfigure}{.38\textwidth}
  \centering
  \includegraphics[width=\textwidth]{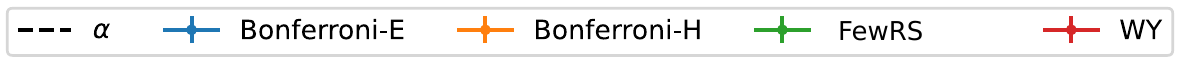}
\end{subfigure} \\
\begin{subfigure}{.245\textwidth}
  \centering
  \includegraphics[width=\textwidth]{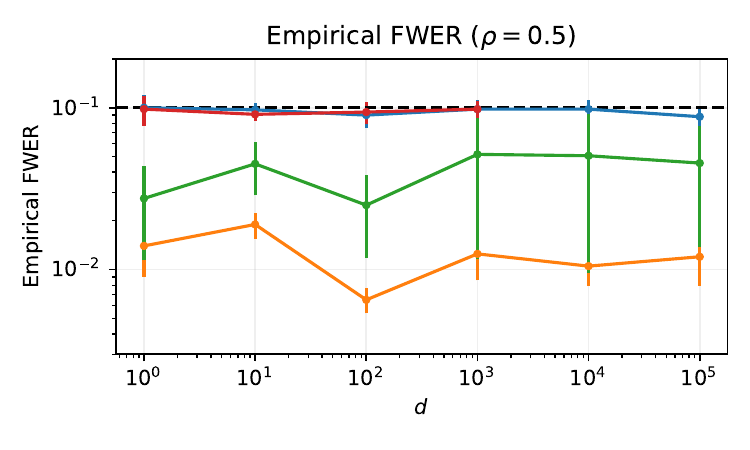}
  %\caption{}
\end{subfigure}
\begin{subfigure}{.245\textwidth}
  \centering
  \includegraphics[width=\textwidth]{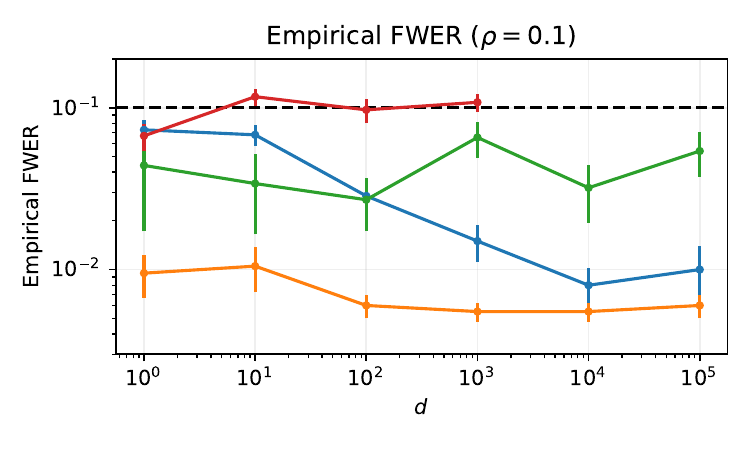}
  %\caption{}
\end{subfigure}
\begin{subfigure}{.245\textwidth}
  \centering
  \includegraphics[width=\textwidth]{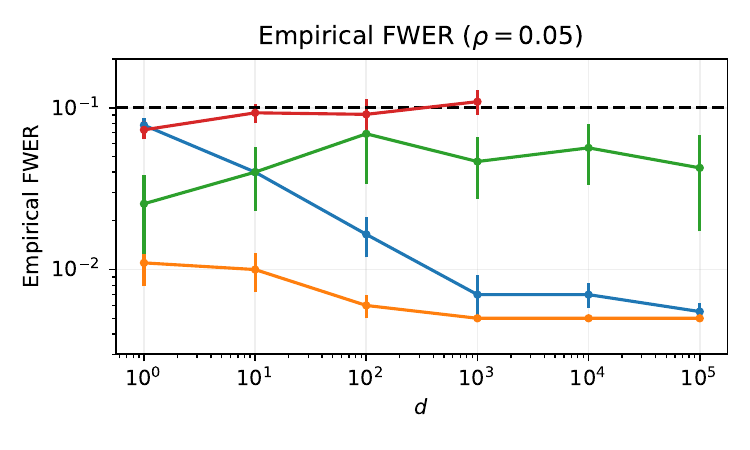}
  %\caption{}
\end{subfigure}
\begin{subfigure}{.245\textwidth}
  \centering
  \includegraphics[width=\textwidth]{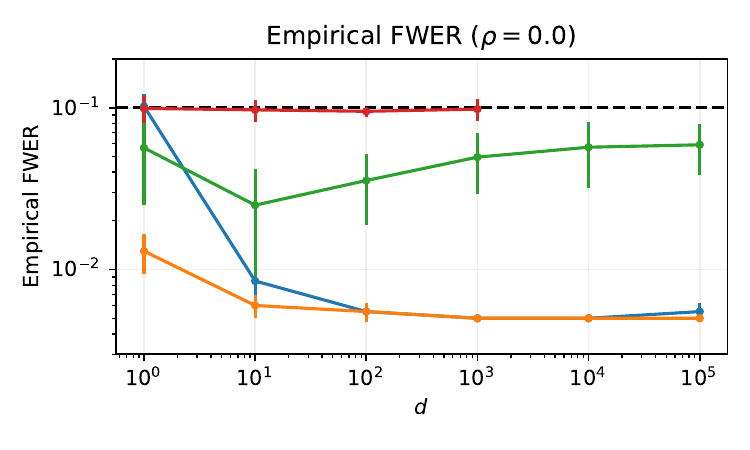}
  %\caption{}
\end{subfigure}
\caption{ Statistical power (empirical FWER) of resampling-based methods (WY and \algname) and Bonferroni w.r.t. the correlation $\rho$ and the number of tests $d$. 
}
\Description{Statistical power (empirical FWER) of resampling-based methods (WY and \algname) and Bonferroni w.r.t. the correlation $\rho$ and the number of tests $d$. }
\label{fig:empiricalfwerfull}
\end{figure*}

\else

\begin{figure}[ht]
\begin{subfigure}{.45\textwidth}
  \centering
  \includegraphics[width=\textwidth]{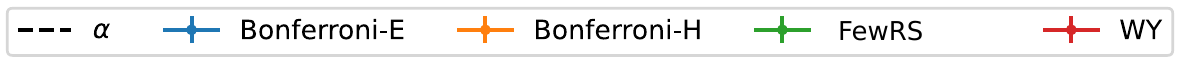}
\end{subfigure} \\ 
\begin{subfigure}{.33\textwidth}
  \centering
  \includegraphics[width=\textwidth]{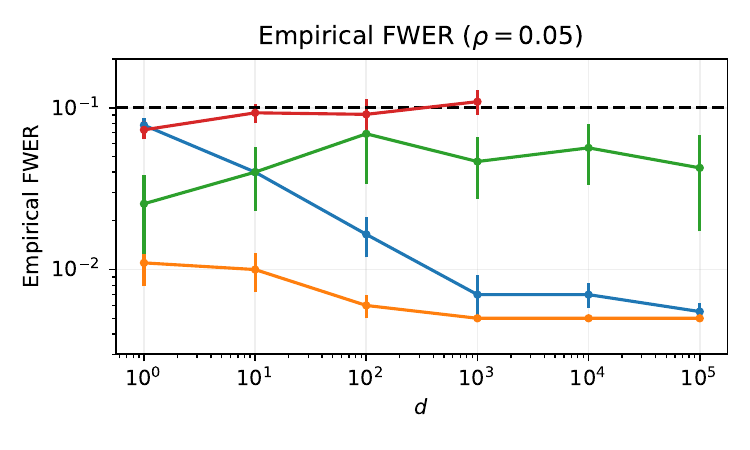}
  %\caption{}
\end{subfigure}
\caption{ Statistical power (empirical FWER) of resampling-based methods (WY and \algname) and Bonferroni w.r.t. the number of tests $d$ (other plots in~\cite{fewrsextended}). 
}
\Description{Statistical power (empirical FWER) of resampling-based methods (WY and \algname) and Bonferroni w.r.t. the number of tests $d$ (other plots in~\cite{fewrsextended}). }
\label{fig:empiricalfwerappendix}
\end{figure}
\fi

\ifextversion
\begin{figure*}[ht]
\begin{subfigure}{.38\textwidth}
  \centering
  \includegraphics[width=\textwidth]{./figures/legend-synth.pdf}
\end{subfigure} \\
\begin{subfigure}{.245\textwidth}
  \centering
  \includegraphics[width=\textwidth]{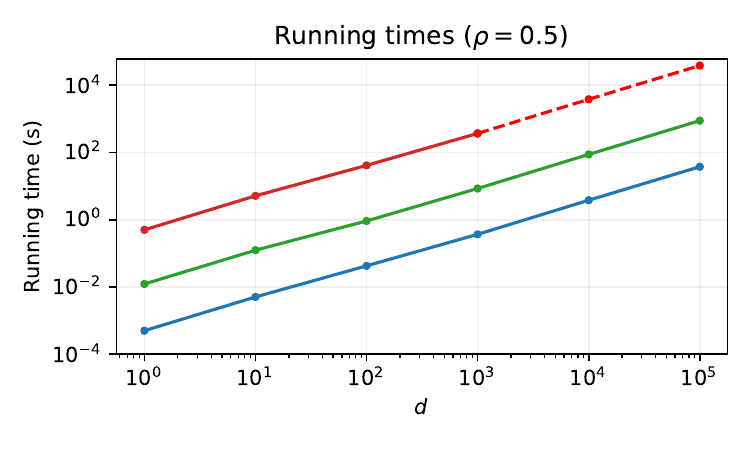}
  %\caption{}
\end{subfigure}
\begin{subfigure}{.245\textwidth}
  \centering
  \includegraphics[width=\textwidth]{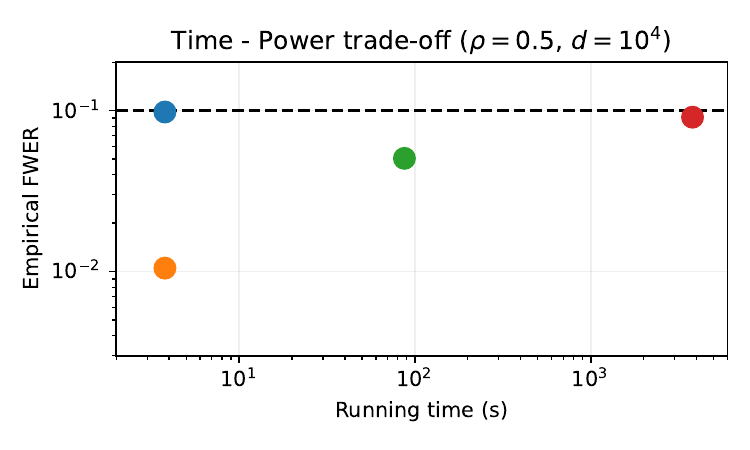}
  %\caption{}
\end{subfigure}
\begin{subfigure}{.245\textwidth}
  \centering
  \includegraphics[width=\textwidth]{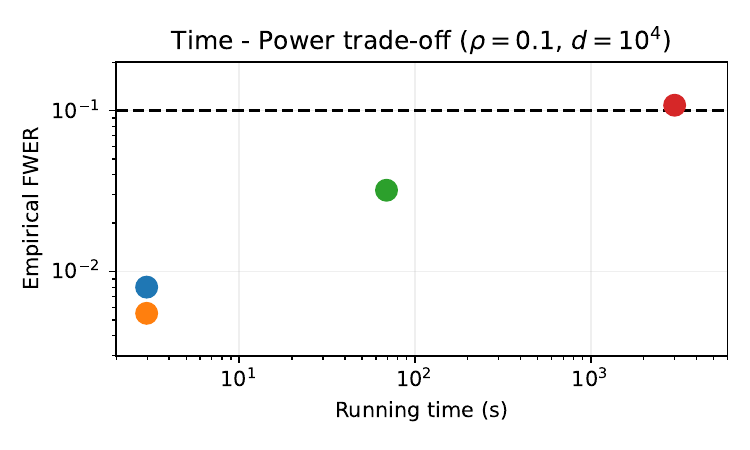}
  %\caption{}
\end{subfigure}
\begin{subfigure}{.245\textwidth}
  \centering
  \includegraphics[width=\textwidth]{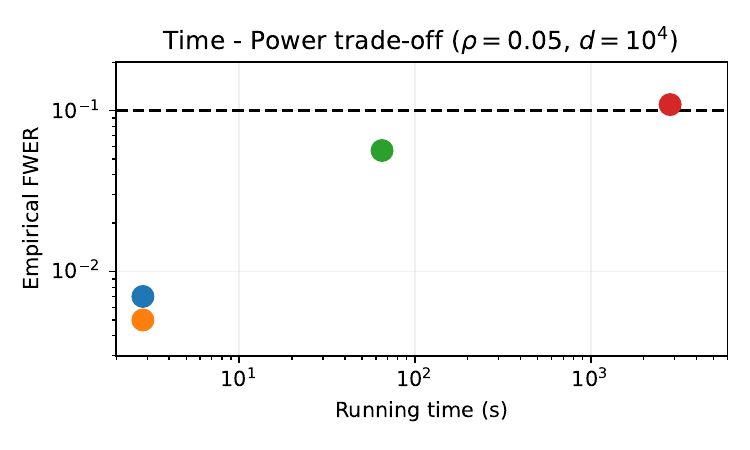}
  %\caption{}
\end{subfigure}
\caption{ Left: running time of resampling-based methods (WY and \algname) and Bonferroni for $\rho=0.5$ (other values are very similar) vs. $d$. 
Other plots: trade-off between running time and statistical power between all methods for different values of $\rho$. 
}
\Description{Left: running time of resampling-based methods (WY and \algname) and Bonferroni for $\rho=0.5$ (other values are very similar) vs. $d$. 
Other plots: trade-off between running time and statistical power between all methods for different values of $\rho$.  }
\label{fig:empiricalfwervstime}
\end{figure*}
\fi

\ifextversion
\begin{figure*}[ht]
\begin{subfigure}{.38\textwidth}
  \centering
  \includegraphics[width=\textwidth]{./figures/legend-synth.pdf}
\end{subfigure} \\
\begin{subfigure}{.245\textwidth}
  \centering
  \includegraphics[width=\textwidth]{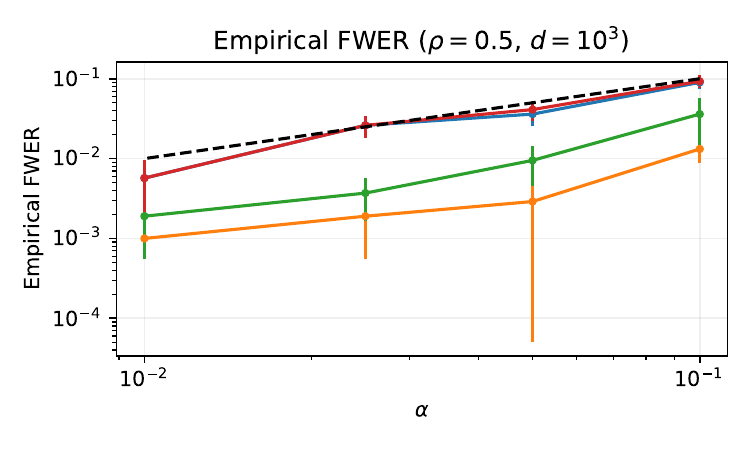}
  %\caption{}
\end{subfigure}
\begin{subfigure}{.245\textwidth}
  \centering
  \includegraphics[width=\textwidth]{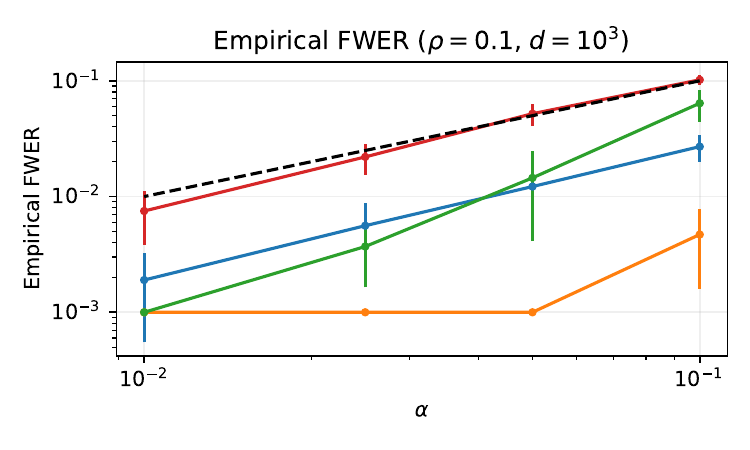}
  %\caption{}
\end{subfigure}
\begin{subfigure}{.245\textwidth}
  \centering
  \includegraphics[width=\textwidth]{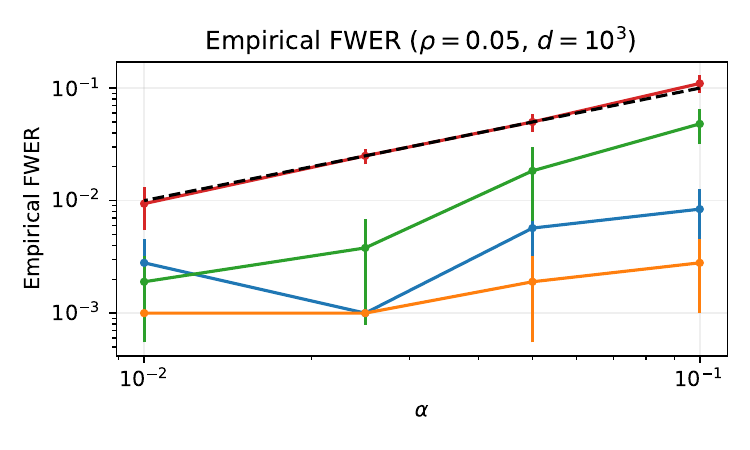}
  %\caption{}
\end{subfigure}
\begin{subfigure}{.245\textwidth}
  \centering
  \includegraphics[width=\textwidth]{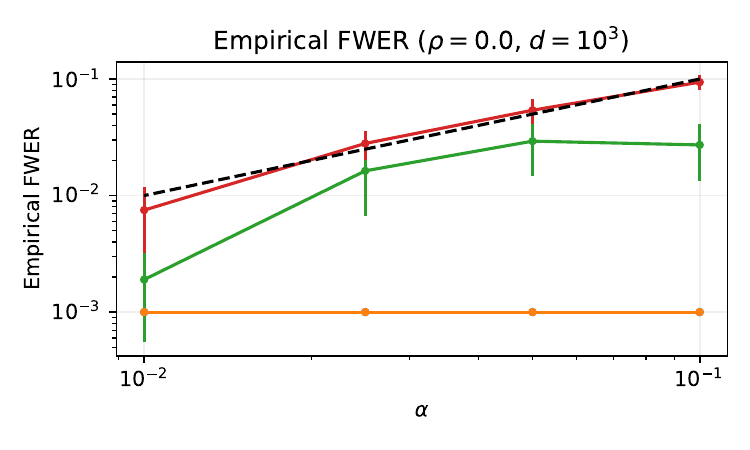}
  %\caption{}
\end{subfigure}
\caption{ Statistical power (empirical FWER) of resampling-based methods (WY and \algname) and Bonferroni w.r.t. the correlation $\rho$ and the FWER bound $\alpha$. 
}
\Description{Statistical power (empirical FWER) of resampling-based methods (WY and \algname) and Bonferroni w.r.t. the correlation $\rho$ and the FWER bound $\alpha$.   }
\label{fig:empiricalfwervsalpha}
\end{figure*}
\fi

\ifextversion
\begin{figure*}[ht]
\begin{subfigure}{.38\textwidth}
  \centering
  \includegraphics[width=\textwidth]{./figures/legend-synth.pdf}
\end{subfigure} \\
\begin{subfigure}{.245\textwidth}
  \centering
  \includegraphics[width=\textwidth]{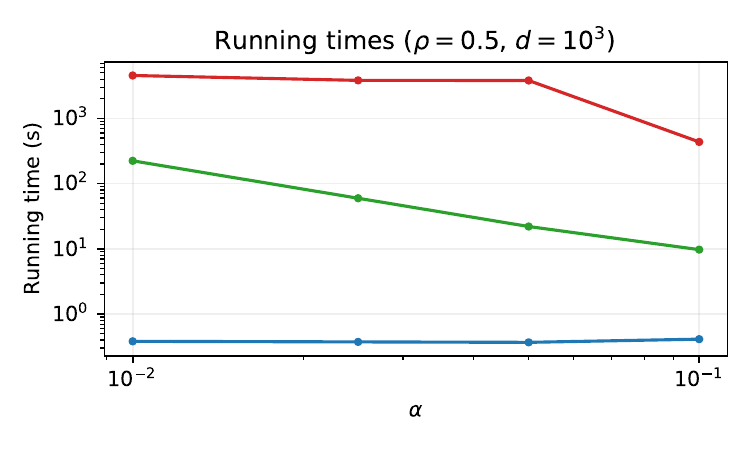}
  %\caption{}
\end{subfigure}
\begin{subfigure}{.245\textwidth}
  \centering
  \includegraphics[width=\textwidth]{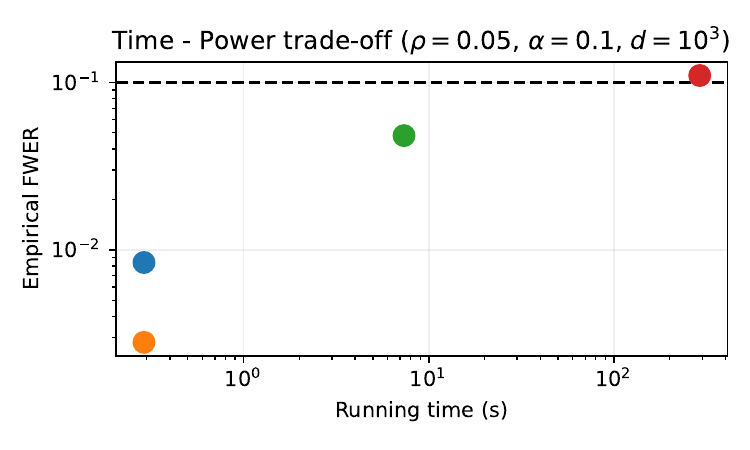}
  %\caption{}
\end{subfigure}
\begin{subfigure}{.245\textwidth}
  \centering
  \includegraphics[width=\textwidth]{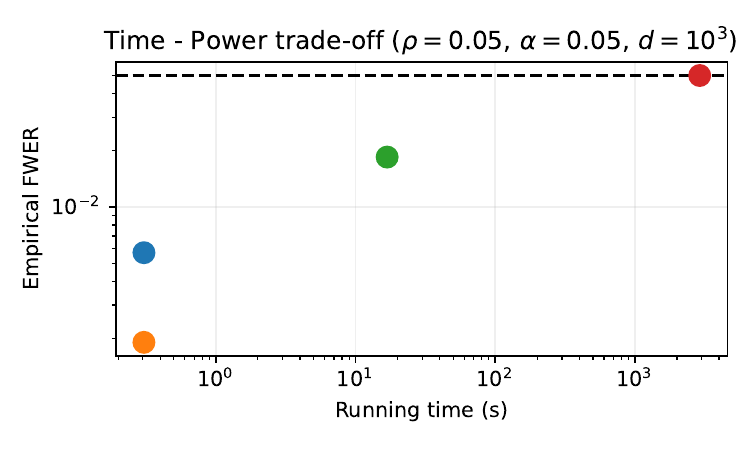}
  %\caption{}
\end{subfigure}
\begin{subfigure}{.245\textwidth}
  \centering
  \includegraphics[width=\textwidth]{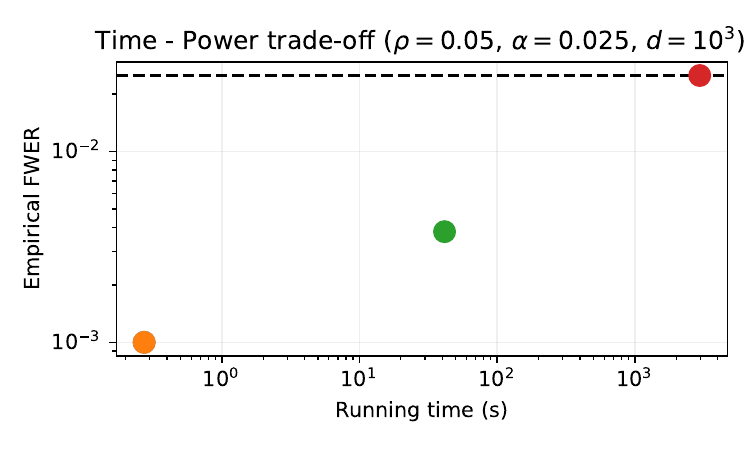}
  %\caption{}
\end{subfigure}
\caption{ Left: running time of resampling-based methods (WY and \algname) and Bonferroni for $\rho=0.5$ (other values are very similar) vs. $\alpha$. 
WY uses $10^3$ resamples for $\alpha=0.1$, and $10^4$ for other values. 
Other plots: trade-off between running time and statistical power between all methods for different values of $\alpha$. 
}
\Description{Left: running time of resampling-based methods (WY and \algname) and Bonferroni for $\rho=0.5$ (other values are very similar) vs. $\alpha$. 
WY uses $10^3$ resamples for $\alpha=0.1$, and $10^4$ for other values. 
Other plots: trade-off between running time and statistical power between all methods for different values of $\alpha$.   }
\label{fig:empiricalfwervstimealpha}
\end{figure*}
\fi

\ifextversion
\begin{figure*}[ht]
\begin{subfigure}{.5085\textwidth}
  \centering
  \resizebox{.98\columnwidth}{!}{%
   \begin{tabular}{lrrrrrrr}
    \toprule
    $\dataset$       &   $|FI(\dataset , \theta)|$    &   GMMT: \hspace{7px} $\overline{\text{FI}}$  & $\hat{\delta}(\sample)$  &  ALICE: \hspace{7px} $\overline{\text{FI}}$  & $\hat{\delta}(\sample)$     \\
    \midrule
    pumsb-star  &  615  &    83.2  &   84 &   87.0  &  87  \\ 
mushrooms  &  1869  &    932.3  &  949 &   928.1  &  941 \\ 
  covtype  &  2020   &   1884.6  &   1895 &   1928.0  &  1934 \\
retail  &  2691  &    2918.1  &  2950 &   2352.7  &  2376 \\ 
BMS2  &  3683  &    600.3  &  606 &  613.6  &  619 \\ 
BMS1  &  3991  &    1808.6  &  1905 &  1618.5  &  1654 \\ 
foodmart  &  4247    &  2221.9  &  2266 &   2228.9  &  2279 \\ 
accidents  &  8049    &  7876.4  &  7888 &   6054.6  &  6061  \\ 
chess  &  8227  &    6187.6  &  6332 &   6188.1  &  6401 \\ 
susy  &  20178  &    16267.6  &  16327 &  13525.6  &  13540 \\ 
bms-pos  &  62642    &  56105.9  &  56876 &    51589.2  &  51954 \\ 
connect  &  533975    &  421387.5  &  431110 &   385061.6  &  388014  \\ 
kosarak  &  759400   &   463243.0  &  490279 &   33617.9  &  33771  \\
phishing  &  1469655  &    110872.0 &  112690 &  110757.0  &  113458 \\ 
  \bottomrule
\end{tabular}
}
  %\caption{}
\end{subfigure}
\hspace{11px}
\begin{subfigure}{.46\textwidth}
  \centering
  \resizebox{.98\columnwidth}{!}{%
  \begin{tabular}{llllllll}
    \toprule
    $G$       &   $M(G)$    &  ~~CM:      & $\overline{M}$  & $\hat{\delta}(\sample)$  & Polaris: & $\overline{M}$  & $\hat{\delta}(\sample)$     \\
    \midrule
    Phy-Cit  &  0.171  &  &  0.11  &   0.112 & &  0.161  &  0.162  \\ 
Walmart  &  0.487  &  &  0.224  &   0.225 & &  0.438  &  0.439 \\ 
  Com-Orkut  &  0.714   & &  0.6  &   0.6 & &  0.771  &  0.771 \\
Cite  &  0.721  &  &  0.183  &  0.2 & &  0.728  &  0.738 \\ 
US-Elect  &  0.775  &  &  0.534  &  0.538 & & 0.783  &  0.785 \\ 
Com-DBLP  &  0.775  &  &  0.524  &  0.525 & & 0.762  &  0.763 \\ 
Com-Youtube  &  0.816  &  &  0.76  &  0.76 & &  0.795  &  0.795 \\ 
Twitter  &  0.827  &  &  0.393  &  0.397 & &  0.786  &  0.792  \\ 
Guns  &  0.844  &  &  0.563  &  0.564 & & 0.675  &  0.675 \\ 
Comb  &  0.882  &  &  0.653  &  0.654 & & 0.74  &  0.741 \\ 
Brexit  &  0.966  &  &  0.506  &  0.516 & & 0.85  &  0.854 \\ 
Trivago  &  0.975  &  &  0.044  &  0.045 & &  0.972  &  0.973  \\ 
Com-LJ  &  0.979   & &  0.918  &  0.918 & & 0.972  &  0.972  \\
Obamacare  &  0.992  &  &  0.496  &  0.498 &&  0.977  &  0.978 \\ 
Abortion  &  0.993  &  &  0.497  &  0.5 & &  0.978  &  0.978 \\
  \bottomrule
\end{tabular}
}
  %\caption{}
\end{subfigure}
\caption{ 
Left: results for FI.
Right: results for labelled networks.  
}
\Description{
Left: results for FI.
Right: results for labelled networks.  }
\label{fig:restables}
\end{figure*}
\fi

\ifextversion
\begin{figure*}[ht]
\begin{subfigure}{.8\textwidth}
  \centering
  \includegraphics[width=\textwidth]{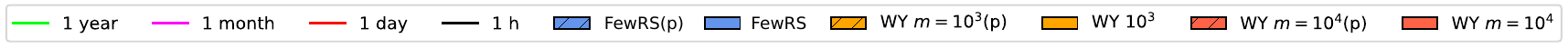}
\end{subfigure} \\
\begin{subfigure}{.28\textwidth}
  \centering
  \includegraphics[width=\textwidth]{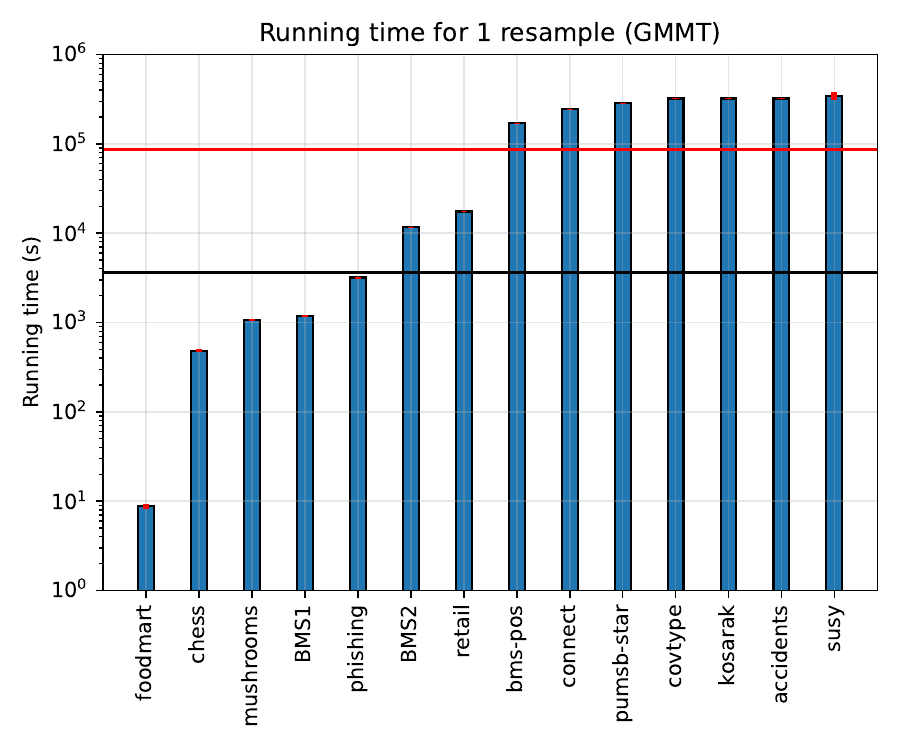}
  %\caption{}
\end{subfigure}
\begin{subfigure}{.28\textwidth}
  \centering
  \includegraphics[width=\textwidth]{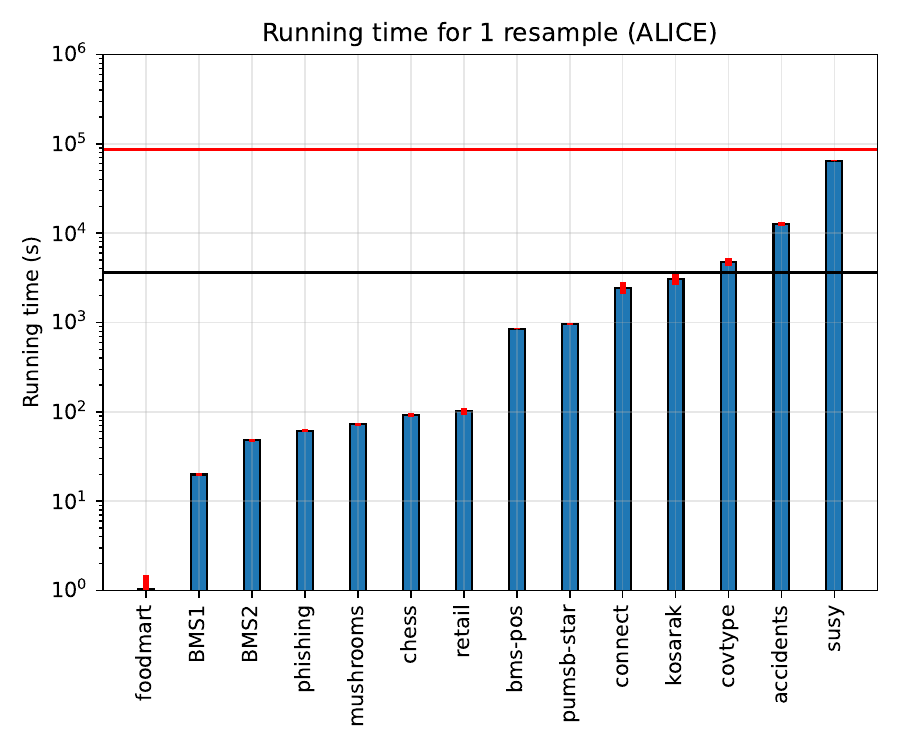}
  %\caption{}
\end{subfigure}
\begin{subfigure}{.43\textwidth}
  \centering
  \includegraphics[width=\textwidth]{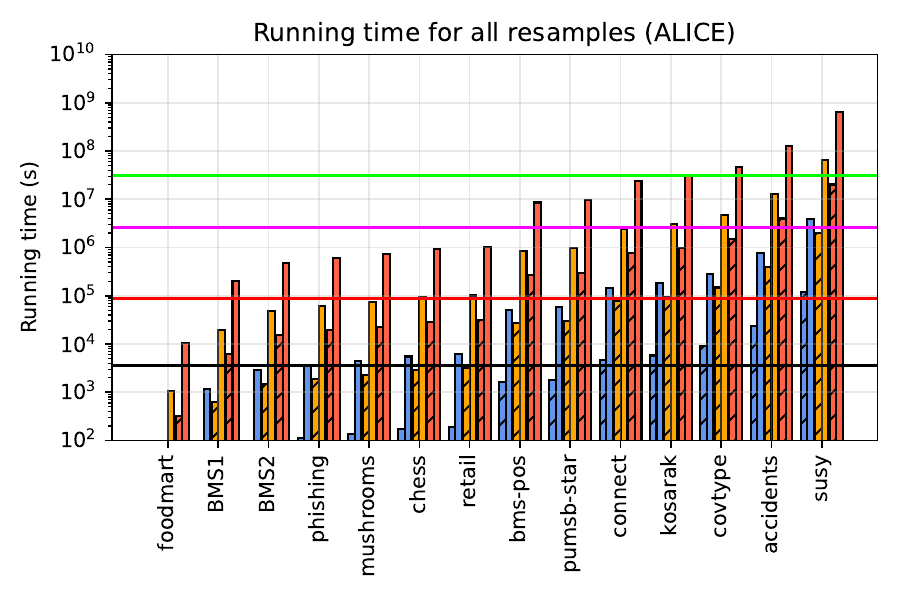}
  %\caption{}
\end{subfigure}
\caption{ 
Left: time to generate one sample from the GMMT and ALICE models. 
Right: time to generate and analyze all samples with \algname\ and WY under ALICE model. The label (p) denotes parallelized approaches with $32$ cores.
}
\Description{
Left: time to generate one sample from the GMMT and ALICE models. 
Right: time to generate and analyze all samples with \algname\ and WY under ALICE model. The label (p) denotes parallelized approaches with $32$ cores. }
\label{fig:runningtimealice}
\end{figure*}

\else

\begin{figure}[ht]
\begin{subfigure}{.14\textwidth}
  \centering
  \includegraphics[width=\textwidth]{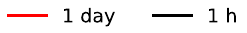}
\end{subfigure} \\
\begin{subfigure}{.33\textwidth}
  \centering
  \includegraphics[width=\textwidth]{./figures/running-time-1-GMMT.pdf}
  %\caption{}
\end{subfigure}
\caption{ 
Time to generate one sample from GMMT model. 
}
\Description{Time to generate one sample from GMMT model. }
\label{fig:runningtimegmmtsingle}
\end{figure}

\fi

\ifextversion
\begin{figure*}[ht]
\begin{subfigure}{.75\textwidth}
  \centering
  \includegraphics[width=\textwidth]{./figures/time-legend.pdf}
\end{subfigure} \\
\begin{subfigure}{.36\textwidth}
  \centering
  \includegraphics[width=\textwidth]{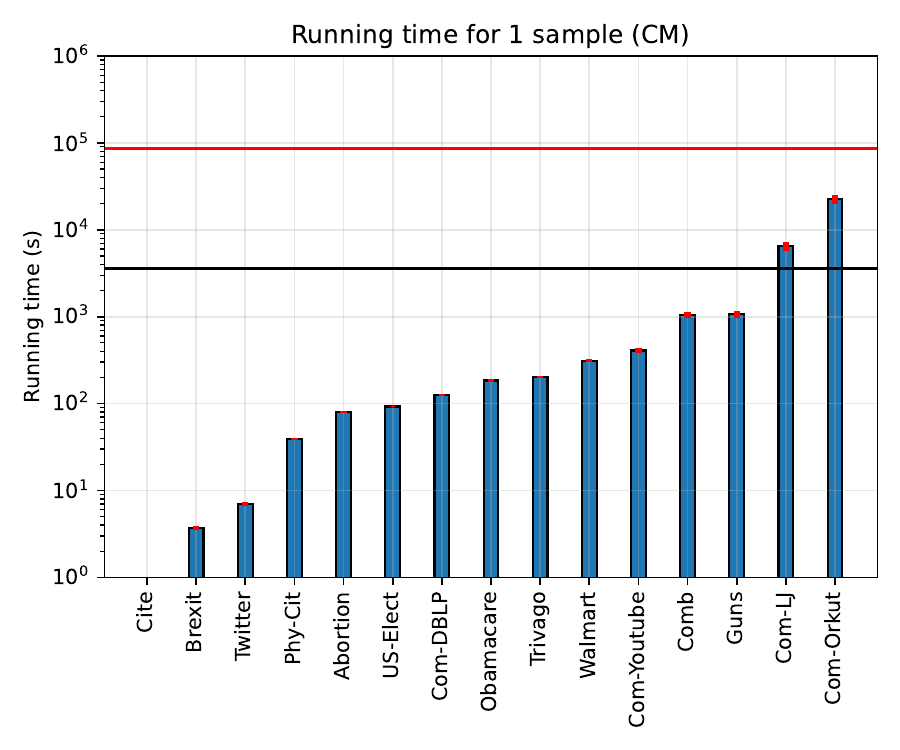}
  %\caption{}
\end{subfigure}
\begin{subfigure}{.475\textwidth}
  \centering
  \includegraphics[width=\textwidth]{./figures/running-time-all-CM.pdf}
  %\caption{}
\end{subfigure}
\caption{ 
Running times for the CM model.
Left: time to generate one sample. 
Right: time to generate and analyze all samples with \algname\ and WY. 
The label (p) denotes parallelized approaches with up to $32$ cores. 
}
\Description{Running times for the CM model.
Left: time to generate one sample. 
Right: time to generate and analyze all samples with \algname\ and WY. 
The label (p) denotes parallelized approaches with up to $32$ cores. }
\label{fig:runningtimecm}
\end{figure*}
\fi

\ifextversion
\begin{figure*}[ht]
\begin{subfigure}{.75\textwidth}
  \centering
  \includegraphics[width=\textwidth]{./figures/time-legend.pdf}
\end{subfigure} \\
\begin{subfigure}{.36\textwidth}
  \centering
  \includegraphics[width=\textwidth]{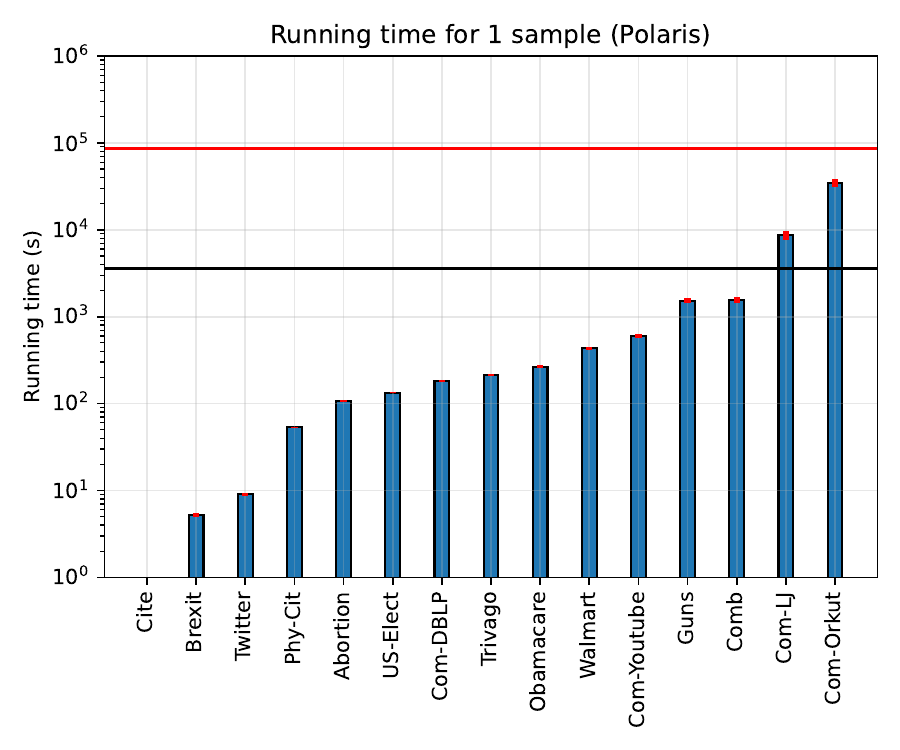}
  %\caption{}
\end{subfigure}
\begin{subfigure}{.475\textwidth}
  \centering
  \includegraphics[width=\textwidth]{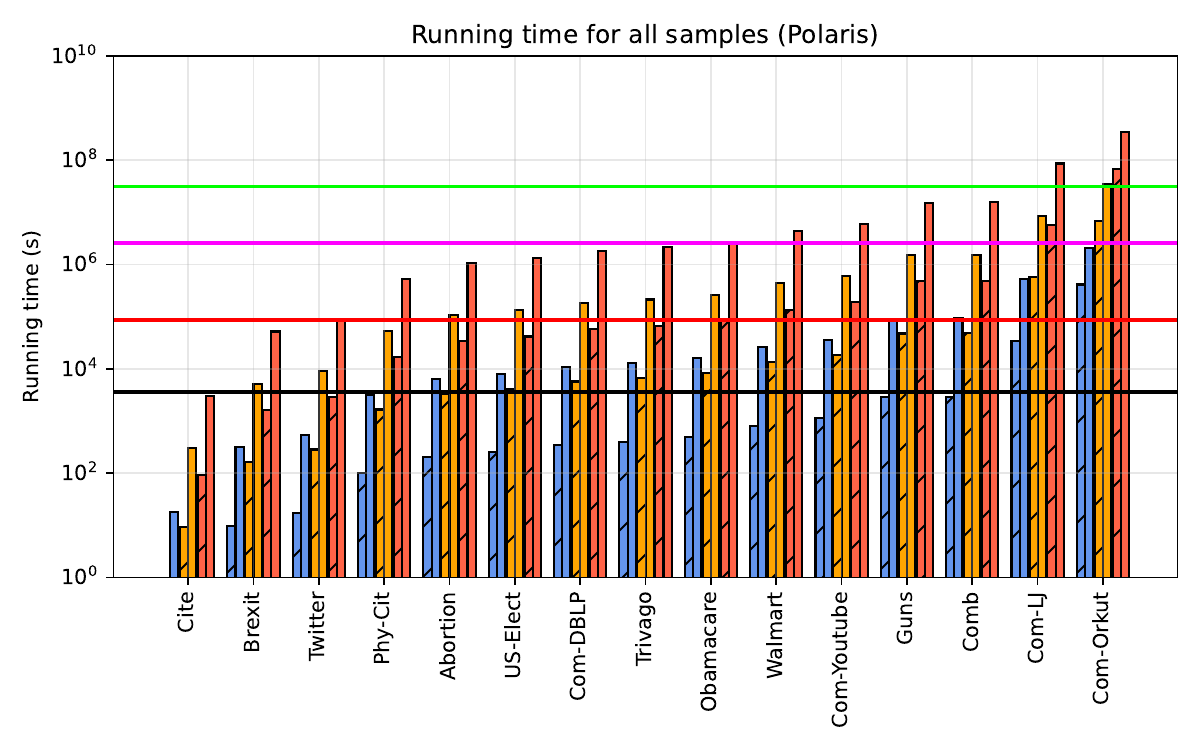}
  %\caption{}
\end{subfigure}
\caption{ 
Running times for the Polaris model.
Left: time to generate one sample. 
Right: time to generate and analyze all samples with \algname\ and WY. 
The label (p) denotes parallelized approaches with $32$ cores, apart from Com-LJ and Com-Orkut, where only $15$ and $5$ parallel executions can be performed due to high memory usage.
}
\Description{
Running times for the Polaris model.
Left: time to generate one sample. 
Right: time to generate and analyze all samples with \algname\ and WY. 
The label (p) denotes parallelized approaches with $32$ cores, apart from Com-LJ and Com-Orkut, where only $15$ and $5$ parallel executions can be performed due to high memory usage.}
\label{fig:runningtimepol}
\end{figure*}
\fi

\ifextversion
\begin{figure*}[ht]
\begin{subfigure}{.175\textwidth}
  \centering
  \includegraphics[width=\textwidth]{./figures/power-legend.pdf}
\end{subfigure} \\
\begin{subfigure}{.275\textwidth}
  \centering
  \includegraphics[width=\textwidth]{./figures/num-fi-len-retail-GMMT.pdf} 
\end{subfigure}
\begin{subfigure}{.275\textwidth}
  \centering
  \includegraphics[width=\textwidth]{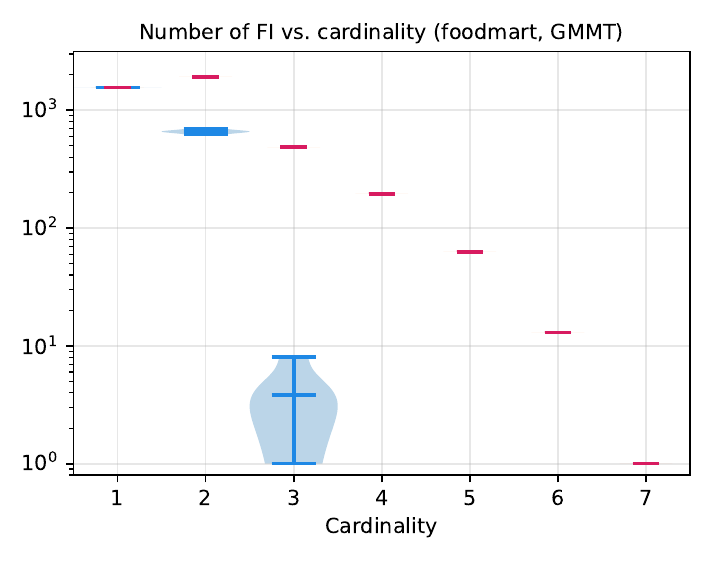} 
\end{subfigure}
\begin{subfigure}{.275\textwidth}
  \centering
  \includegraphics[width=\textwidth]{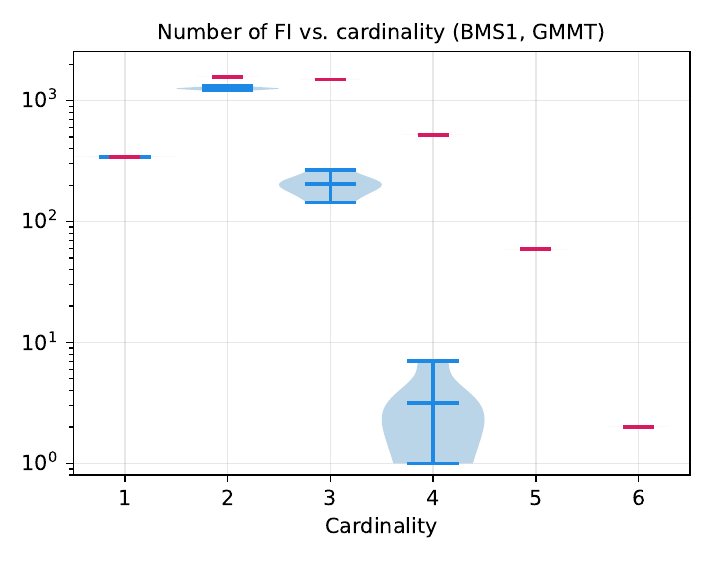} 
\end{subfigure}
\begin{subfigure}{.275\textwidth}
  \centering
  \includegraphics[width=\textwidth]{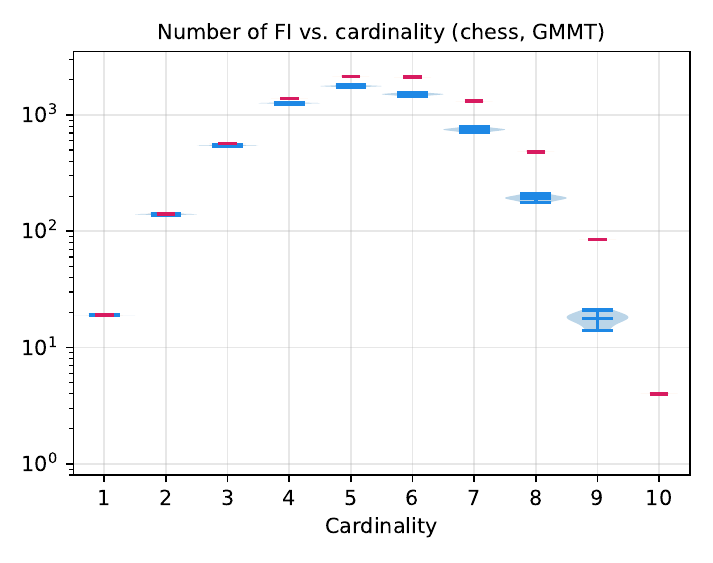} 
\end{subfigure}
\begin{subfigure}{.275\textwidth}
  \centering
  \includegraphics[width=\textwidth]{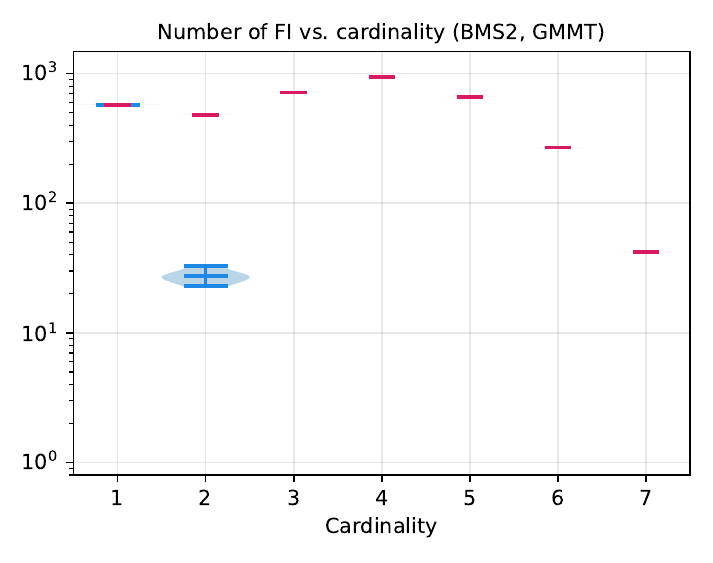} 
\end{subfigure}
\begin{subfigure}{.275\textwidth}
  \centering
  \includegraphics[width=\textwidth]{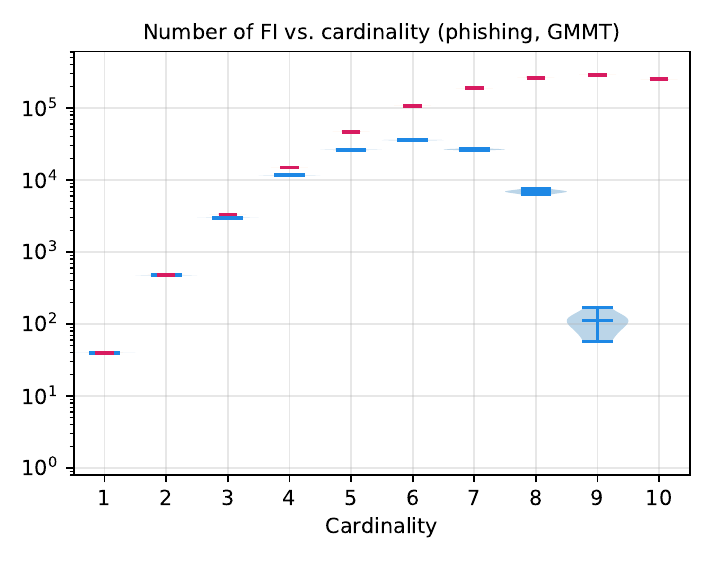} 
\end{subfigure}
\begin{subfigure}{.275\textwidth}
  \centering
  \includegraphics[width=\textwidth]{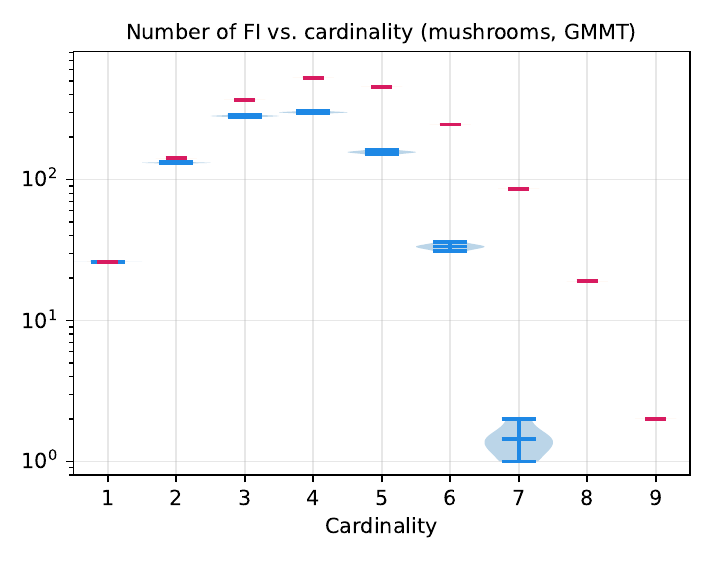} 
\end{subfigure}
\begin{subfigure}{.275\textwidth}
  \centering
  \includegraphics[width=\textwidth]{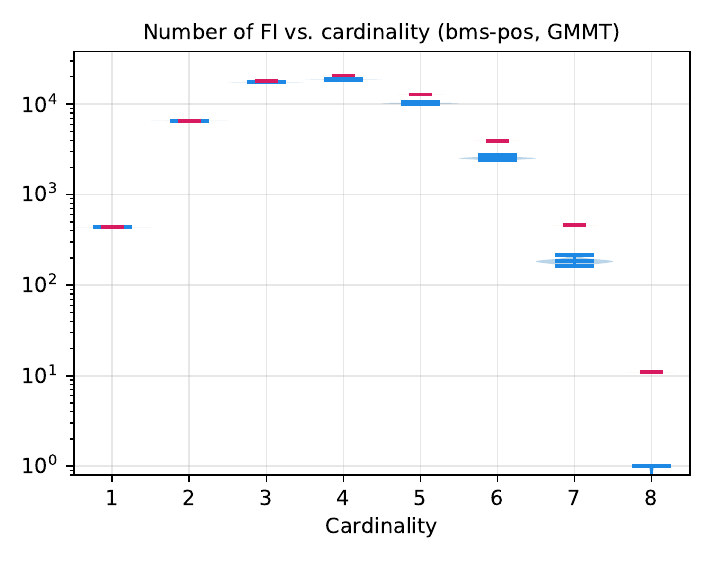} 
\end{subfigure}
\begin{subfigure}{.275\textwidth}
  \centering
  \includegraphics[width=\textwidth]{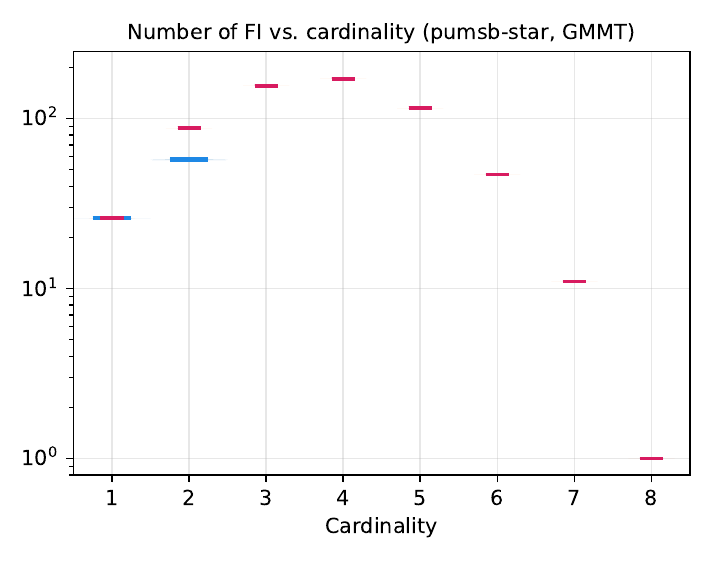} 
\end{subfigure}
\begin{subfigure}{.275\textwidth}
  \centering
  \includegraphics[width=\textwidth]{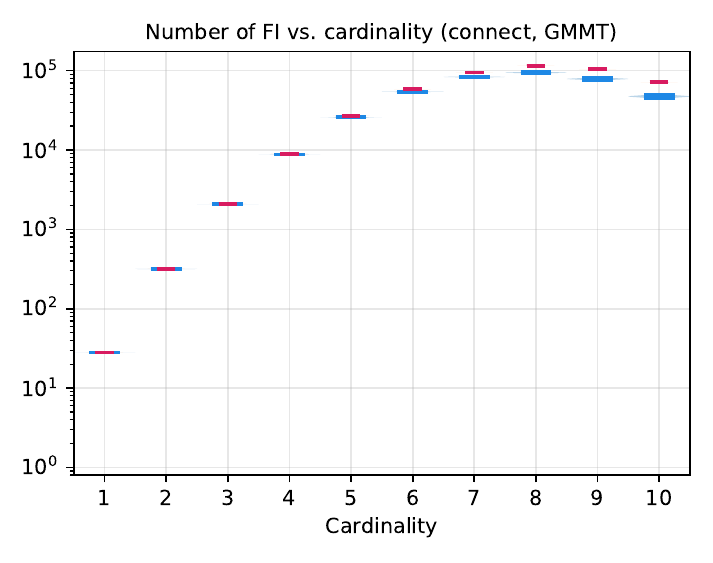} 
\end{subfigure}
\begin{subfigure}{.275\textwidth}
  \centering
  \includegraphics[width=\textwidth]{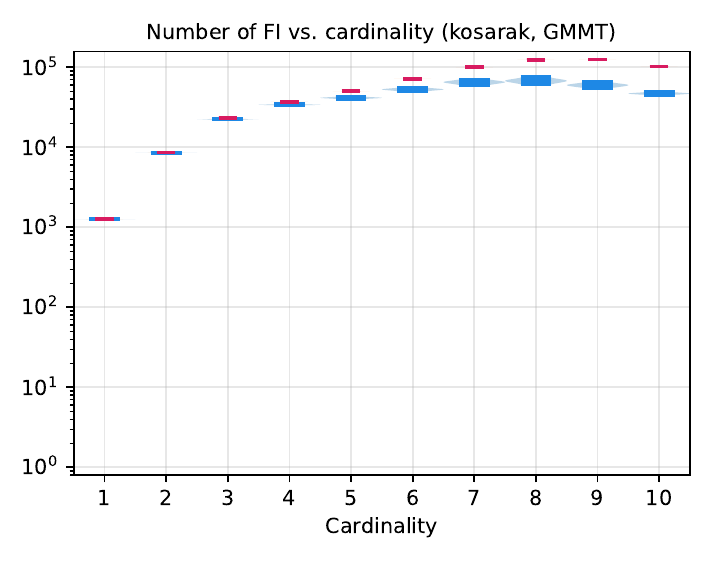} 
\end{subfigure}
\begin{subfigure}{.275\textwidth}
  \centering
  \includegraphics[width=\textwidth]{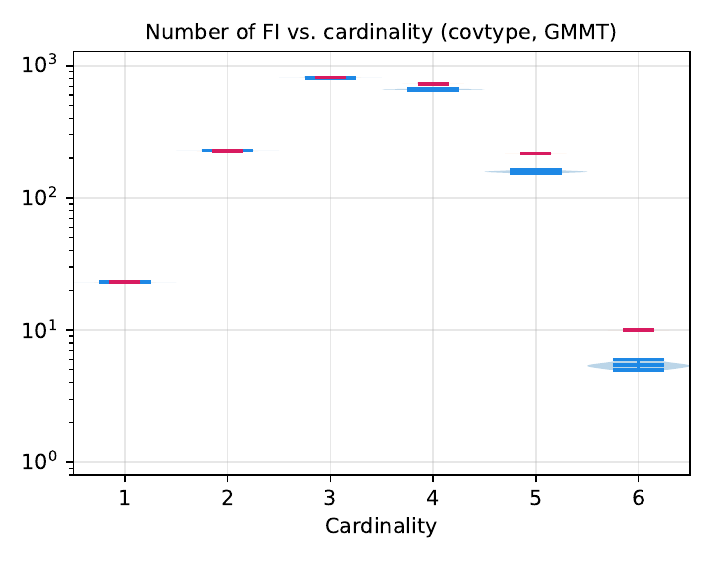} 
\end{subfigure}
\begin{subfigure}{.275\textwidth}
  \centering
  \includegraphics[width=\textwidth]{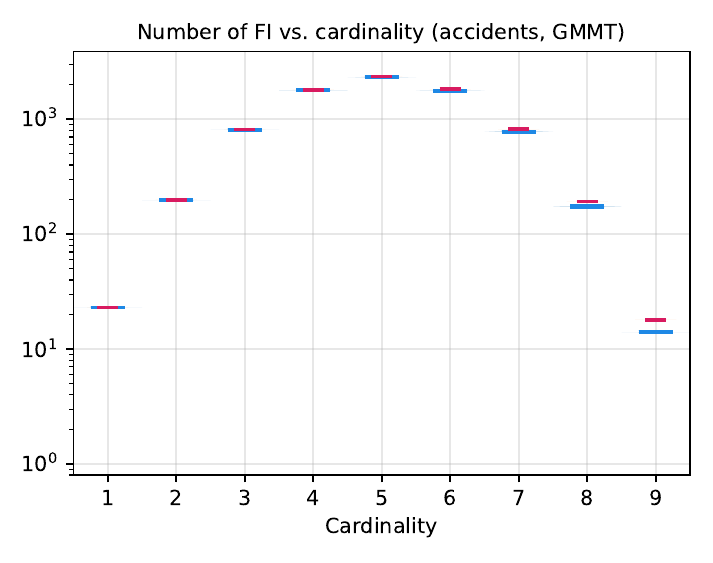} 
\end{subfigure}
\begin{subfigure}{.275\textwidth}
  \centering
  \includegraphics[width=\textwidth]{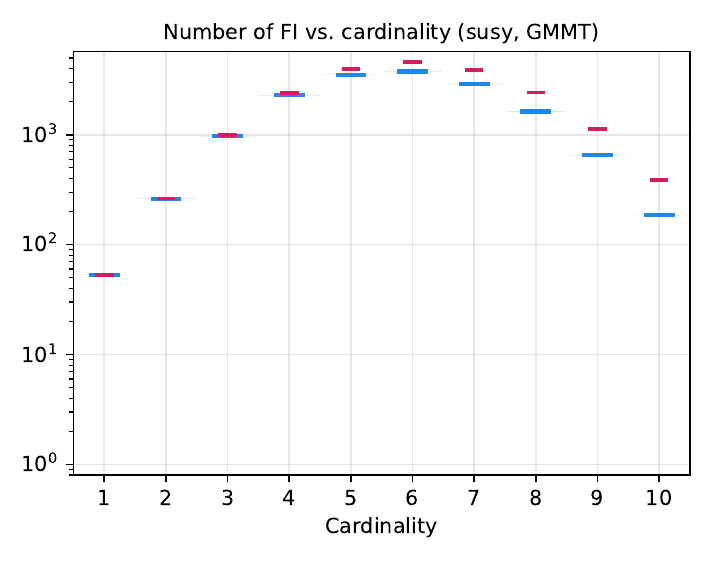} 
\end{subfigure}
\caption{ 
Testing the significance of the number of FI for different cardinalities under the GMMT model.
}
\Description{
Testing the significance of the number of FI for different cardinalities under the GMMT model.}
\label{fig:numfilengmmtappendix}
\end{figure*}
\fi

\ifextversion
\begin{figure*}[ht]
\begin{subfigure}{.175\textwidth}
  \centering
  \includegraphics[width=\textwidth]{./figures/power-legend.pdf}
\end{subfigure} \\
\begin{subfigure}{.275\textwidth}
  \centering
  \includegraphics[width=\textwidth]{./figures/num-fi-len-retail-norm-GMMT.pdf} 
\end{subfigure}
\begin{subfigure}{.275\textwidth}
  \centering
  \includegraphics[width=\textwidth]{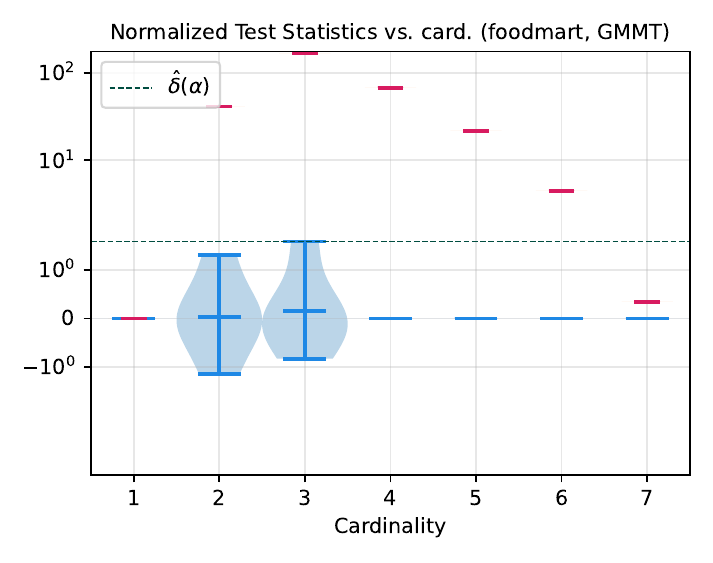} 
\end{subfigure}
\begin{subfigure}{.275\textwidth}
  \centering
  \includegraphics[width=\textwidth]{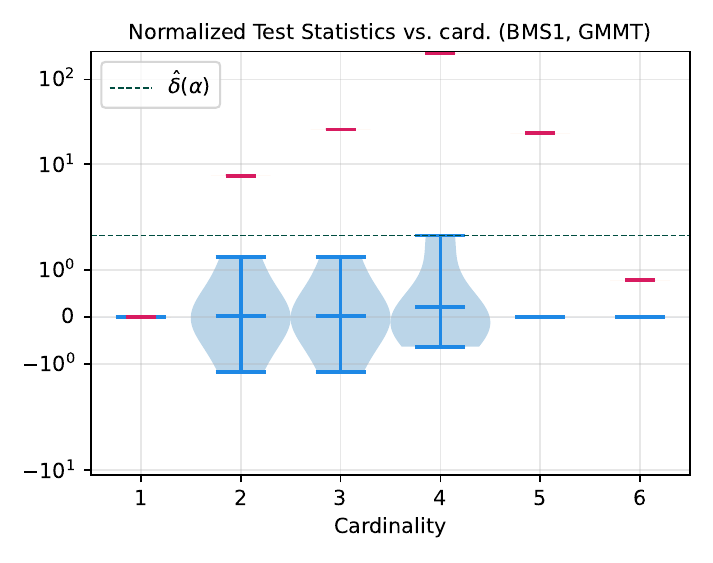} 
\end{subfigure}
\begin{subfigure}{.275\textwidth}
  \centering
  \includegraphics[width=\textwidth]{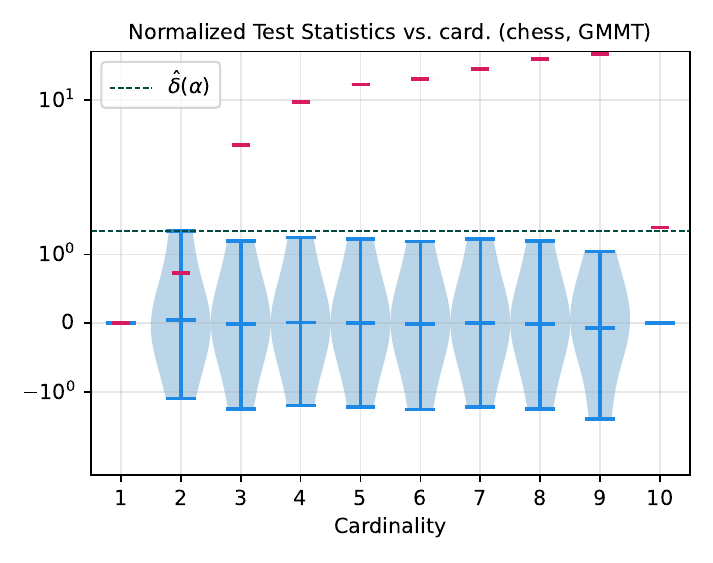} 
\end{subfigure}
\begin{subfigure}{.275\textwidth}
  \centering
  \includegraphics[width=\textwidth]{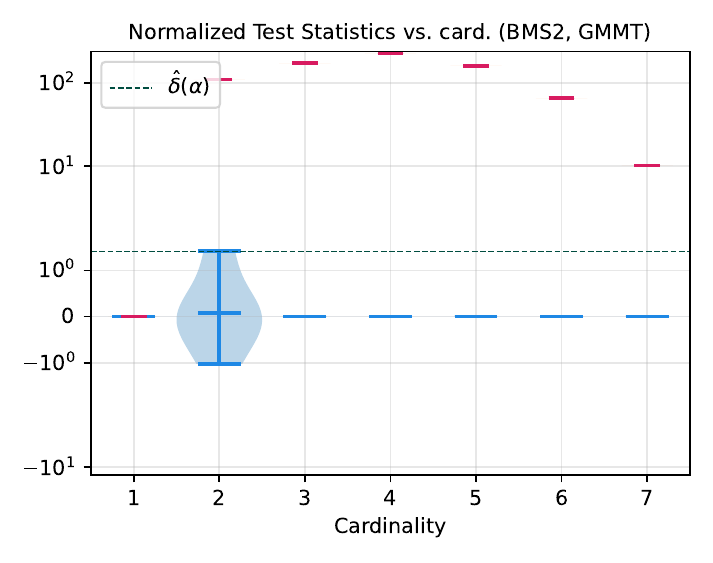} 
\end{subfigure}
\begin{subfigure}{.275\textwidth}
  \centering
  \includegraphics[width=\textwidth]{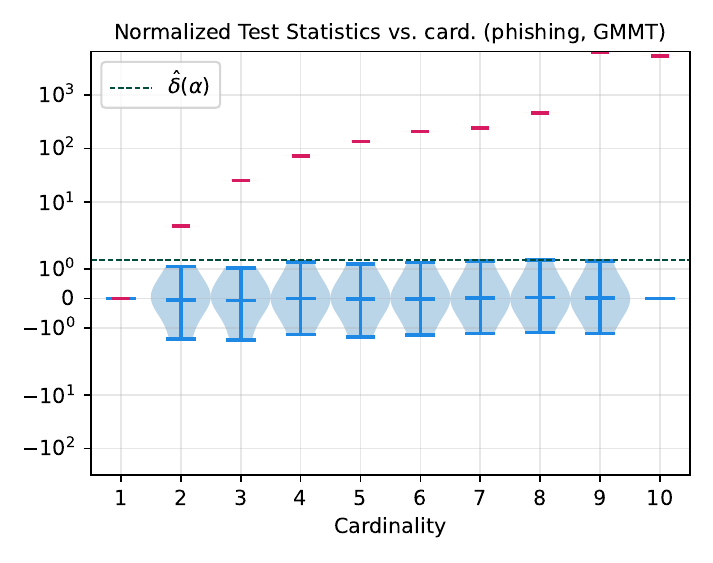} 
\end{subfigure}
\begin{subfigure}{.275\textwidth}
  \centering
  \includegraphics[width=\textwidth]{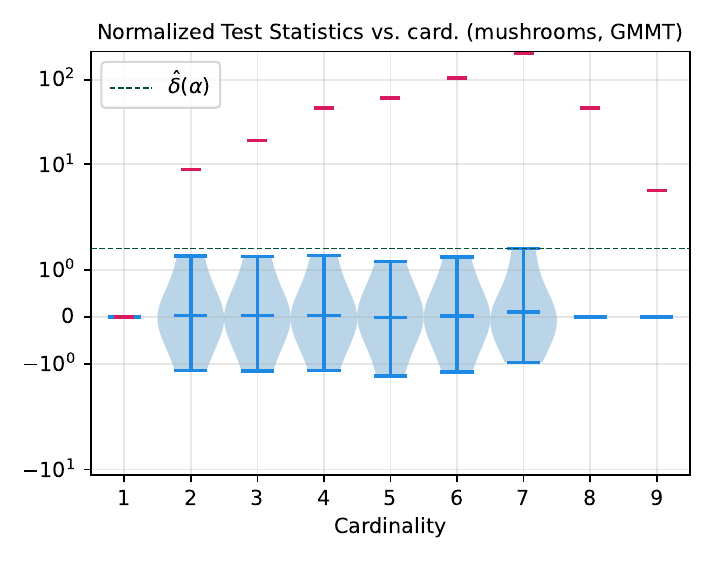} 
\end{subfigure}
\begin{subfigure}{.275\textwidth}
  \centering
  \includegraphics[width=\textwidth]{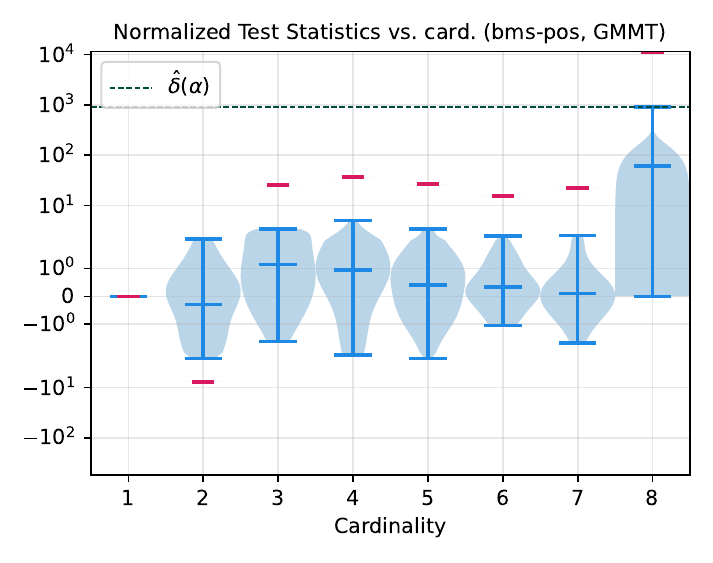} 
\end{subfigure}
\begin{subfigure}{.275\textwidth}
  \centering
  \includegraphics[width=\textwidth]{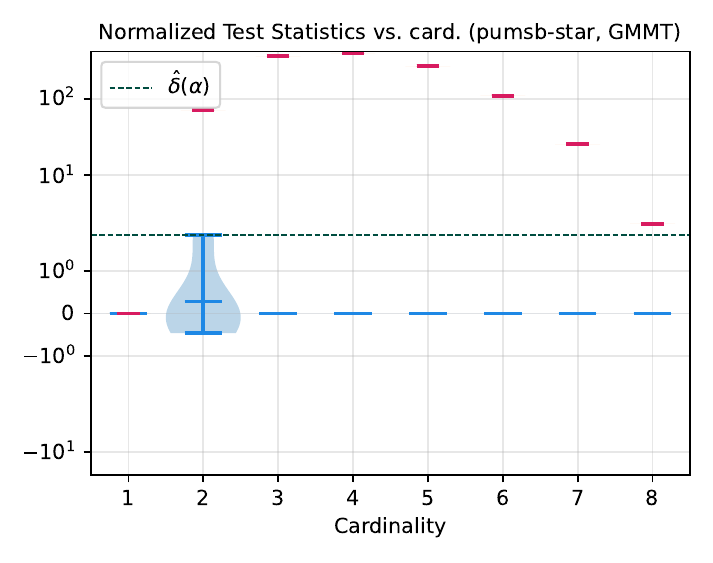} 
\end{subfigure}
\begin{subfigure}{.275\textwidth}
  \centering
  \includegraphics[width=\textwidth]{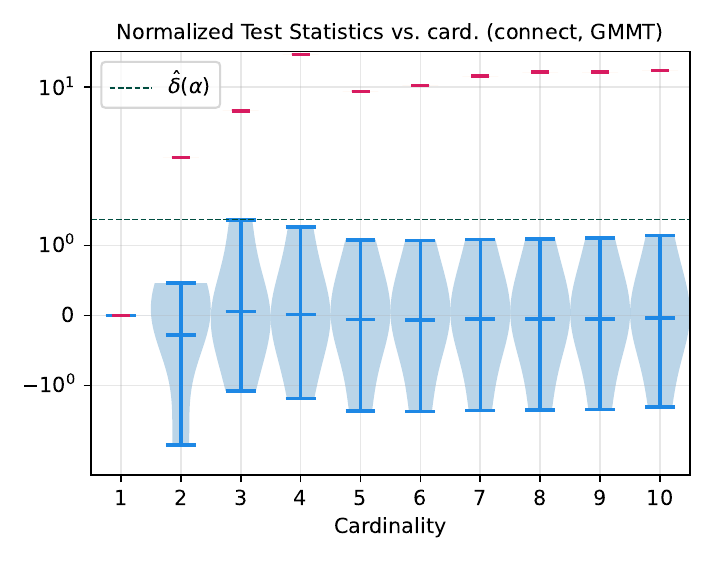} 
\end{subfigure}
\begin{subfigure}{.275\textwidth}
  \centering
  \includegraphics[width=\textwidth]{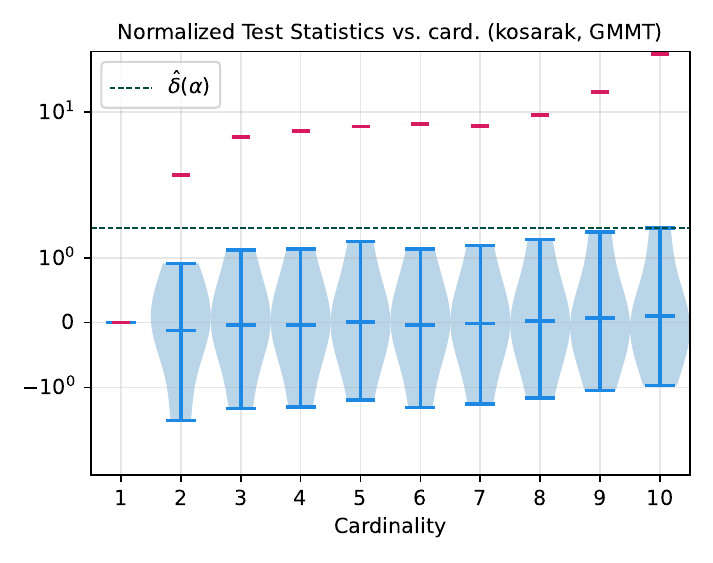} 
\end{subfigure}
\begin{subfigure}{.275\textwidth}
  \centering
  \includegraphics[width=\textwidth]{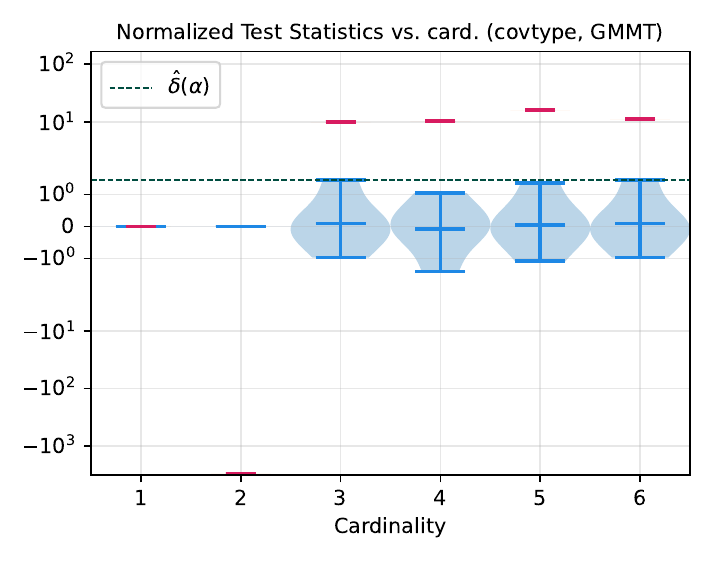} 
\end{subfigure}
\begin{subfigure}{.275\textwidth}
  \centering
  \includegraphics[width=\textwidth]{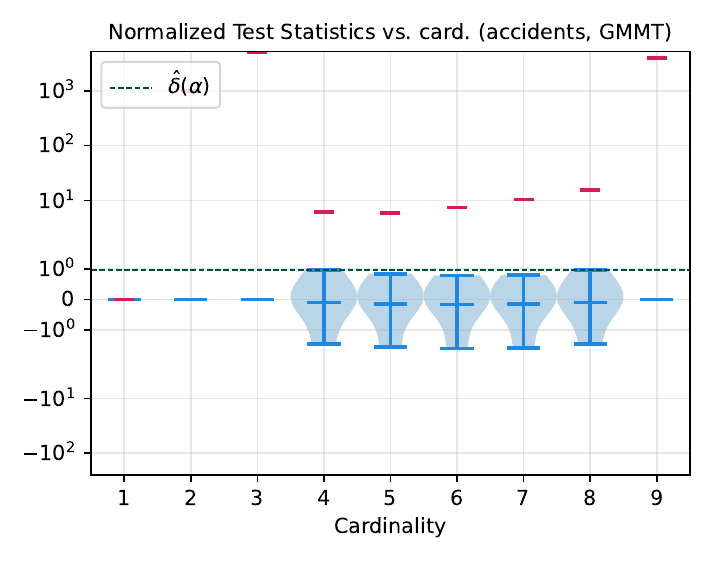} 
\end{subfigure}
\begin{subfigure}{.275\textwidth}
  \centering
  \includegraphics[width=\textwidth]{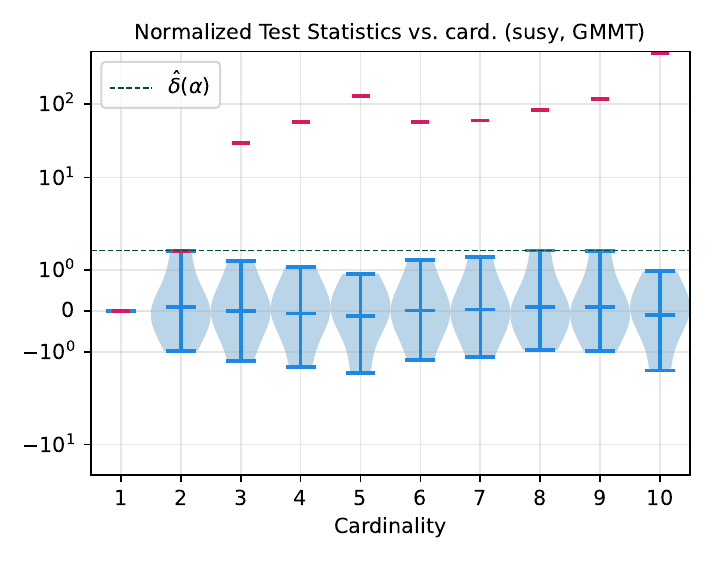} 
\end{subfigure}
\caption{ 
Testing the significance of the number of FI for different cardinalities under the GMMT model.
}
\Description{
Testing the significance of the number of FI for different cardinalities under the GMMT model.}
\label{fig:numfilengmmtnormappendix}
\end{figure*}
\fi

\ifextversion
\begin{figure*}[ht]
\begin{subfigure}{.175\textwidth}
  \centering
  \includegraphics[width=\textwidth]{./figures/power-legend.pdf}
\end{subfigure} \\
\begin{subfigure}{.275\textwidth}
  \centering
  \includegraphics[width=\textwidth]{./figures/num-fi-len-retail-ALICE.pdf} 
\end{subfigure}
\begin{subfigure}{.275\textwidth}
  \centering
  \includegraphics[width=\textwidth]{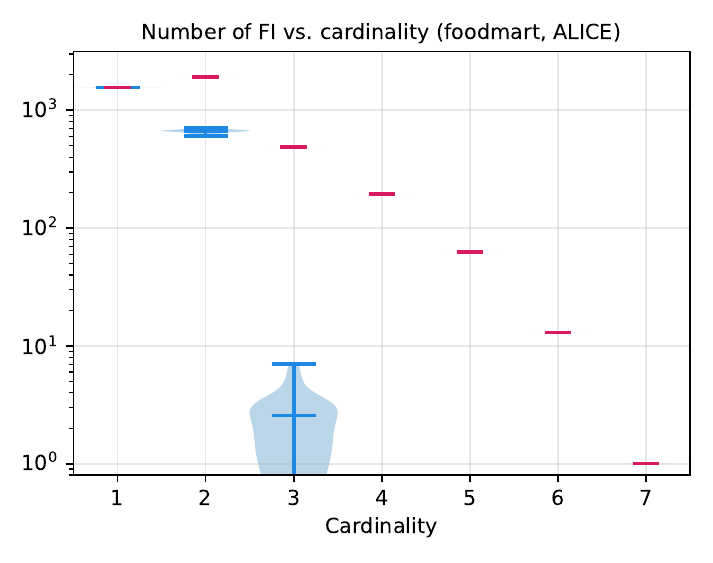} 
\end{subfigure}
\begin{subfigure}{.275\textwidth}
  \centering
  \includegraphics[width=\textwidth]{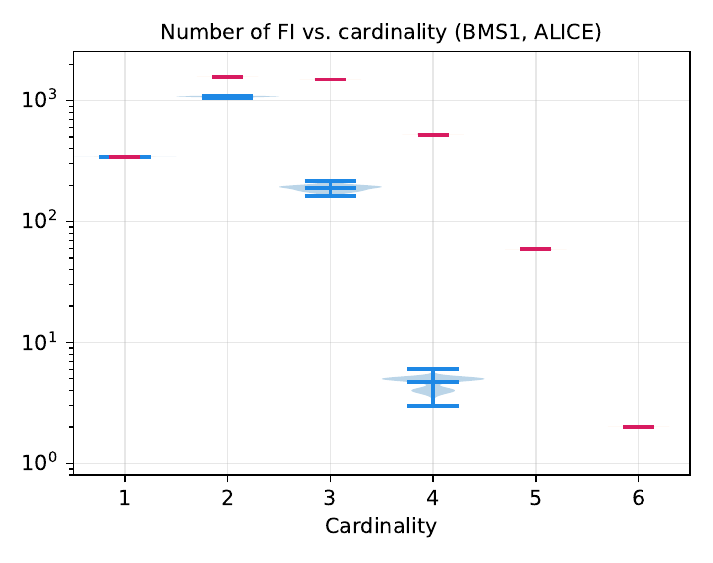} 
\end{subfigure}
\begin{subfigure}{.275\textwidth}
  \centering
  \includegraphics[width=\textwidth]{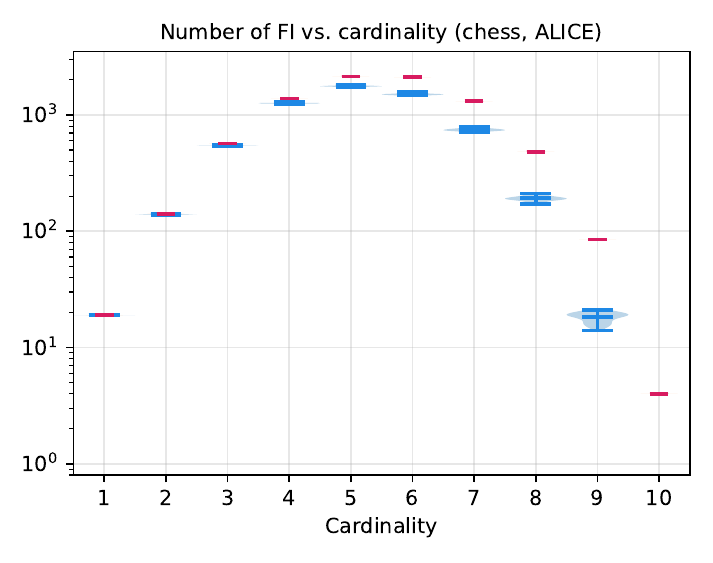} 
\end{subfigure}
\begin{subfigure}{.275\textwidth}
  \centering
  \includegraphics[width=\textwidth]{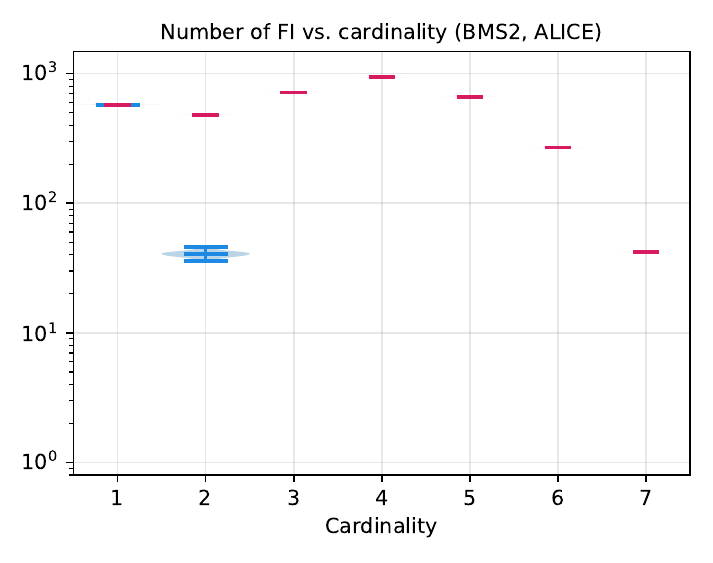} 
\end{subfigure}
\begin{subfigure}{.275\textwidth}
  \centering
  \includegraphics[width=\textwidth]{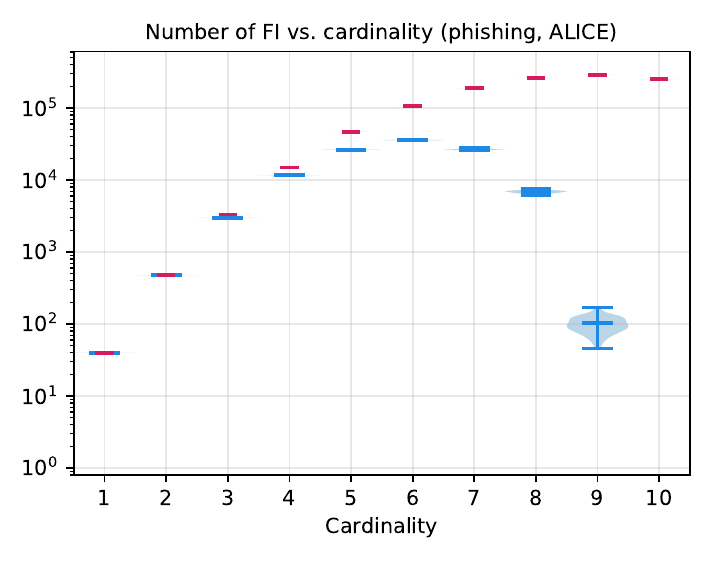} 
\end{subfigure}
\begin{subfigure}{.275\textwidth}
  \centering
  \includegraphics[width=\textwidth]{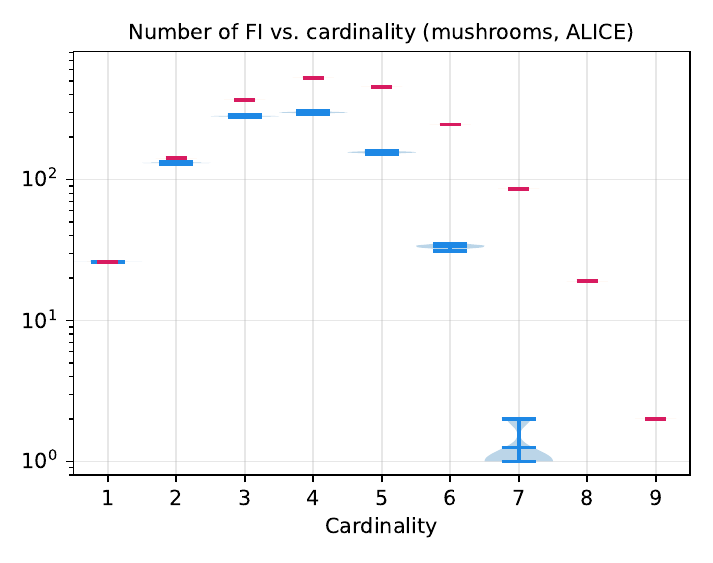} 
\end{subfigure}
\begin{subfigure}{.275\textwidth}
  \centering
  \includegraphics[width=\textwidth]{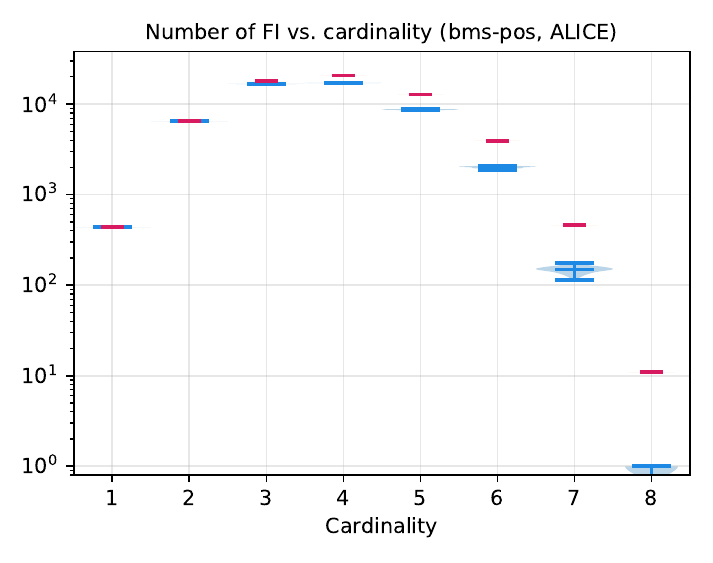} 
\end{subfigure}
\begin{subfigure}{.275\textwidth}
  \centering
  \includegraphics[width=\textwidth]{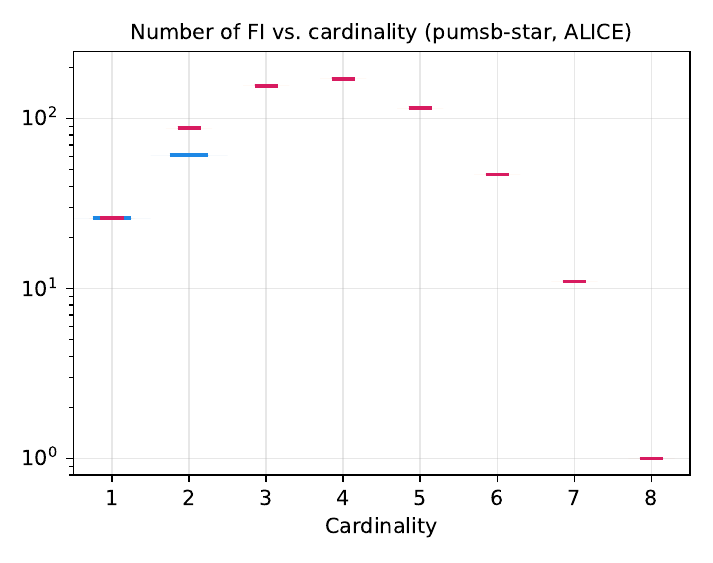} 
\end{subfigure}
\begin{subfigure}{.275\textwidth}
  \centering
  \includegraphics[width=\textwidth]{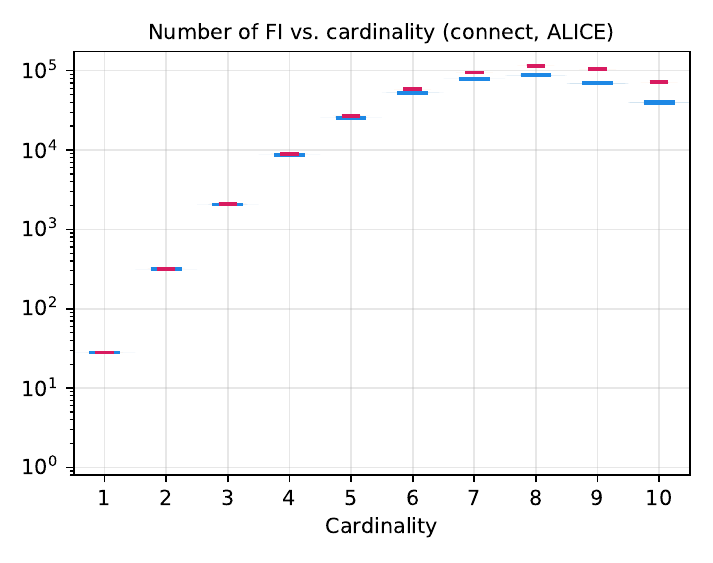} 
\end{subfigure}
\begin{subfigure}{.275\textwidth}
  \centering
  \includegraphics[width=\textwidth]{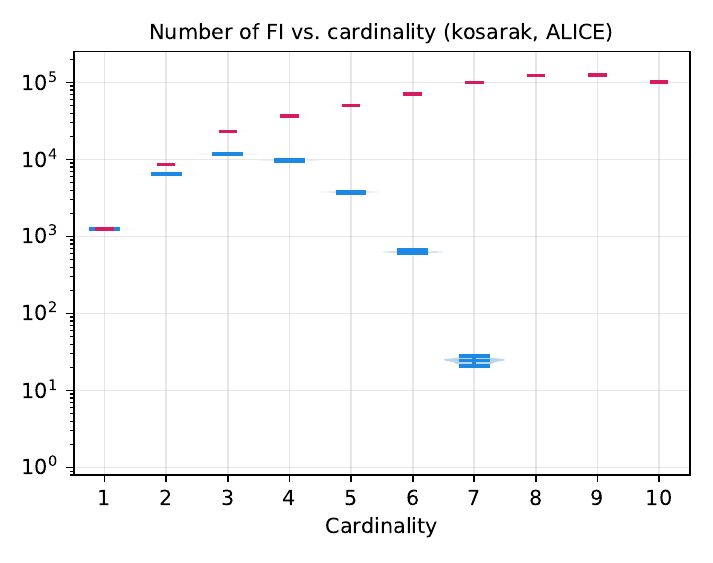} 
\end{subfigure}
\begin{subfigure}{.275\textwidth}
  \centering
  \includegraphics[width=\textwidth]{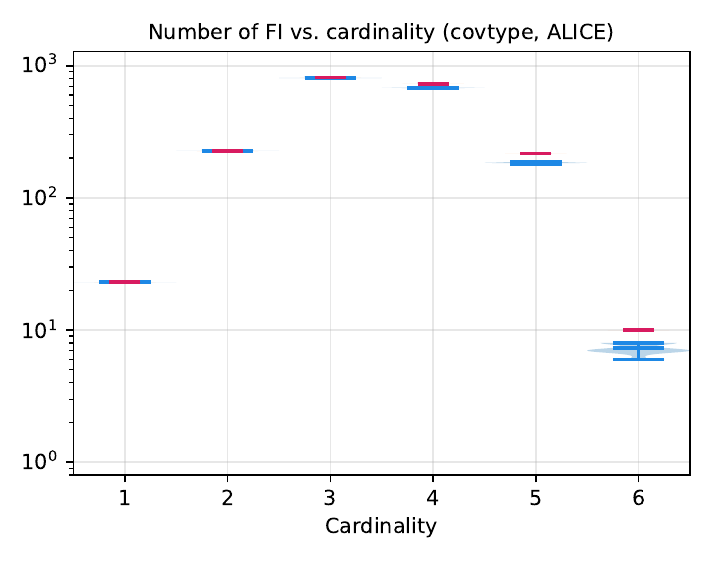} 
\end{subfigure}
\begin{subfigure}{.275\textwidth}
  \centering
  \includegraphics[width=\textwidth]{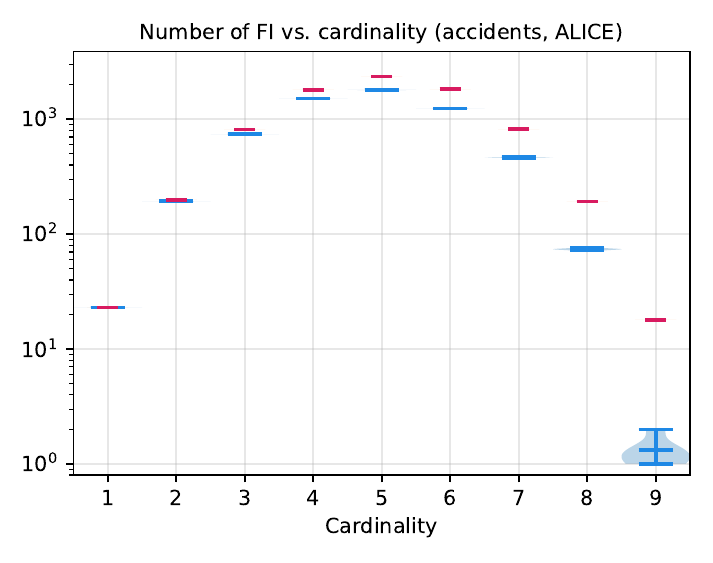} 
\end{subfigure}
\begin{subfigure}{.275\textwidth}
  \centering
  \includegraphics[width=\textwidth]{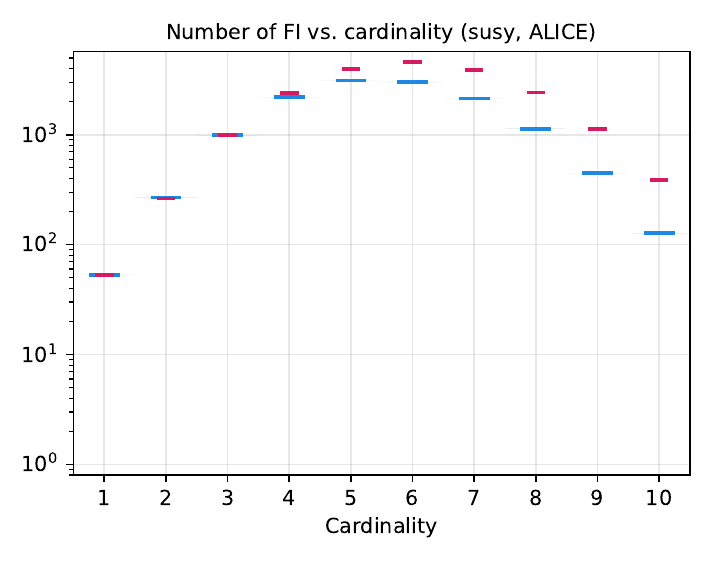} 
\end{subfigure}
\caption{ 
Testing the significance of the number of FI for different cardinalities under the ALICE model.
}
\Description{
Testing the significance of the number of FI for different cardinalities under the ALICE model.}
\label{fig:numfilenaliceappendix}
\end{figure*}
\fi

\ifextversion
\begin{figure*}[ht]
\begin{subfigure}{.175\textwidth}
  \centering
  \includegraphics[width=\textwidth]{./figures/power-legend.pdf}
\end{subfigure} \\
\begin{subfigure}{.275\textwidth}
  \centering
  \includegraphics[width=\textwidth]{./figures/num-fi-len-retail-norm-ALICE.pdf} 
\end{subfigure}
\begin{subfigure}{.275\textwidth}
  \centering
  \includegraphics[width=\textwidth]{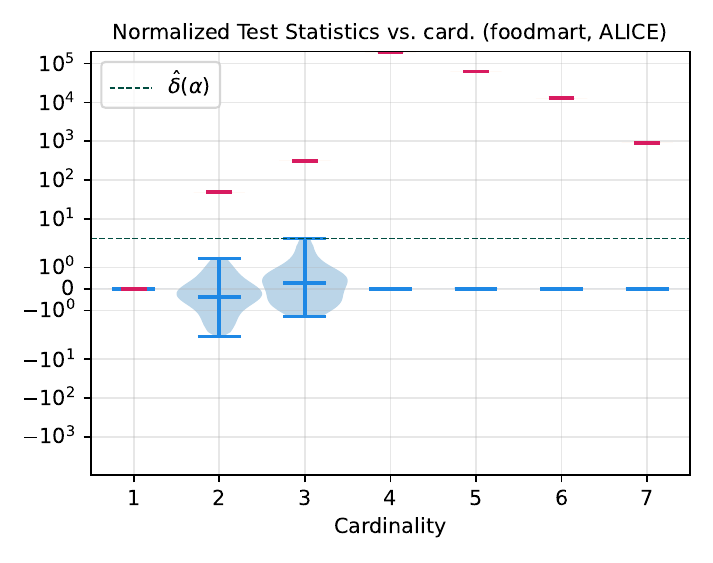} 
\end{subfigure}
\begin{subfigure}{.275\textwidth}
  \centering
  \includegraphics[width=\textwidth]{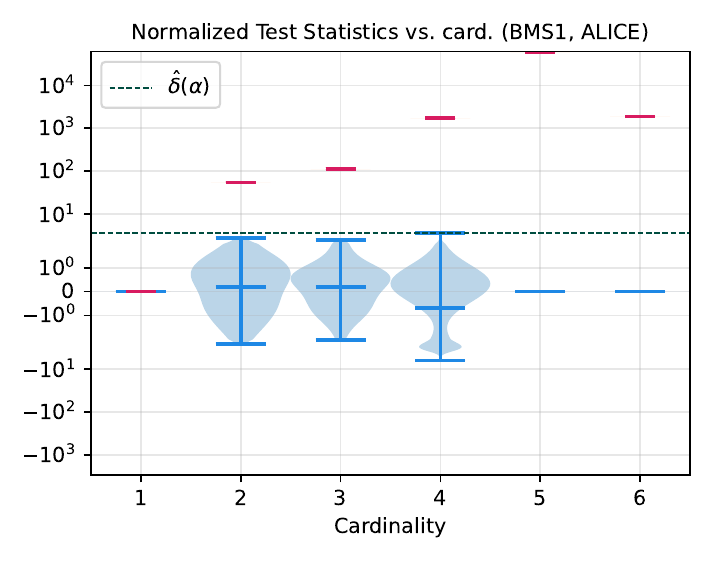} 
\end{subfigure}
\begin{subfigure}{.275\textwidth}
  \centering
  \includegraphics[width=\textwidth]{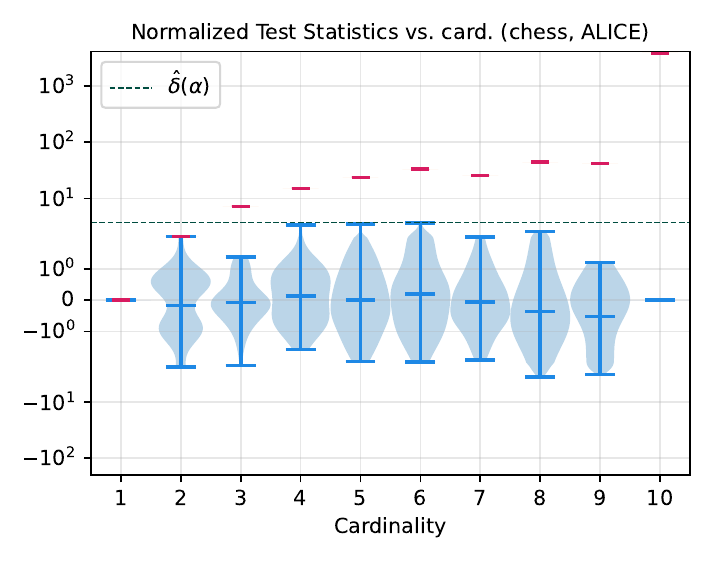} 
\end{subfigure}
\begin{subfigure}{.275\textwidth}
  \centering
  \includegraphics[width=\textwidth]{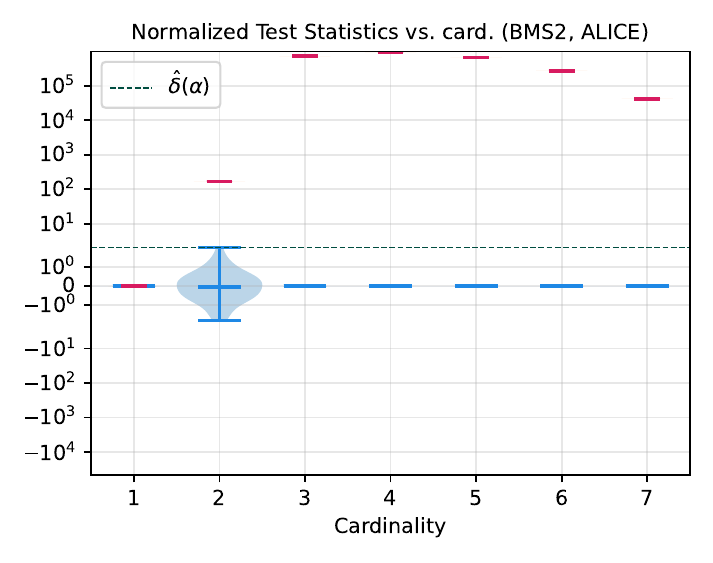} 
\end{subfigure}
\begin{subfigure}{.275\textwidth}
  \centering
  \includegraphics[width=\textwidth]{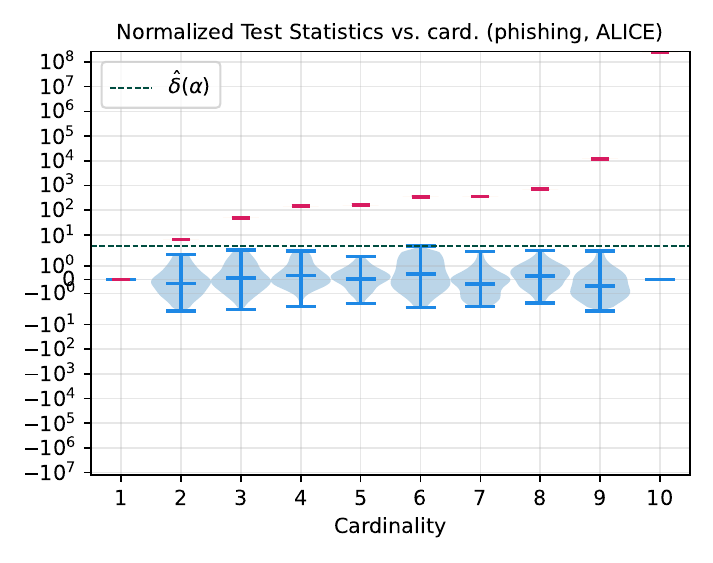} 
\end{subfigure}
\begin{subfigure}{.275\textwidth}
  \centering
  \includegraphics[width=\textwidth]{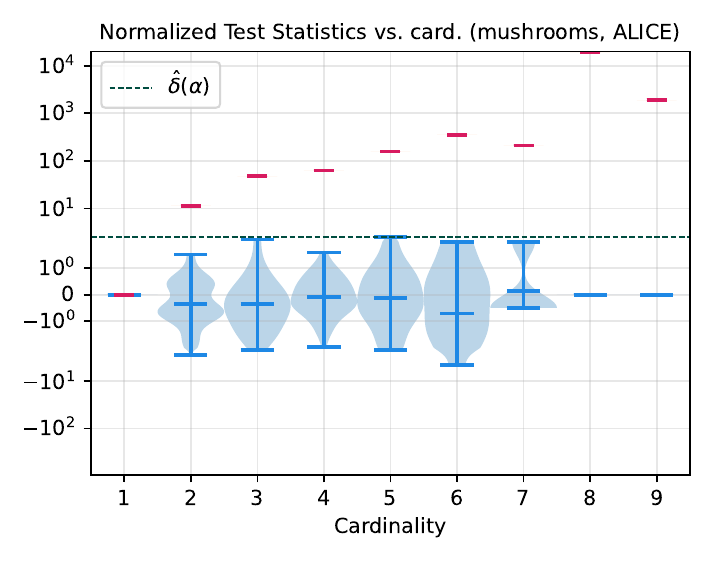} 
\end{subfigure}
\begin{subfigure}{.275\textwidth}
  \centering
  \includegraphics[width=\textwidth]{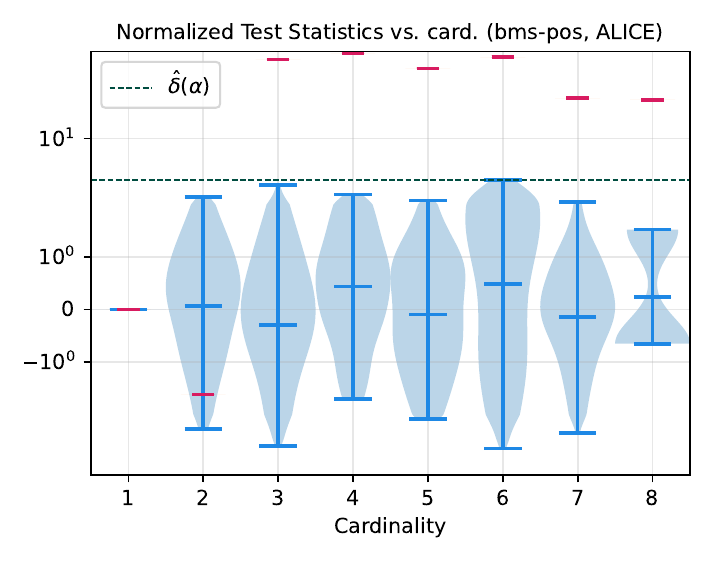} 
\end{subfigure}
\begin{subfigure}{.275\textwidth}
  \centering
  \includegraphics[width=\textwidth]{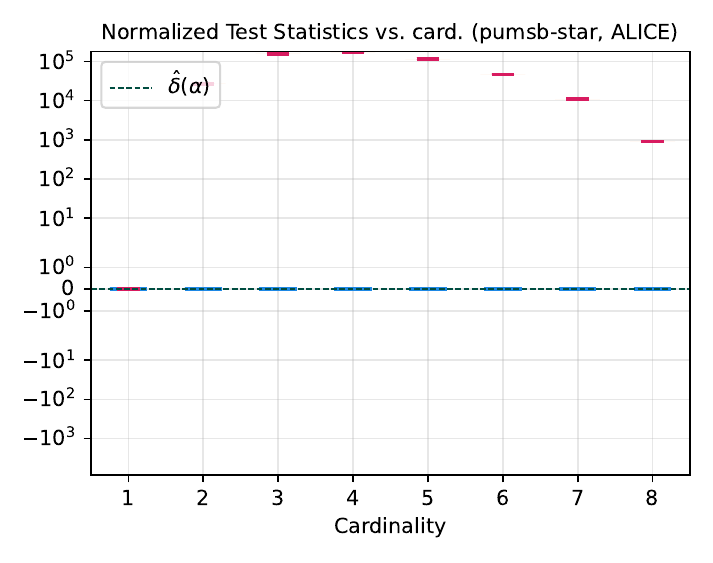} 
\end{subfigure}
\begin{subfigure}{.275\textwidth}
  \centering
  \includegraphics[width=\textwidth]{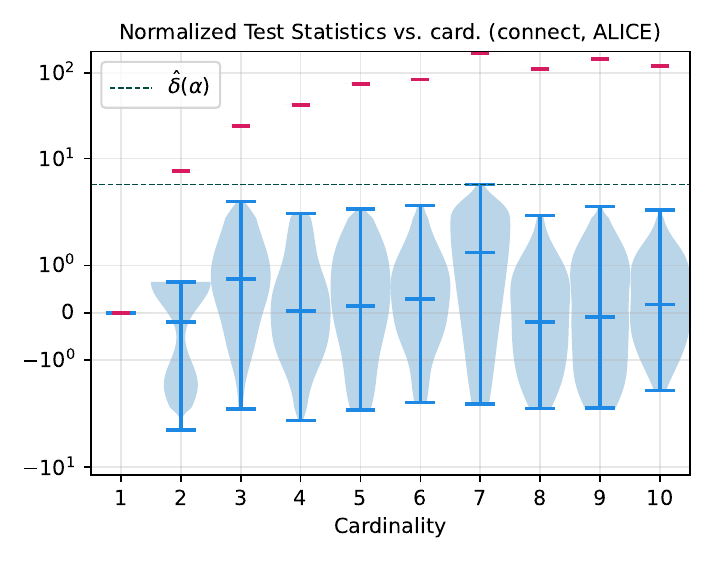} 
\end{subfigure}
\begin{subfigure}{.275\textwidth}
  \centering
  \includegraphics[width=\textwidth]{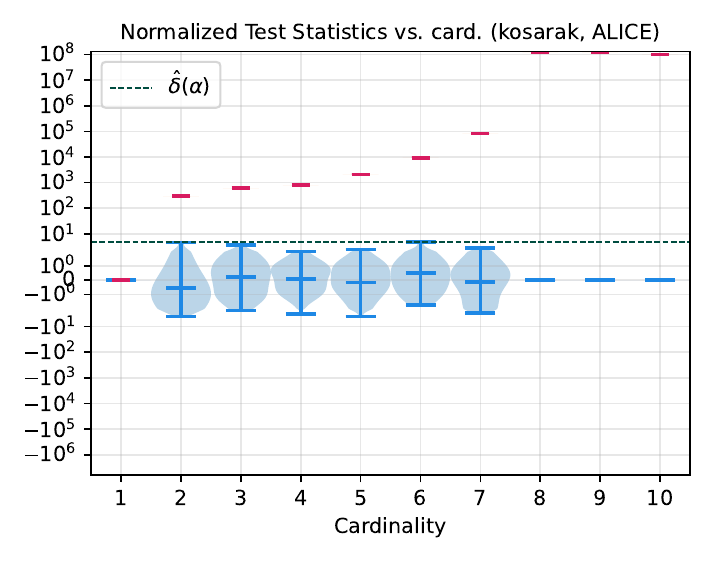} 
\end{subfigure}
\begin{subfigure}{.275\textwidth}
  \centering
  \includegraphics[width=\textwidth]{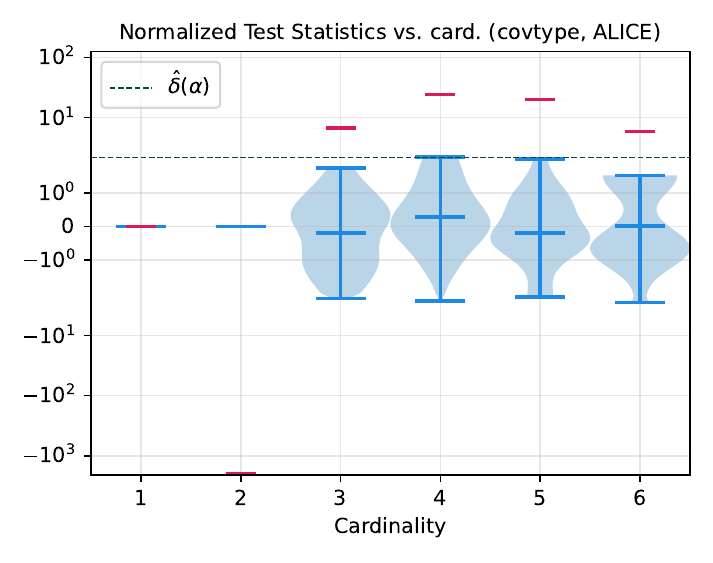} 
\end{subfigure}
\begin{subfigure}{.275\textwidth}
  \centering
  \includegraphics[width=\textwidth]{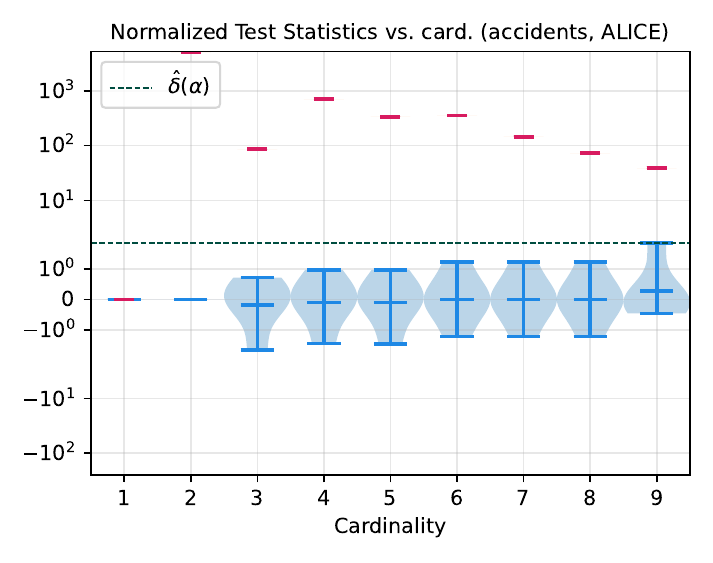} 
\end{subfigure}
\begin{subfigure}{.275\textwidth}
  \centering
  \includegraphics[width=\textwidth]{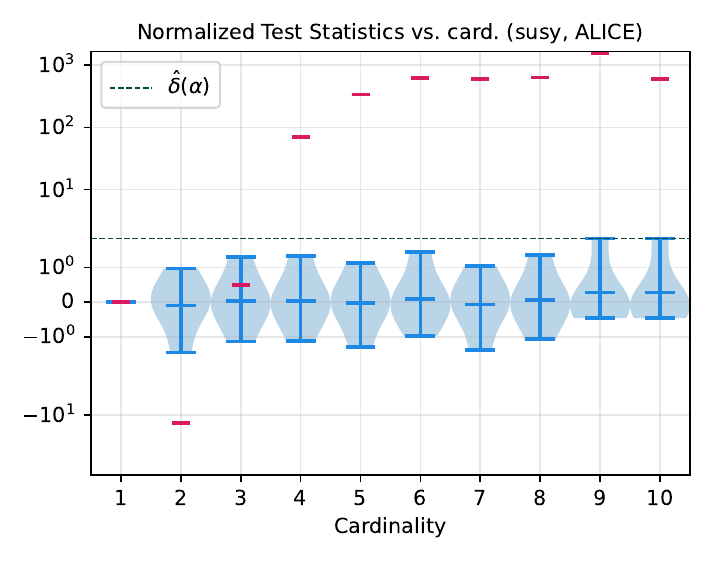} 
\end{subfigure}
\caption{ 
Testing the significance of the number of FI for different cardinalities under the ALICE model.
}
\Description{
Testing the significance of the number of FI for different cardinalities under the ALICE model.}
\label{fig:numfilenalicenormappendix}
\end{figure*}
\fi

\ifextversion
\begin{figure*}[ht]
\begin{subfigure}{.175\textwidth}
  \centering
  \includegraphics[width=\textwidth]{./figures/power-legend.pdf}
\end{subfigure} \\
\begin{subfigure}{.275\textwidth}
  \centering
  \includegraphics[width=\textwidth]{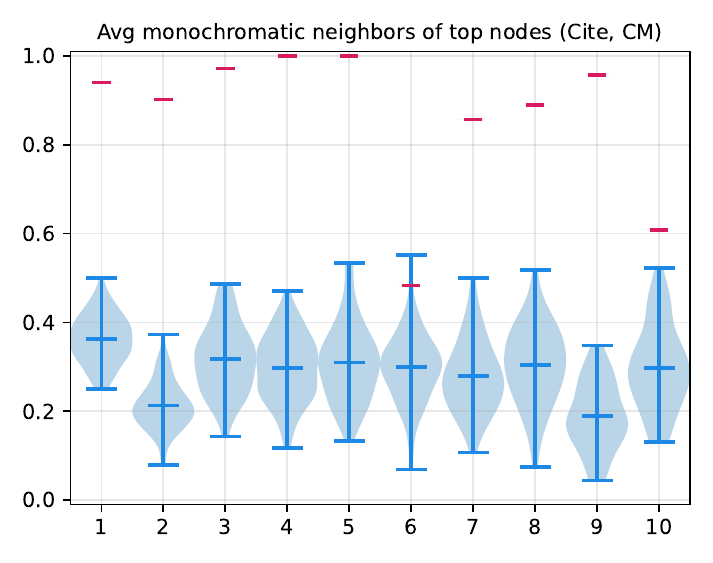} 
\end{subfigure}
\begin{subfigure}{.275\textwidth}
  \centering
  \includegraphics[width=\textwidth]{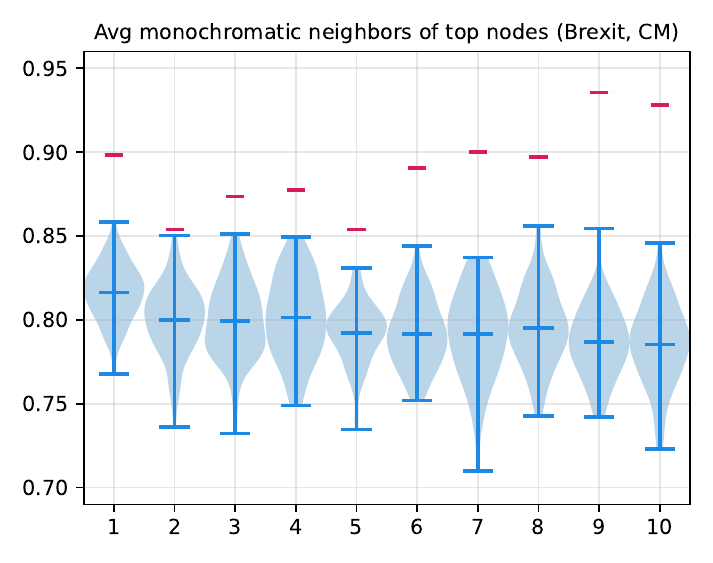} 
\end{subfigure}
\begin{subfigure}{.275\textwidth}
  \centering
  \includegraphics[width=\textwidth]{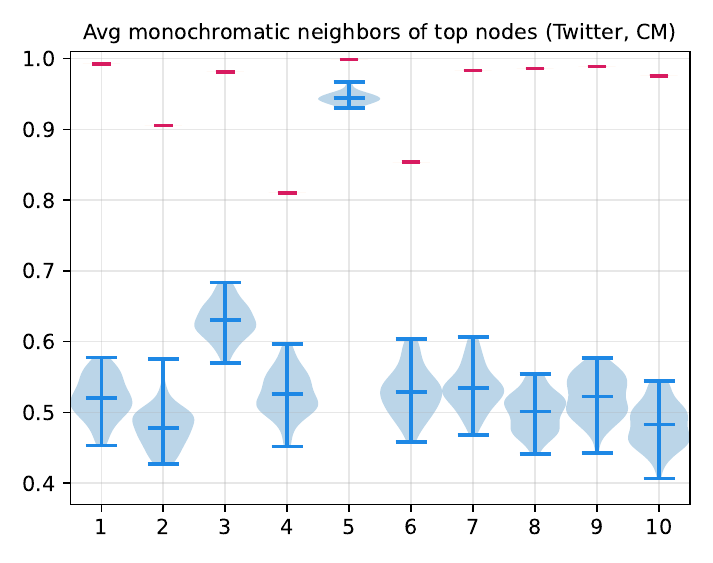} 
\end{subfigure}
\begin{subfigure}{.275\textwidth}
  \centering
  \includegraphics[width=\textwidth]{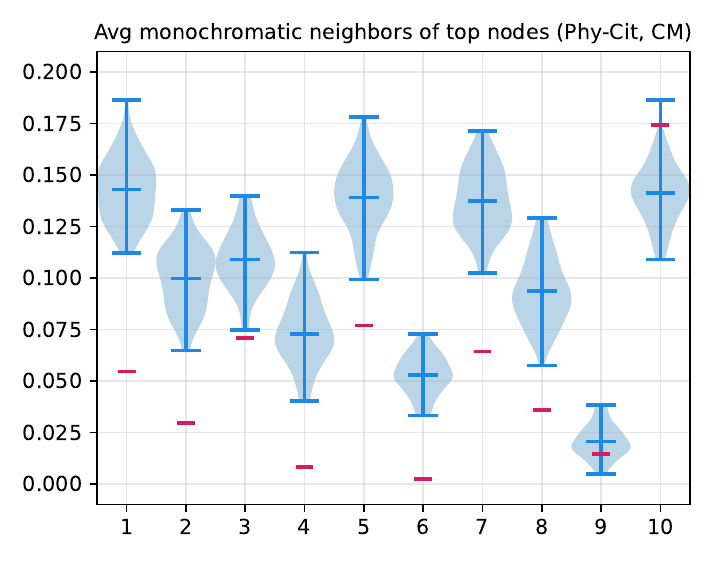} 
\end{subfigure}
\begin{subfigure}{.275\textwidth}
  \centering
  \includegraphics[width=\textwidth]{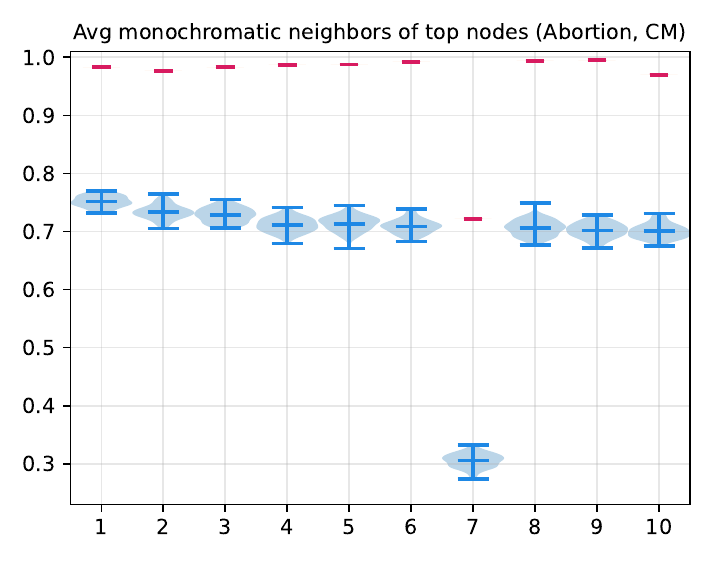} 
\end{subfigure}
\begin{subfigure}{.275\textwidth}
  \centering
  \includegraphics[width=\textwidth]{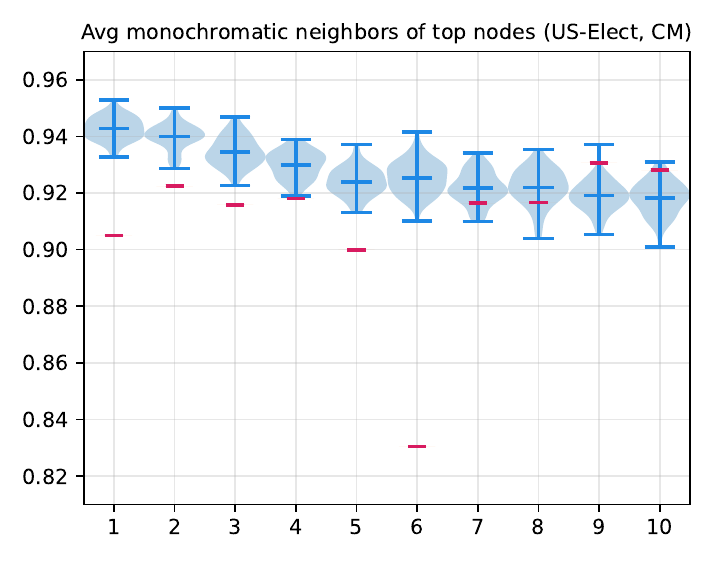} 
\end{subfigure}
\begin{subfigure}{.275\textwidth}
  \centering
  \includegraphics[width=\textwidth]{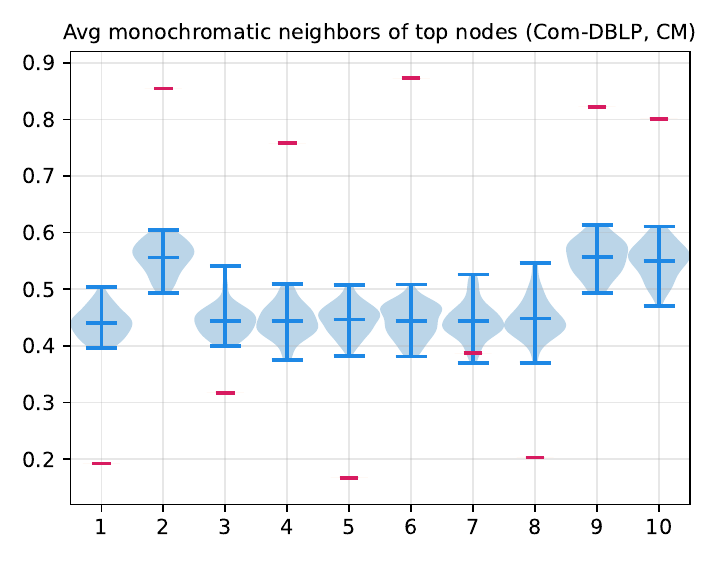} 
\end{subfigure}
\begin{subfigure}{.275\textwidth}
  \centering
  \includegraphics[width=\textwidth]{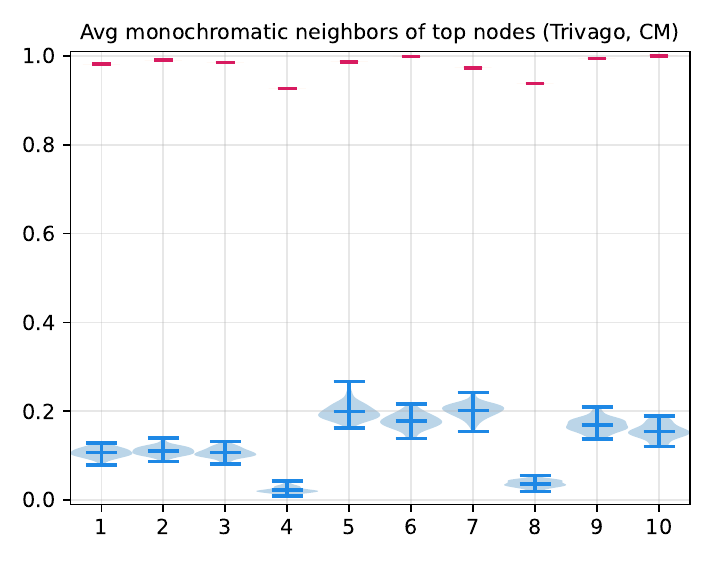} 
\end{subfigure}
\begin{subfigure}{.275\textwidth}
  \centering
  \includegraphics[width=\textwidth]{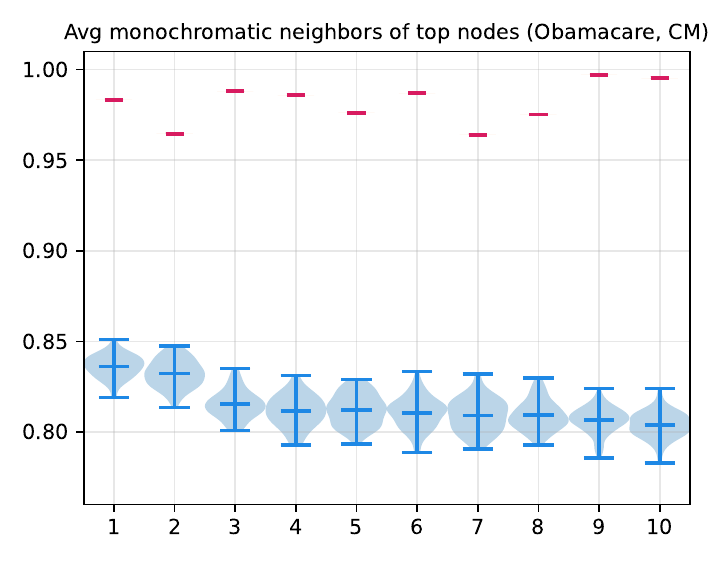} 
\end{subfigure}
\begin{subfigure}{.275\textwidth}
  \centering
  \includegraphics[width=\textwidth]{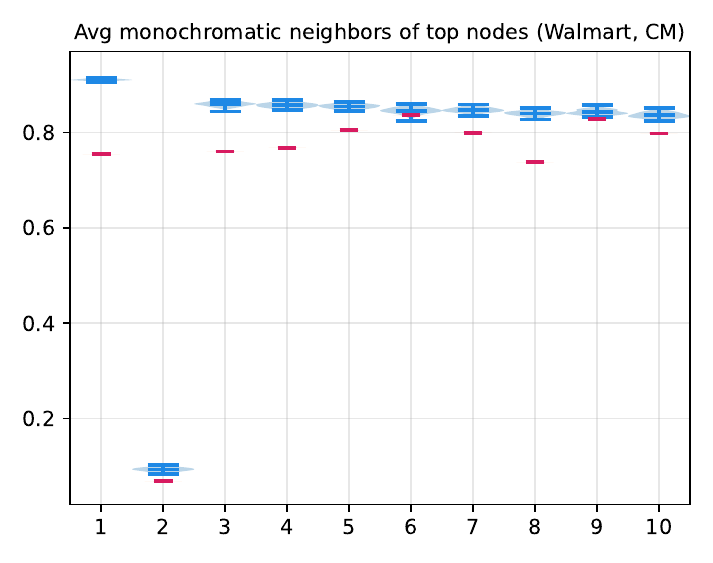} 
\end{subfigure}
\begin{subfigure}{.275\textwidth}
  \centering
  \includegraphics[width=\textwidth]{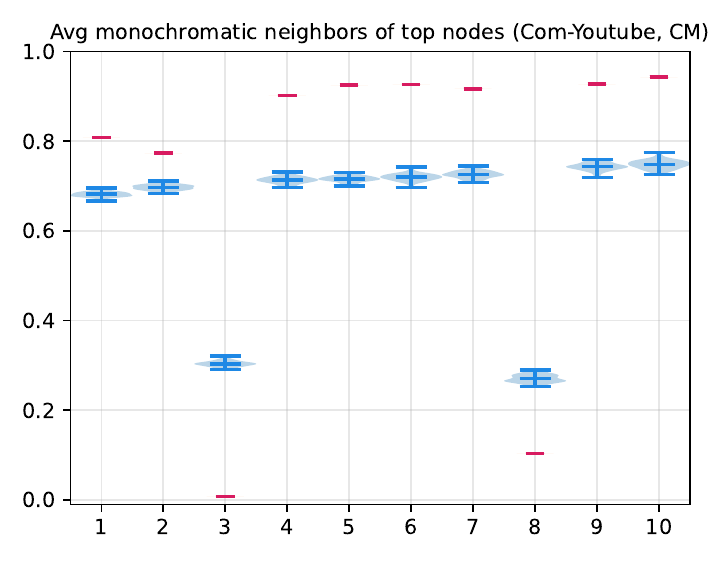} 
\end{subfigure}
\begin{subfigure}{.275\textwidth}
  \centering
  \includegraphics[width=\textwidth]{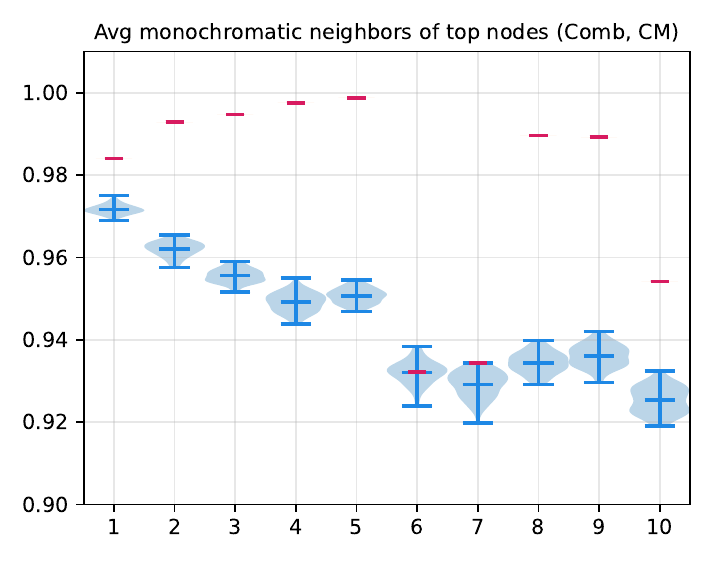} 
\end{subfigure}
\begin{subfigure}{.275\textwidth}
  \centering
  \includegraphics[width=\textwidth]{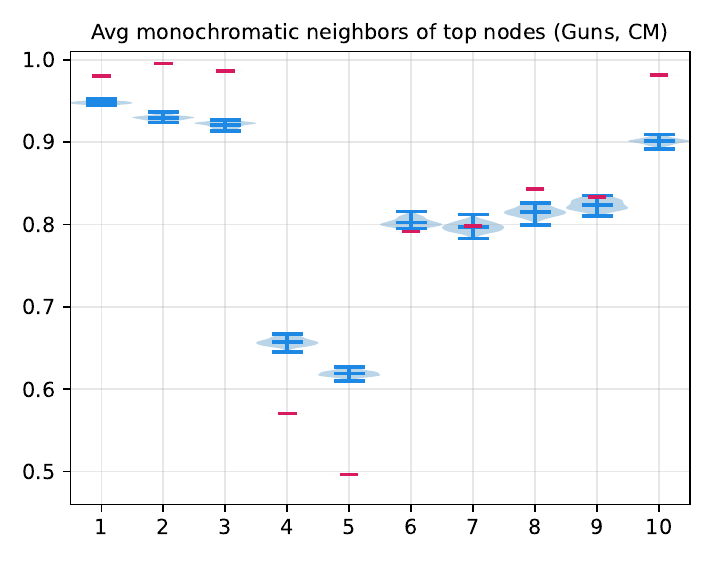} 
\end{subfigure}
\begin{subfigure}{.275\textwidth}
  \centering
  \includegraphics[width=\textwidth]{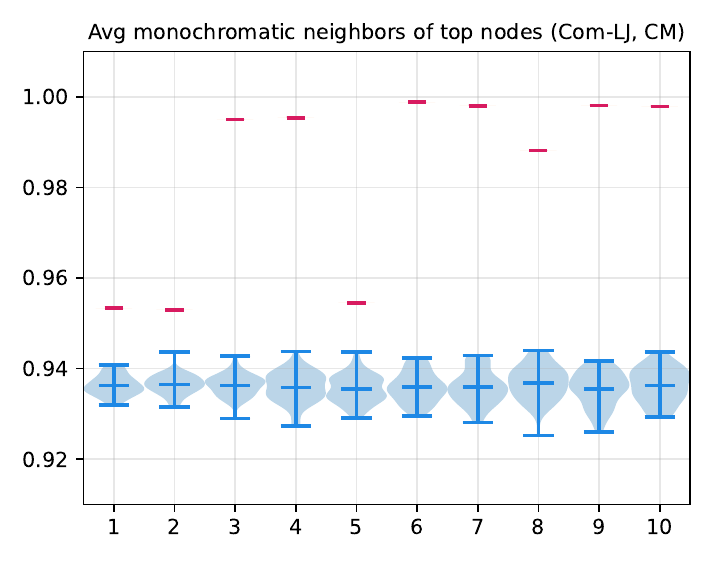} 
\end{subfigure}
\begin{subfigure}{.275\textwidth}
  \centering
  \includegraphics[width=\textwidth]{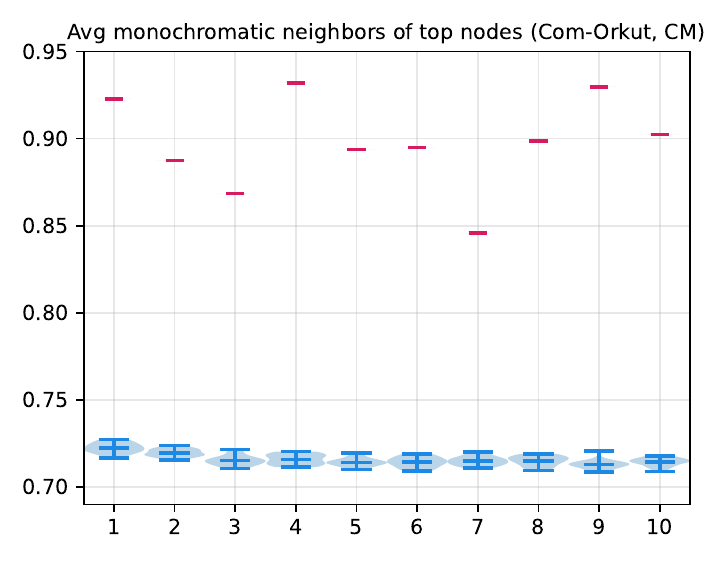} 
\end{subfigure}
\caption{ 
Testing the significance of $M_v(G)$ of the $10$ nodes with highest degree under the CM model.
}
\Description{
Testing the significance of $M_v(G)$ of the $10$ nodes with highest degree under the CM model.}
\label{fig:avgcolcmappendix}
\end{figure*}
\fi

\ifextversion
\begin{figure*}[ht]
\begin{subfigure}{.175\textwidth}
  \centering
  \includegraphics[width=\textwidth]{./figures/power-legend.pdf}
\end{subfigure} \\
\begin{subfigure}{.275\textwidth}
  \centering
  \includegraphics[width=\textwidth]{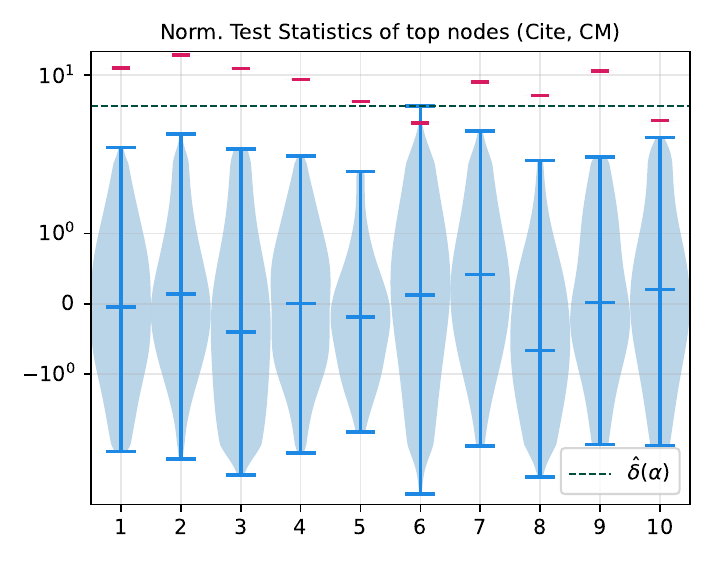} 
\end{subfigure}
\begin{subfigure}{.275\textwidth}
  \centering
  \includegraphics[width=\textwidth]{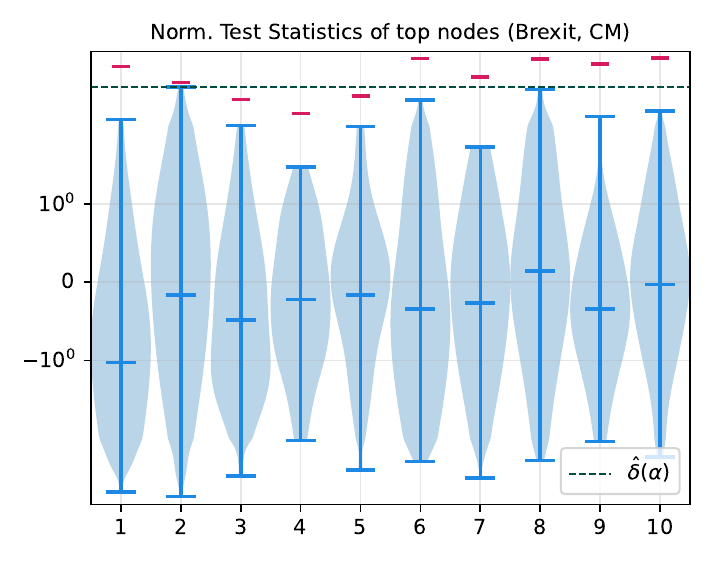} 
\end{subfigure}
\begin{subfigure}{.275\textwidth}
  \centering
  \includegraphics[width=\textwidth]{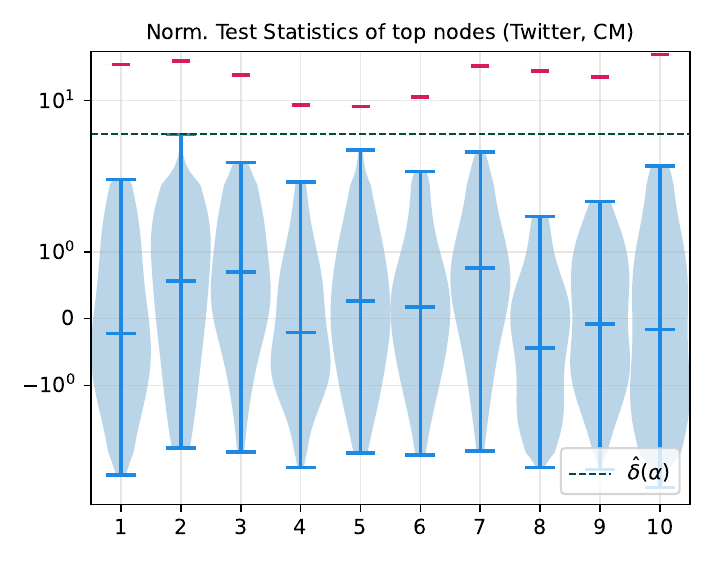} 
\end{subfigure}
\begin{subfigure}{.275\textwidth}
  \centering
  \includegraphics[width=\textwidth]{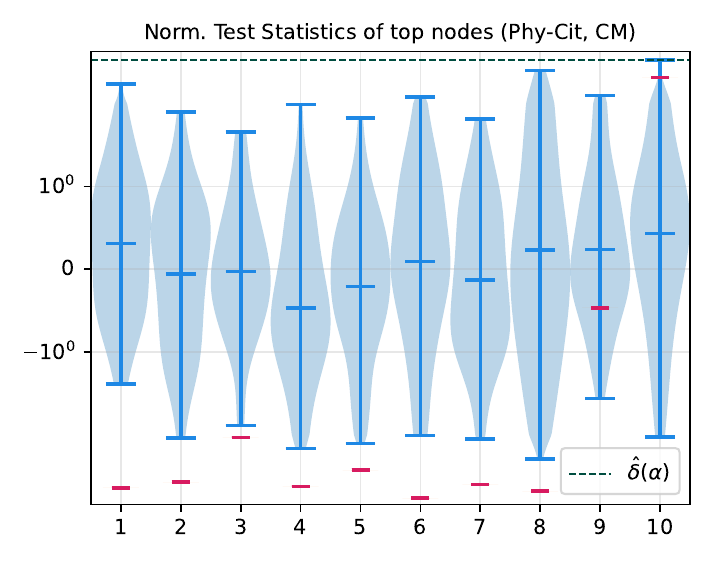} 
\end{subfigure}
\begin{subfigure}{.275\textwidth}
  \centering
  \includegraphics[width=\textwidth]{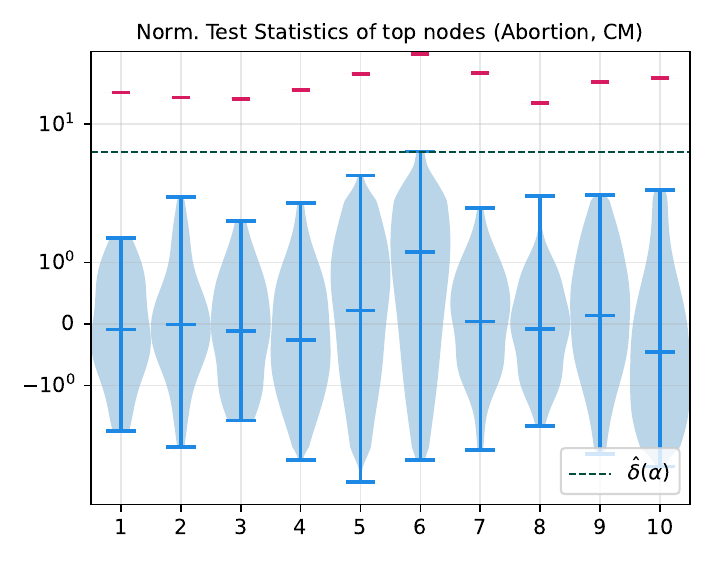} 
\end{subfigure}
\begin{subfigure}{.275\textwidth}
  \centering
  \includegraphics[width=\textwidth]{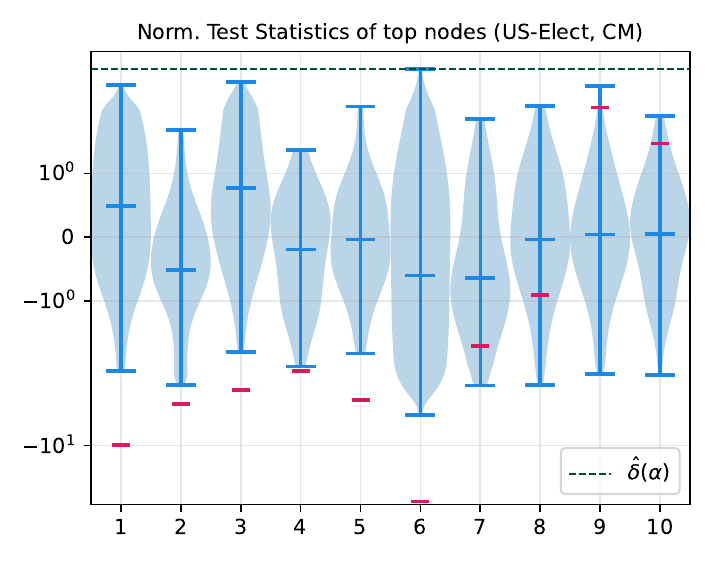} 
\end{subfigure}
\begin{subfigure}{.275\textwidth}
  \centering
  \includegraphics[width=\textwidth]{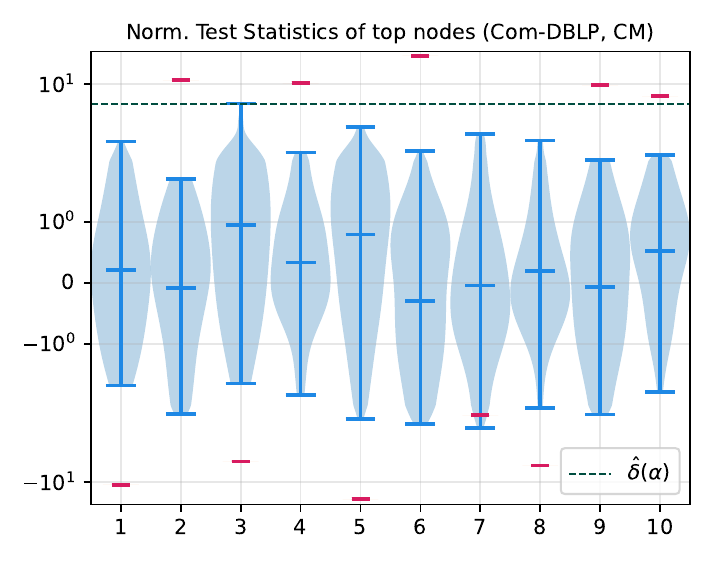} 
\end{subfigure}
\begin{subfigure}{.275\textwidth}
  \centering
  \includegraphics[width=\textwidth]{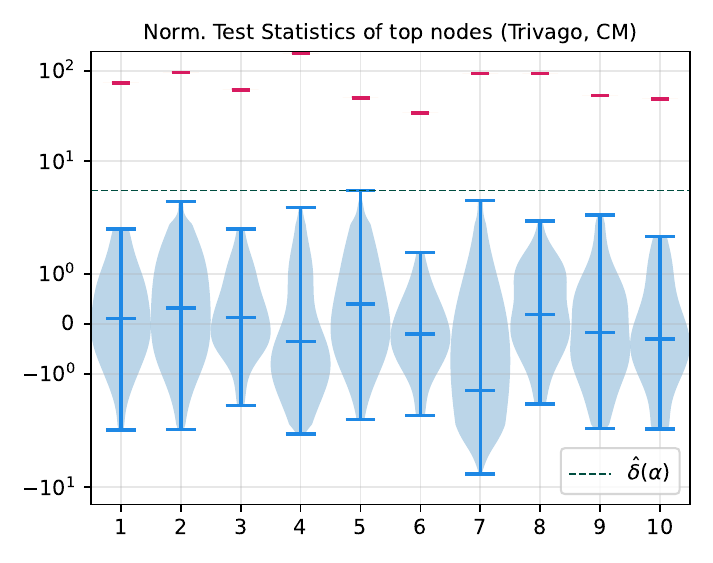} 
\end{subfigure}
\begin{subfigure}{.275\textwidth}
  \centering
  \includegraphics[width=\textwidth]{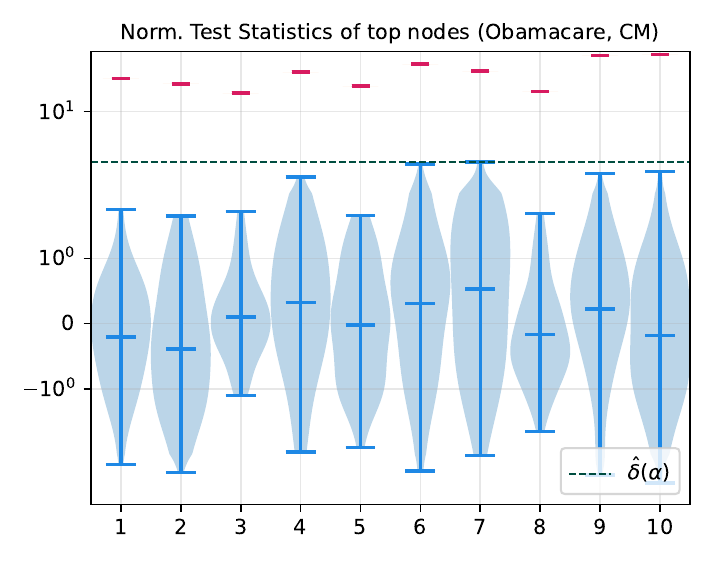} 
\end{subfigure}
\begin{subfigure}{.275\textwidth}
  \centering
  \includegraphics[width=\textwidth]{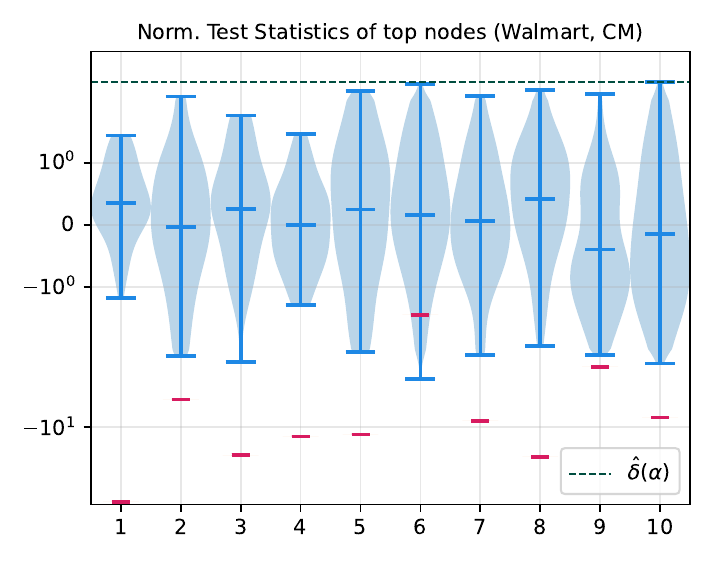} 
\end{subfigure}
\begin{subfigure}{.275\textwidth}
  \centering
  \includegraphics[width=\textwidth]{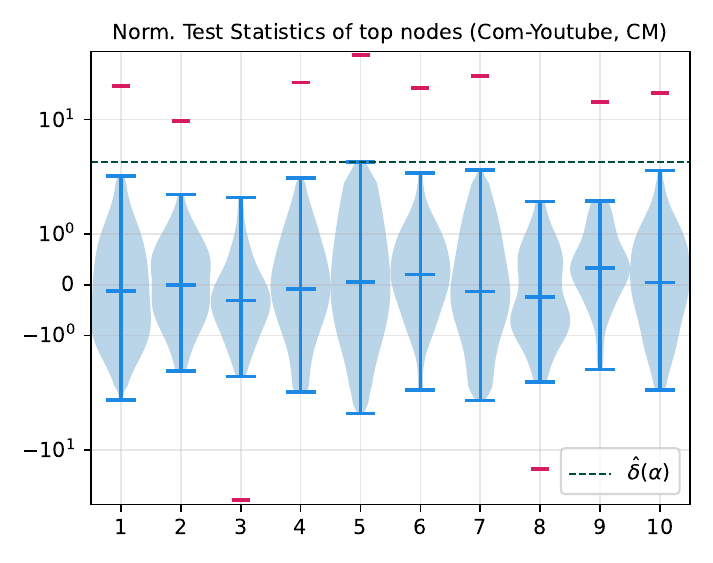} 
\end{subfigure}
\begin{subfigure}{.275\textwidth}
  \centering
  \includegraphics[width=\textwidth]{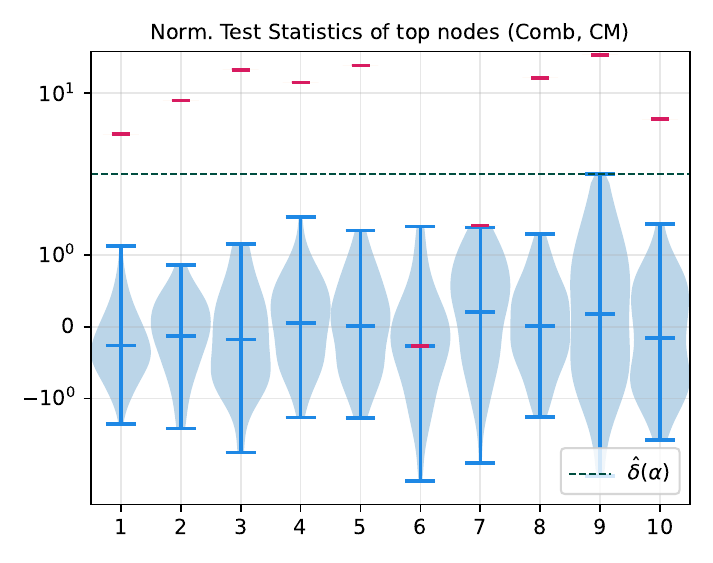} 
\end{subfigure}
\begin{subfigure}{.275\textwidth}
  \centering
  \includegraphics[width=\textwidth]{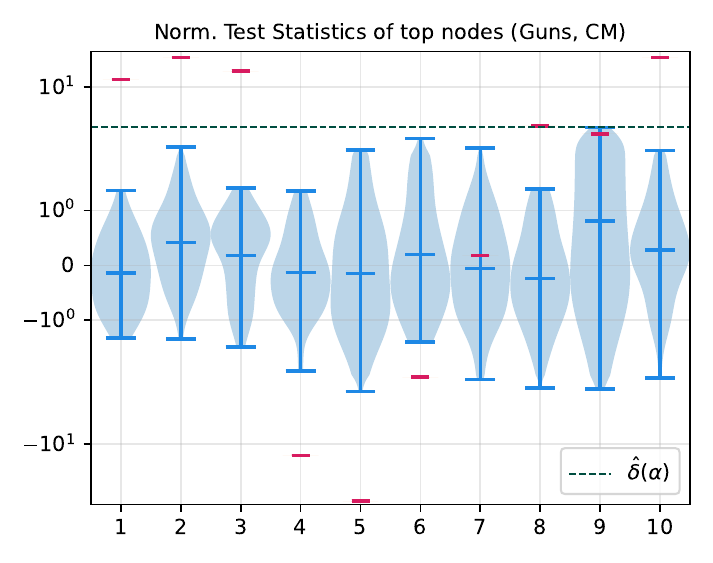} 
\end{subfigure}
\begin{subfigure}{.275\textwidth}
  \centering
  \includegraphics[width=\textwidth]{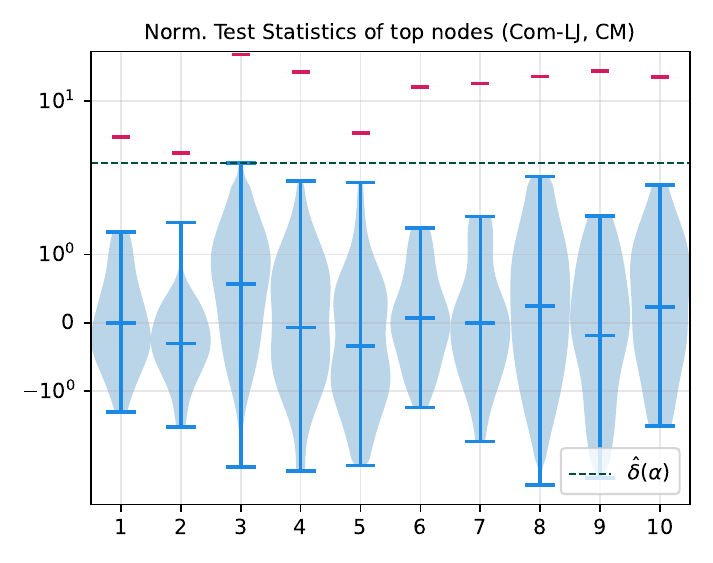} 
\end{subfigure}
\begin{subfigure}{.275\textwidth}
  \centering
  \includegraphics[width=\textwidth]{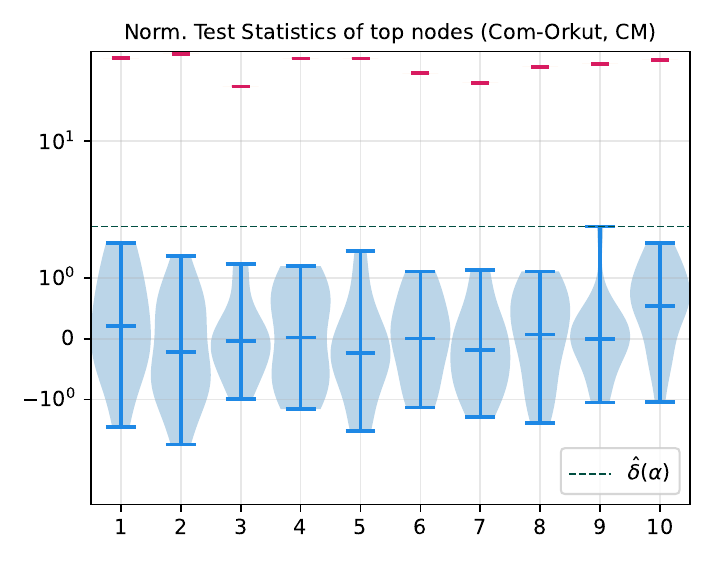} 
\end{subfigure}
\caption{ 
Testing the significance of $M_v(G)$ of the $10$ nodes with highest degree under the CM model.
}
\Description{
Testing the significance of $M_v(G)$ of the $10$ nodes with highest degree under the CM model.}
\label{fig:avgcolcmnormappendix}
\end{figure*}
\fi

\ifextversion
\begin{figure*}[ht]
\begin{subfigure}{.175\textwidth}
  \centering
  \includegraphics[width=\textwidth]{./figures/power-legend.pdf}
\end{subfigure} \\
\begin{subfigure}{.275\textwidth}
  \centering
  \includegraphics[width=\textwidth]{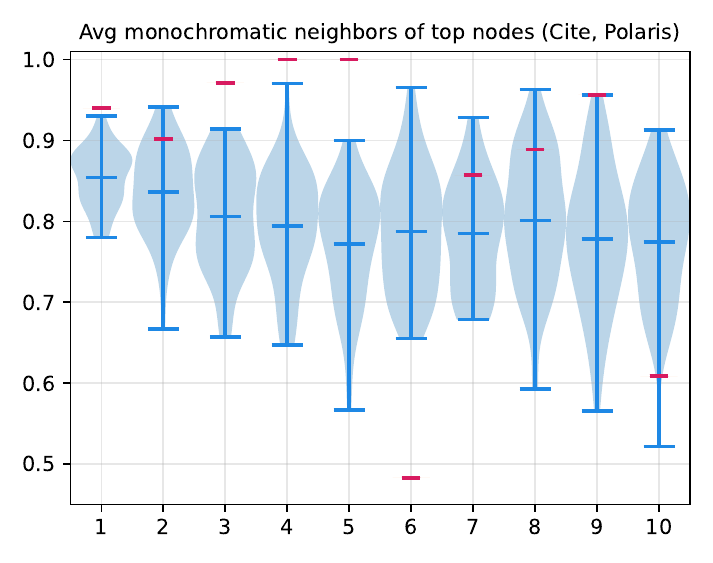} 
\end{subfigure}
\begin{subfigure}{.275\textwidth}
  \centering
  \includegraphics[width=\textwidth]{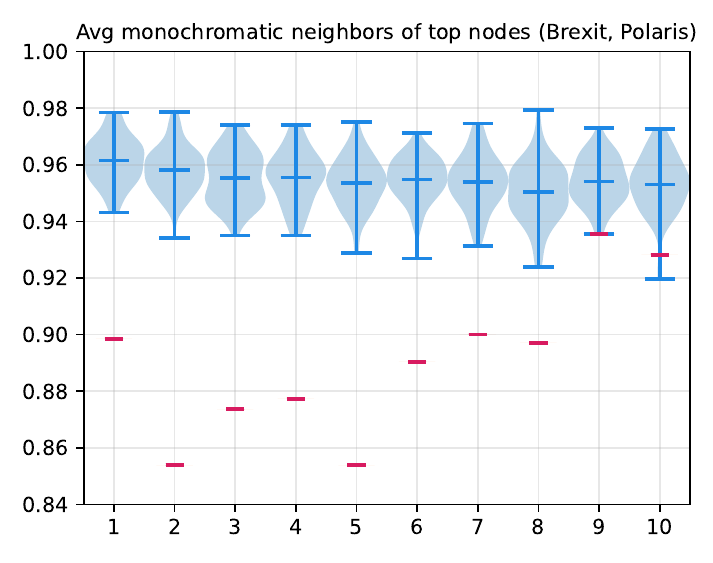} 
\end{subfigure}
\begin{subfigure}{.275\textwidth}
  \centering
  \includegraphics[width=\textwidth]{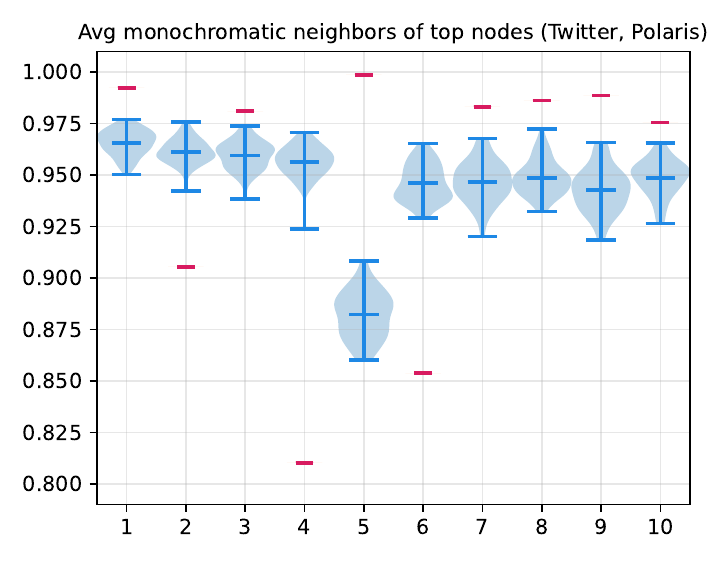} 
\end{subfigure}
\begin{subfigure}{.275\textwidth}
  \centering
  \includegraphics[width=\textwidth]{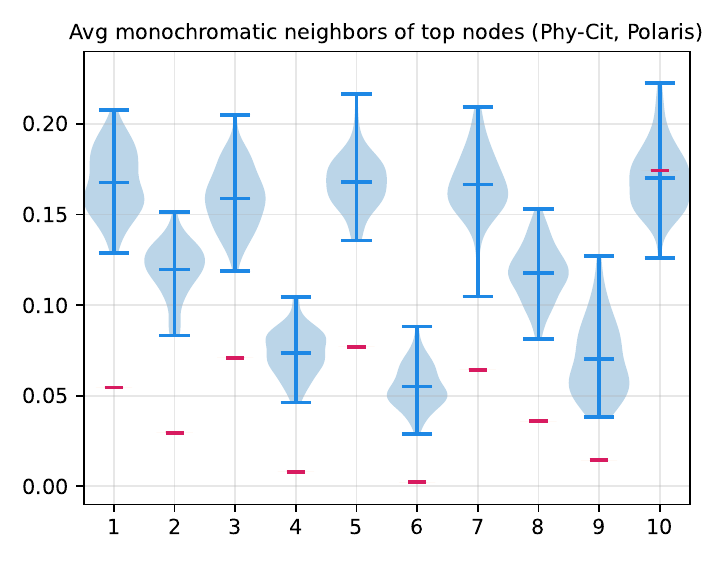} 
\end{subfigure}
\begin{subfigure}{.275\textwidth}
  \centering
  \includegraphics[width=\textwidth]{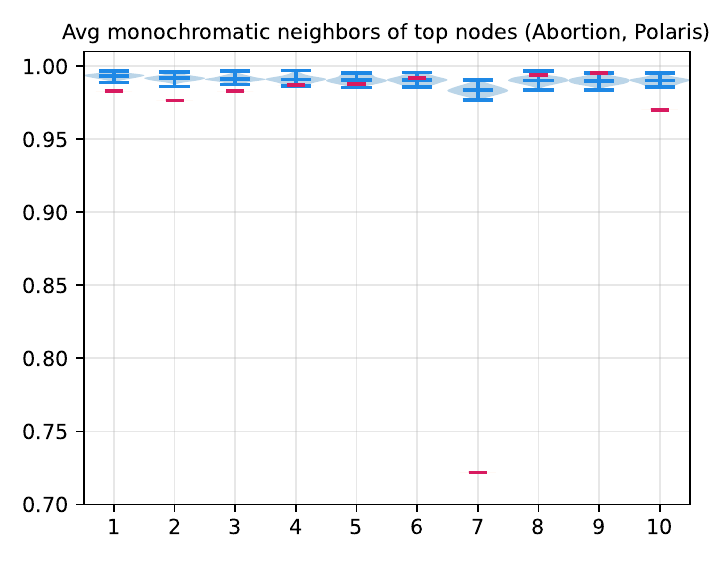} 
\end{subfigure}
\begin{subfigure}{.275\textwidth}
  \centering
  \includegraphics[width=\textwidth]{./figures/avg-col-stats-neigh-US-Elect-Polaris.pdf} 
\end{subfigure}
\begin{subfigure}{.275\textwidth}
  \centering
  \includegraphics[width=\textwidth]{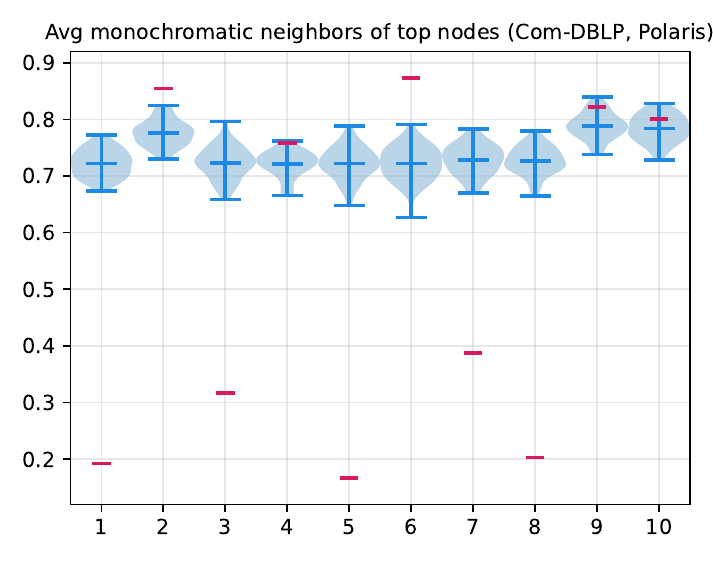} 
\end{subfigure}
\begin{subfigure}{.275\textwidth}
  \centering
  \includegraphics[width=\textwidth]{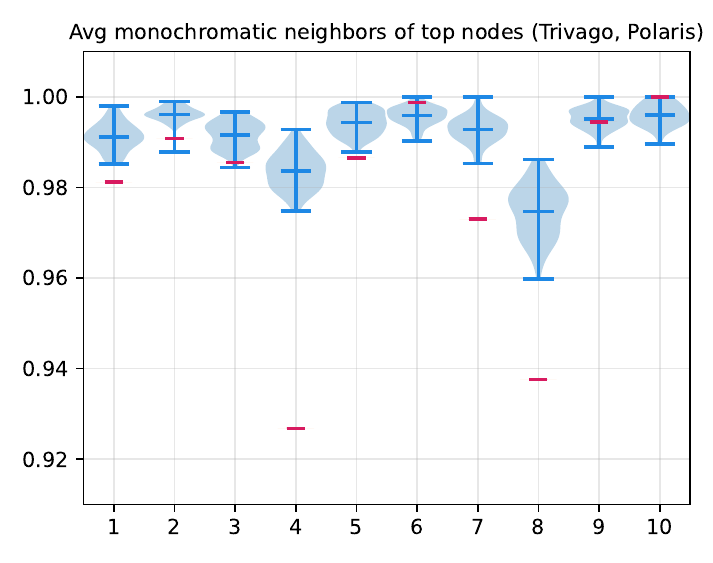} 
\end{subfigure}
\begin{subfigure}{.275\textwidth}
  \centering
  \includegraphics[width=\textwidth]{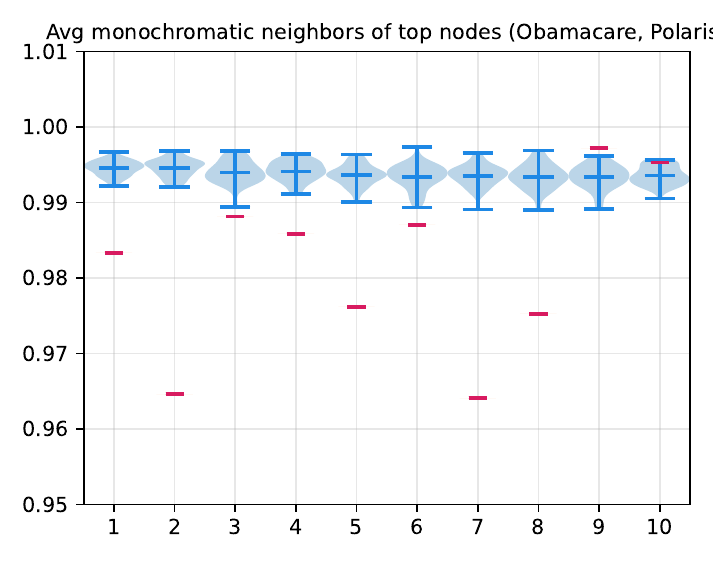} 
\end{subfigure}
\begin{subfigure}{.275\textwidth}
  \centering
  \includegraphics[width=\textwidth]{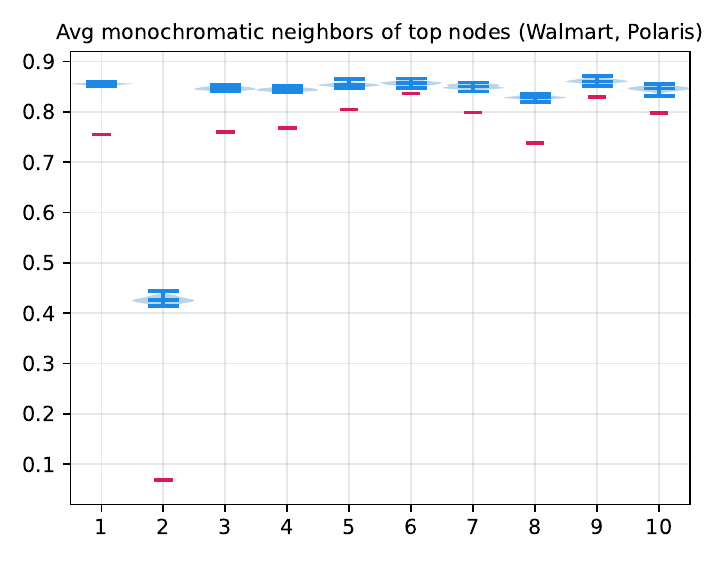} 
\end{subfigure}
\begin{subfigure}{.275\textwidth}
  \centering
  \includegraphics[width=\textwidth]{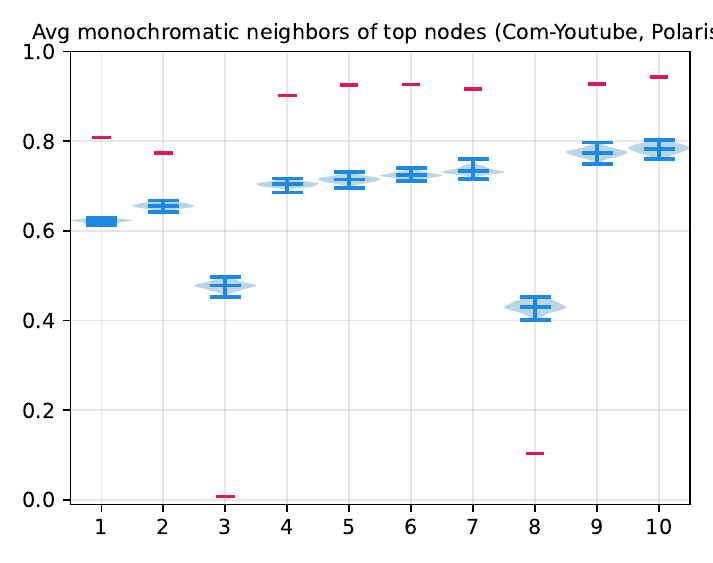} 
\end{subfigure}
\begin{subfigure}{.275\textwidth}
  \centering
  \includegraphics[width=\textwidth]{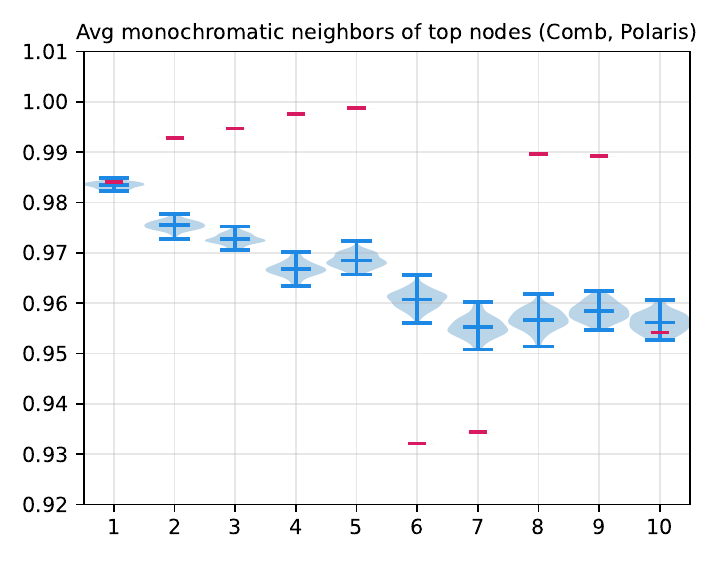} 
\end{subfigure}
\begin{subfigure}{.275\textwidth}
  \centering
  \includegraphics[width=\textwidth]{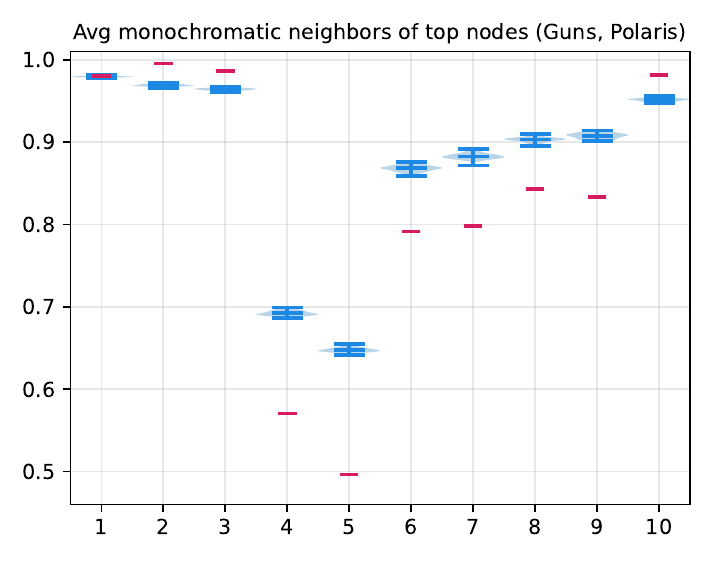} 
\end{subfigure}
\begin{subfigure}{.275\textwidth}
  \centering
  \includegraphics[width=\textwidth]{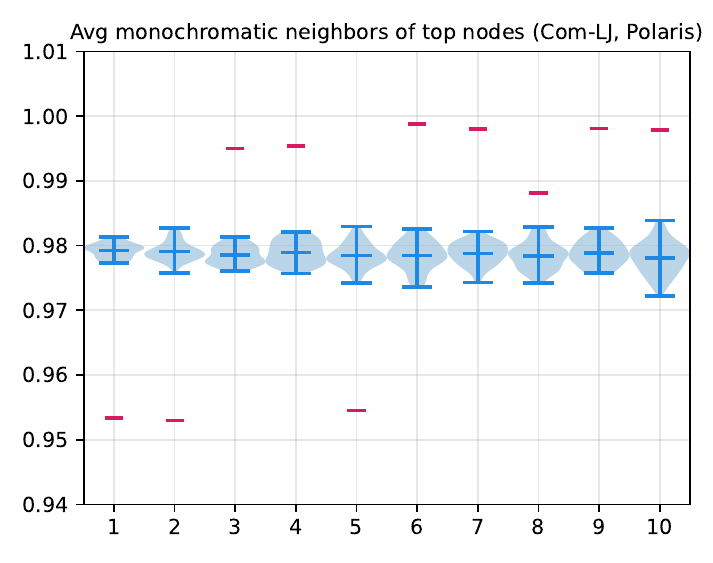} 
\end{subfigure}
\begin{subfigure}{.275\textwidth}
  \centering
  \includegraphics[width=\textwidth]{./figures/avg-col-stats-neigh-Com-Orkut-Polaris.pdf} 
\end{subfigure}
\caption{ 
Testing the significance of $M_v(G)$ of the $10$ nodes with highest degree under the Polaris model.
}
\Description{
Testing the significance of $M_v(G)$ of the $10$ nodes with highest degree under the Polaris model.}
\label{fig:avgcolpolarisappendix}
\end{figure*}
\fi

\ifextversion
\begin{figure*}[ht]
\begin{subfigure}{.175\textwidth}
  \centering
  \includegraphics[width=\textwidth]{./figures/power-legend.pdf}
\end{subfigure} \\
\begin{subfigure}{.275\textwidth}
  \centering
  \includegraphics[width=\textwidth]{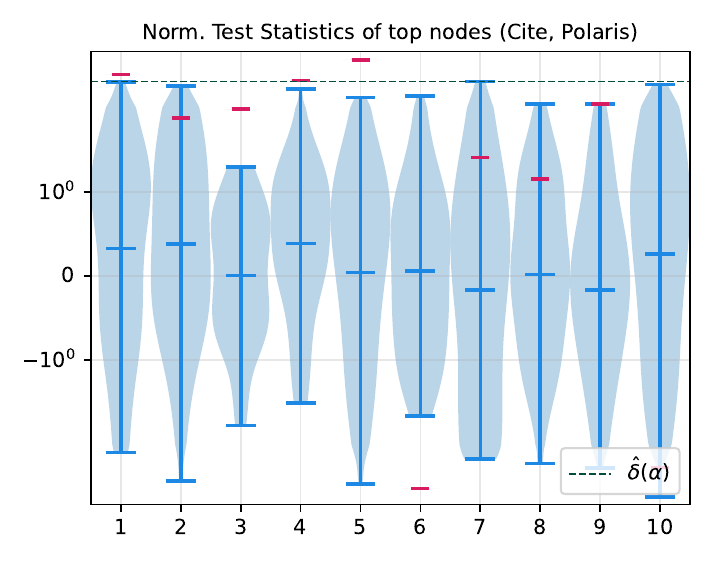} 
\end{subfigure}
\begin{subfigure}{.275\textwidth}
  \centering
  \includegraphics[width=\textwidth]{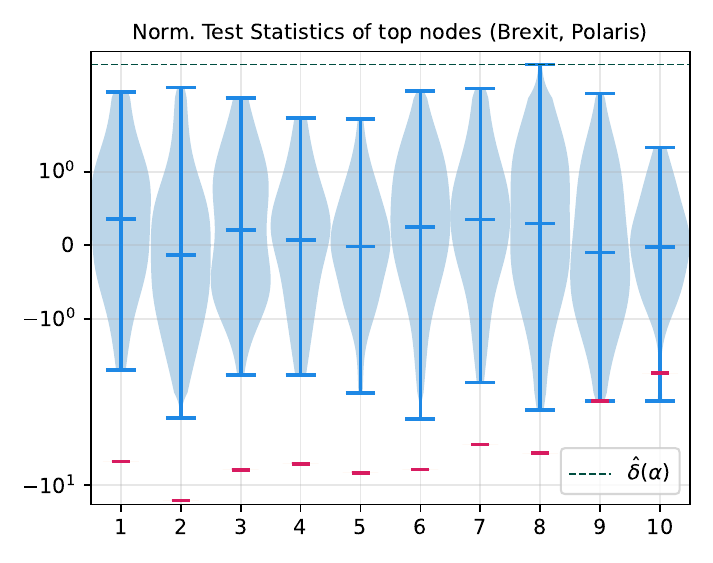} 
\end{subfigure}
\begin{subfigure}{.275\textwidth}
  \centering
  \includegraphics[width=\textwidth]{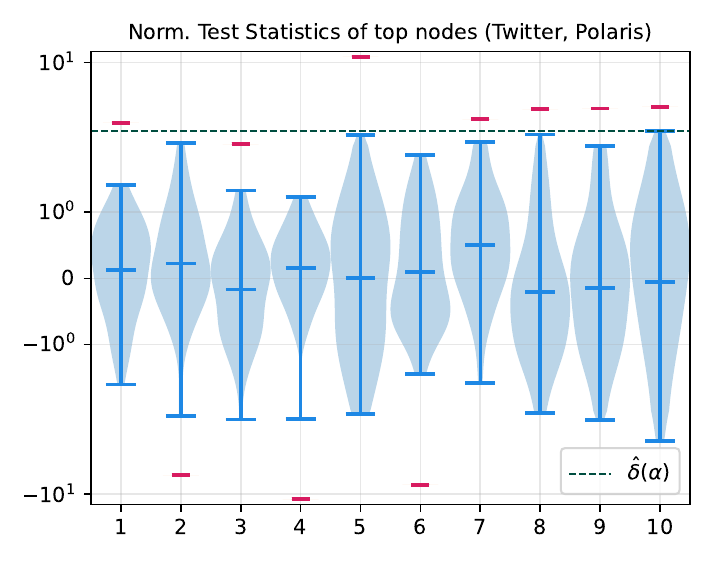} 
\end{subfigure}
\begin{subfigure}{.275\textwidth}
  \centering
  \includegraphics[width=\textwidth]{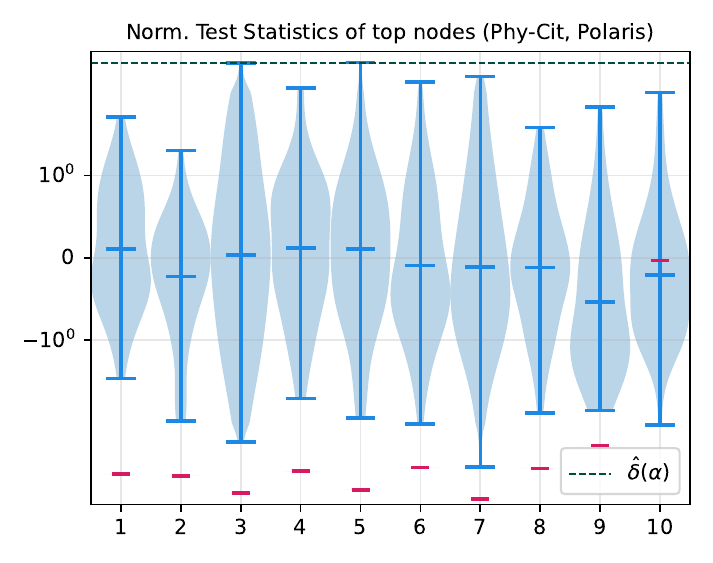} 
\end{subfigure}
\begin{subfigure}{.275\textwidth}
  \centering
  \includegraphics[width=\textwidth]{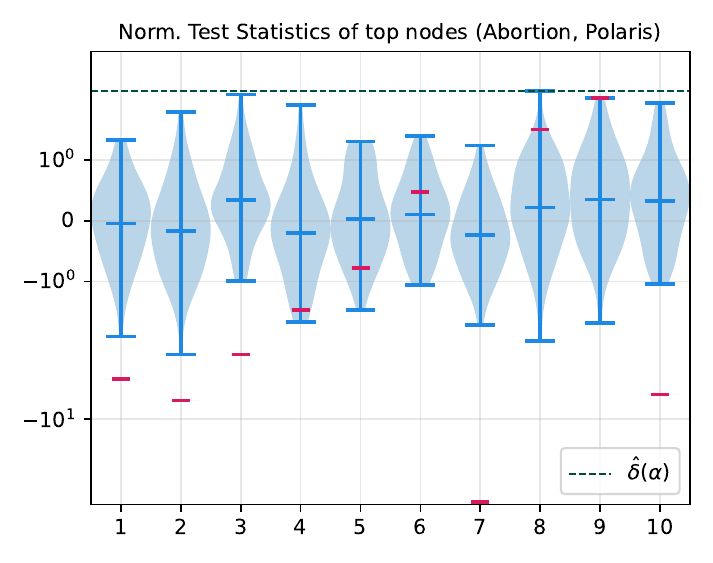} 
\end{subfigure}
\begin{subfigure}{.275\textwidth}
  \centering
  \includegraphics[width=\textwidth]{./figures/avg-col-stats-neigh-US-Elect-norm-Polaris.pdf} 
\end{subfigure}
\begin{subfigure}{.275\textwidth}
  \centering
  \includegraphics[width=\textwidth]{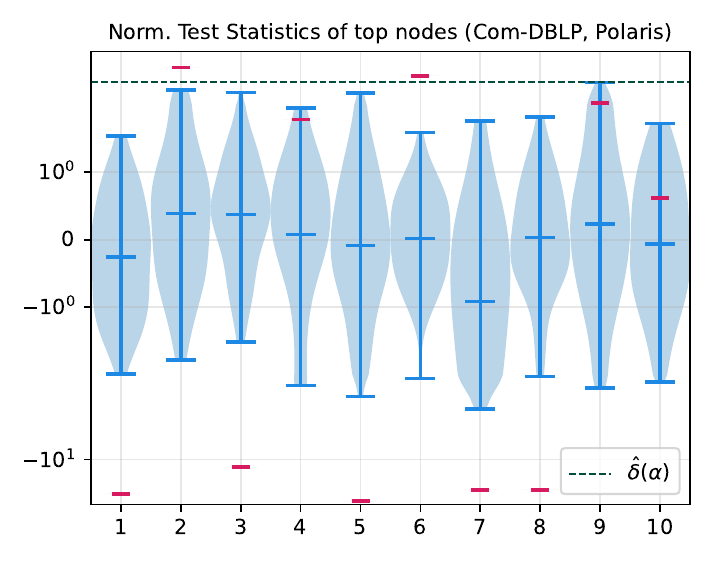} 
\end{subfigure}
\begin{subfigure}{.275\textwidth}
  \centering
  \includegraphics[width=\textwidth]{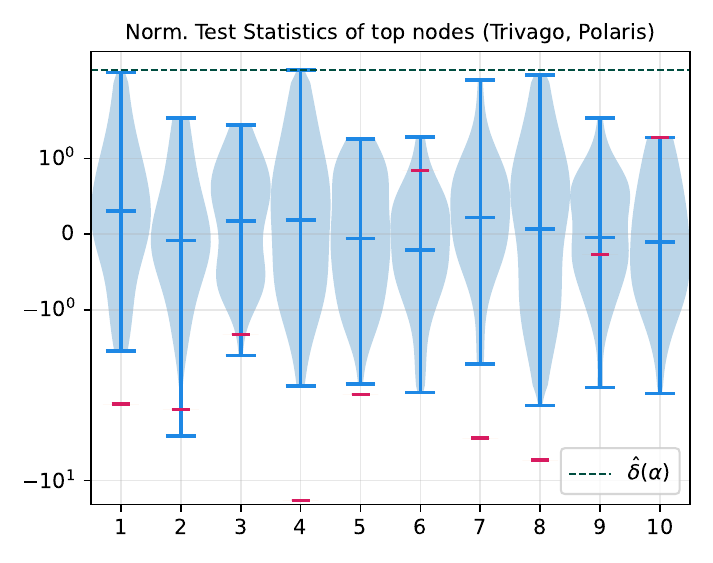} 
\end{subfigure}
\begin{subfigure}{.275\textwidth}
  \centering
  \includegraphics[width=\textwidth]{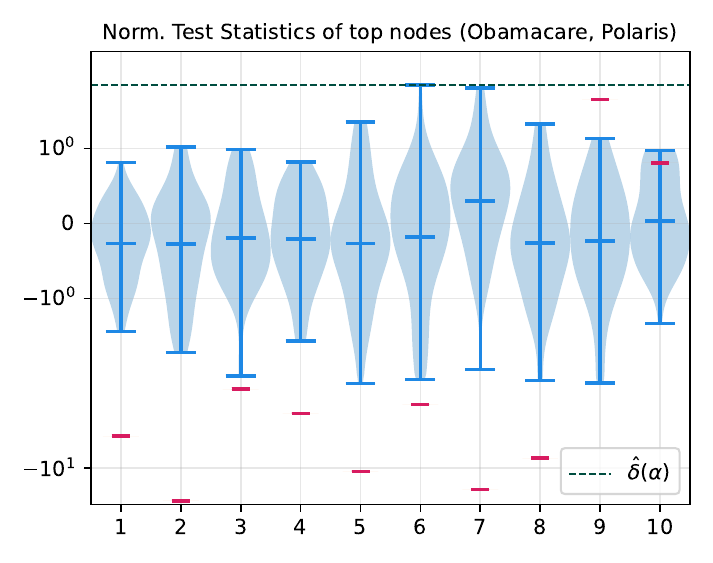} 
\end{subfigure}
\begin{subfigure}{.275\textwidth}
  \centering
  \includegraphics[width=\textwidth]{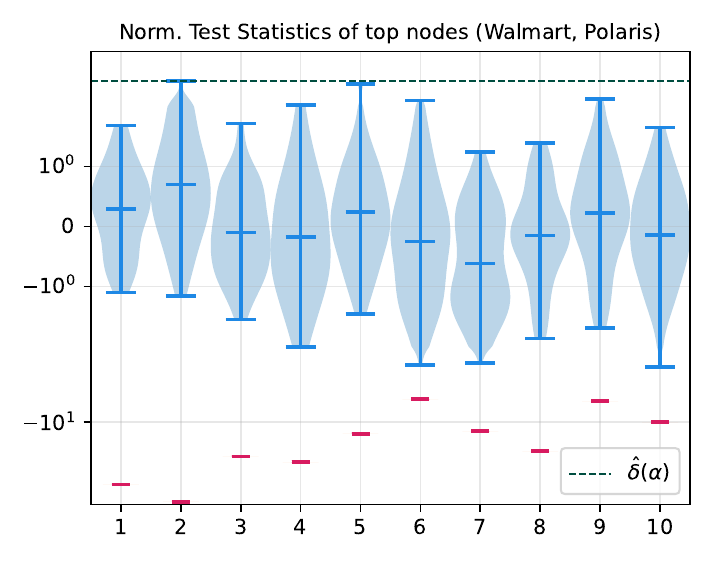} 
\end{subfigure}
\begin{subfigure}{.275\textwidth}
  \centering
  \includegraphics[width=\textwidth]{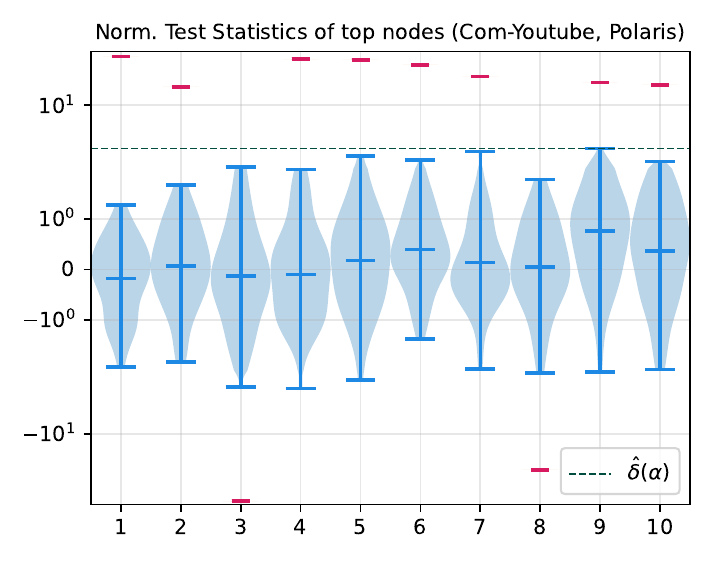} 
\end{subfigure}
\begin{subfigure}{.275\textwidth}
  \centering
  \includegraphics[width=\textwidth]{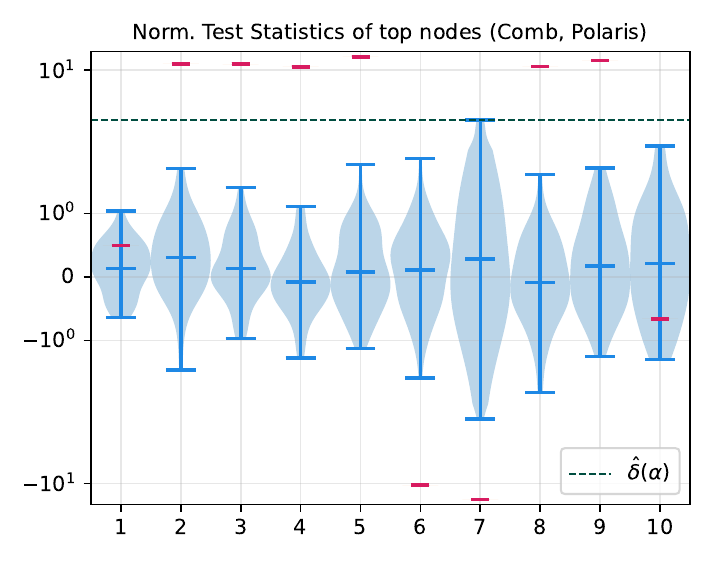} 
\end{subfigure}
\begin{subfigure}{.275\textwidth}
  \centering
  \includegraphics[width=\textwidth]{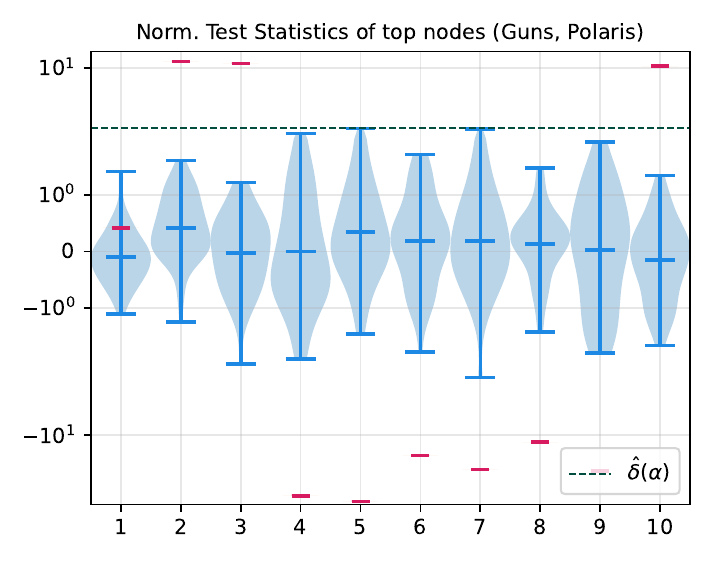} 
\end{subfigure}
\begin{subfigure}{.275\textwidth}
  \centering
  \includegraphics[width=\textwidth]{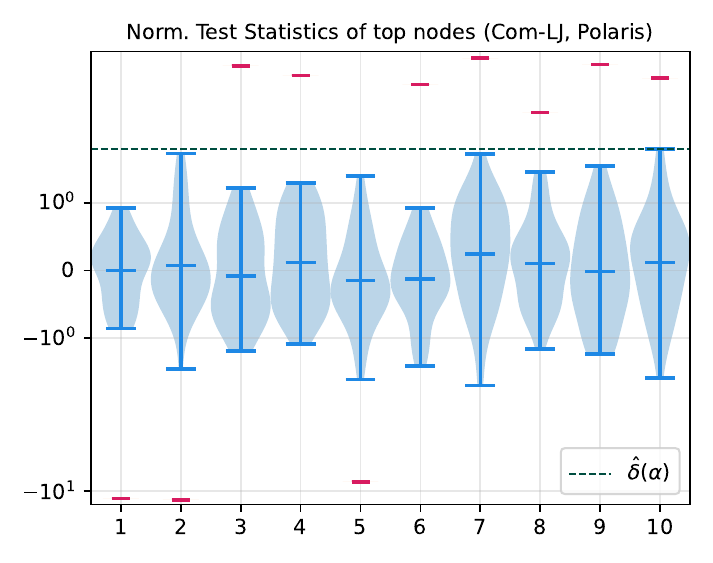} 
\end{subfigure}
\begin{subfigure}{.275\textwidth}
  \centering
  \includegraphics[width=\textwidth]{./figures/avg-col-stats-neigh-Com-Orkut-norm-Polaris.pdf} 
\end{subfigure}
\caption{ 
Testing the significance of $M_v(G)$ of the $10$ nodes with highest degree under the Polaris model.
}
\Description{
Testing the significance of $M_v(G)$ of the $10$ nodes with highest degree under the Polaris model.}
\label{fig:avgcolpolarisnormappendix}
\end{figure*}
\fi

\ifextversion
\begin{figure*}[ht]
\begin{subfigure}{.14\textwidth}
  \centering
  \includegraphics[width=\textwidth]{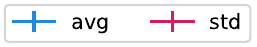}
\end{subfigure} \\
\begin{subfigure}{.275\textwidth}
  \centering
  \includegraphics[width=\textwidth]{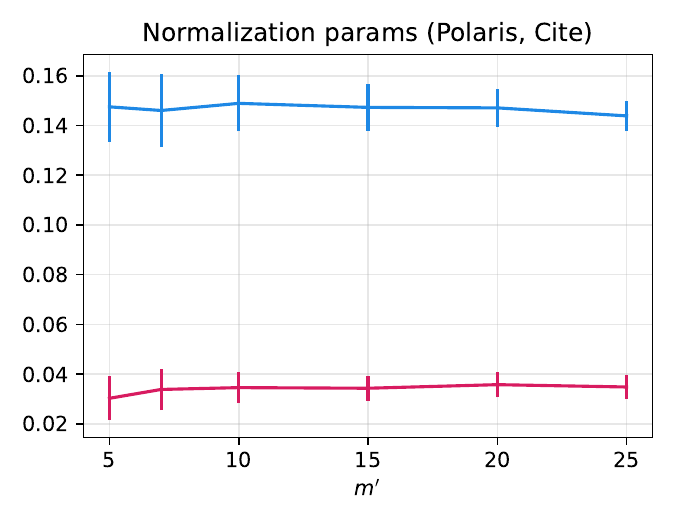} 
\end{subfigure}
\begin{subfigure}{.275\textwidth}
  \centering
  \includegraphics[width=\textwidth]{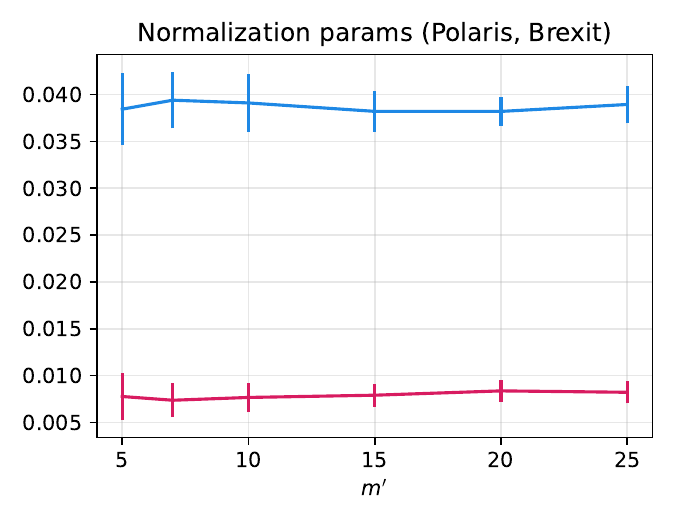} 
\end{subfigure}
\begin{subfigure}{.275\textwidth}
  \centering
  \includegraphics[width=\textwidth]{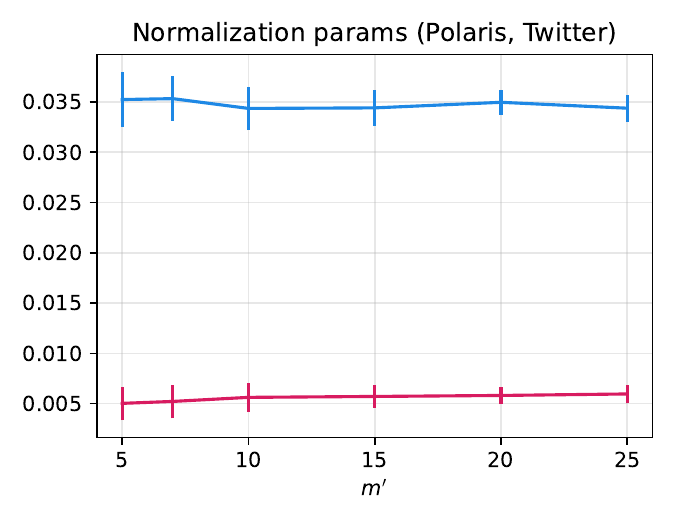} 
\end{subfigure}
\begin{subfigure}{.275\textwidth}
  \centering
  \includegraphics[width=\textwidth]{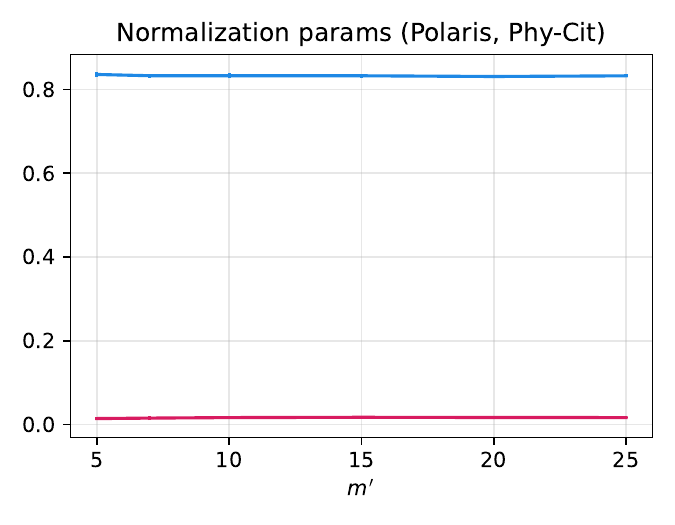} 
\end{subfigure}
\begin{subfigure}{.275\textwidth}
  \centering
  \includegraphics[width=\textwidth]{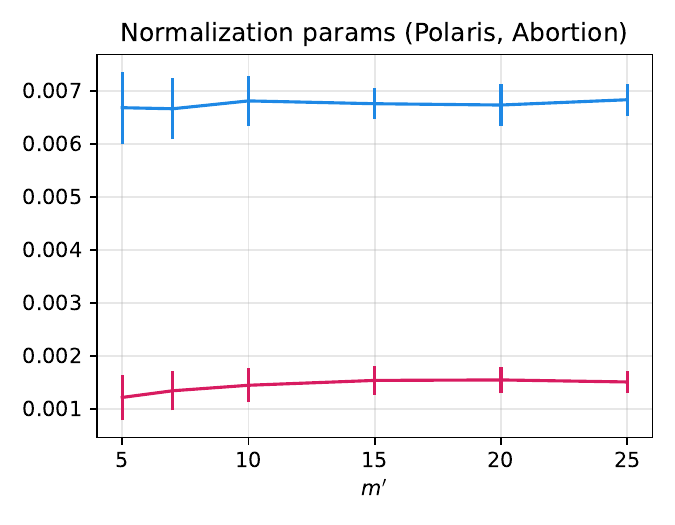} 
\end{subfigure}
\begin{subfigure}{.275\textwidth}
  \centering
  \includegraphics[width=\textwidth]{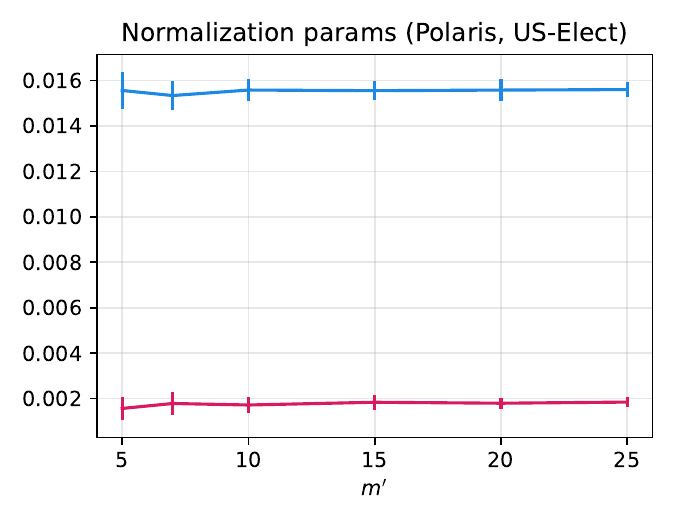} 
\end{subfigure}
\begin{subfigure}{.275\textwidth}
  \centering
  \includegraphics[width=\textwidth]{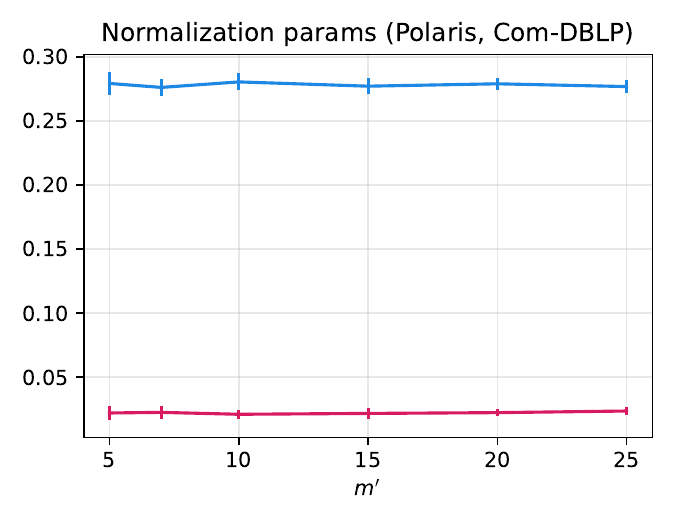} 
\end{subfigure}
\begin{subfigure}{.275\textwidth}
  \centering
  \includegraphics[width=\textwidth]{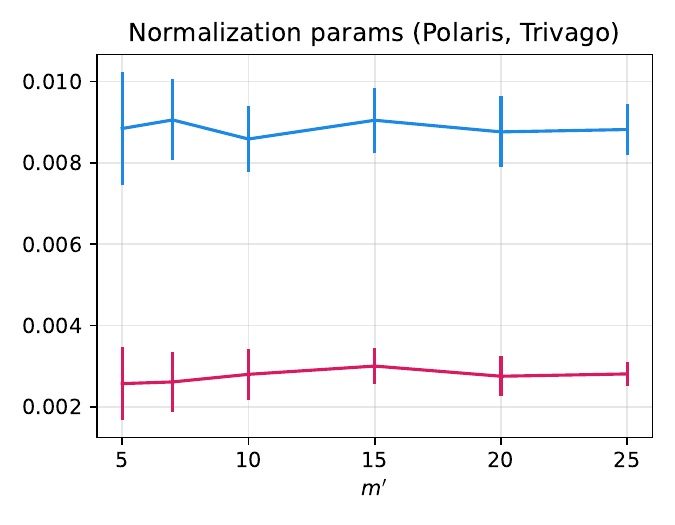} 
\end{subfigure}
\begin{subfigure}{.275\textwidth}
  \centering
  \includegraphics[width=\textwidth]{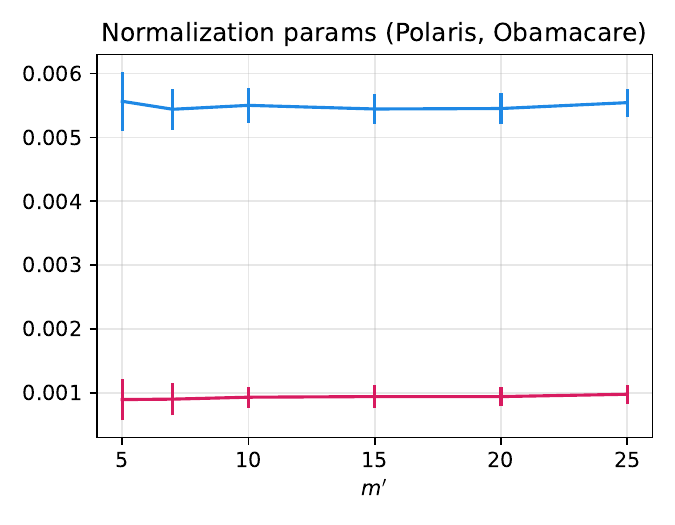} 
\end{subfigure}
\begin{subfigure}{.275\textwidth}
  \centering
  \includegraphics[width=\textwidth]{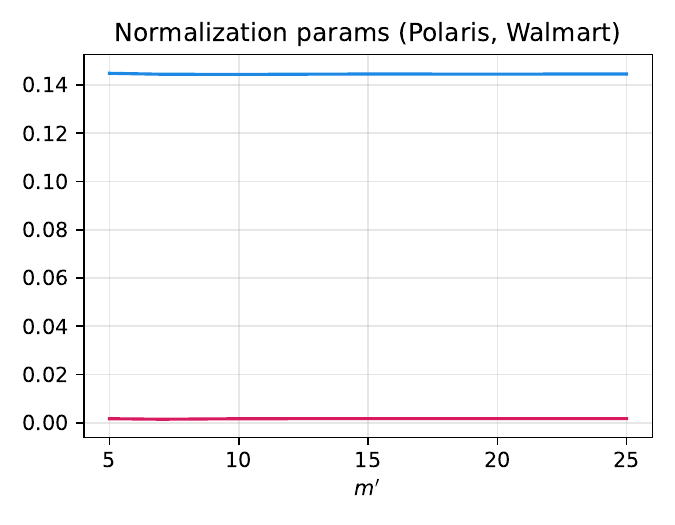} 
\end{subfigure}
\begin{subfigure}{.275\textwidth}
  \centering
  \includegraphics[width=\textwidth]{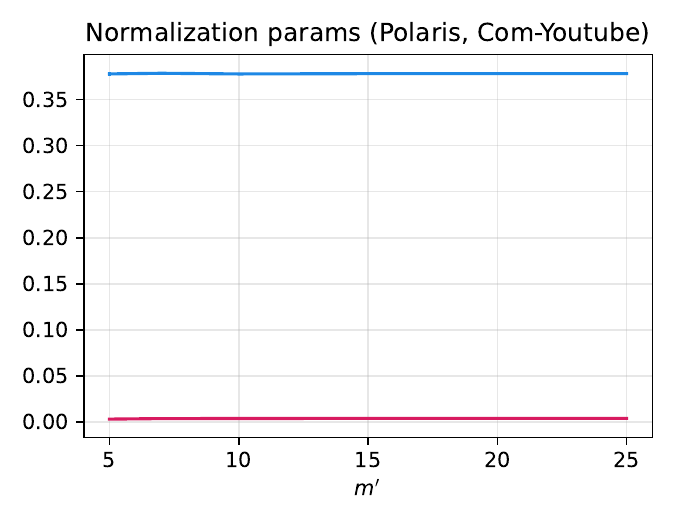} 
\end{subfigure}
\begin{subfigure}{.275\textwidth}
  \centering
  \includegraphics[width=\textwidth]{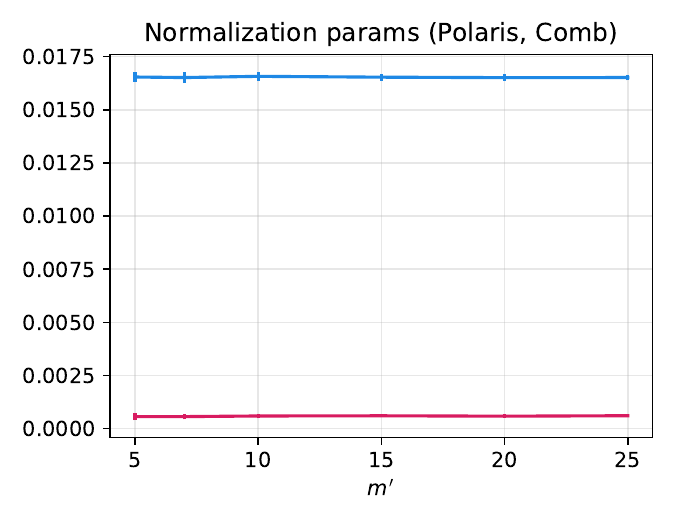} 
\end{subfigure}
\begin{subfigure}{.275\textwidth}
  \centering
  \includegraphics[width=\textwidth]{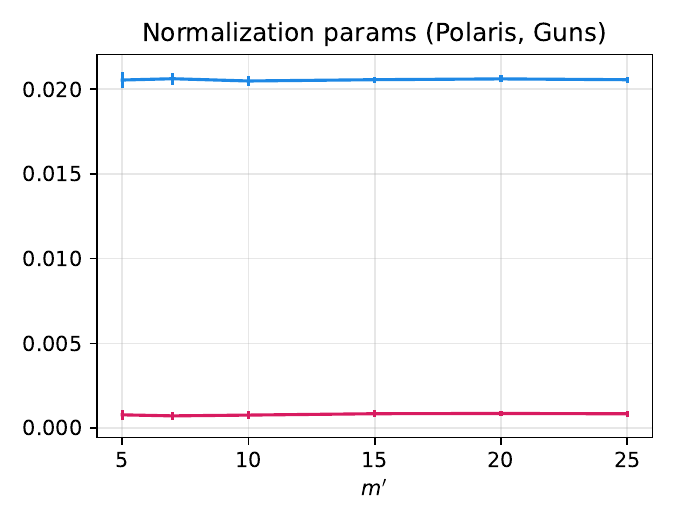} 
\end{subfigure}
\begin{subfigure}{.275\textwidth}
  \centering
  \includegraphics[width=\textwidth]{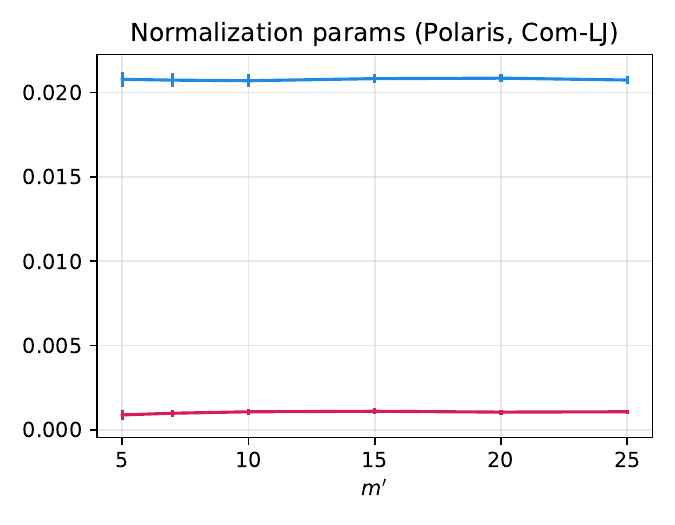} 
\end{subfigure}
\begin{subfigure}{.275\textwidth}
  \centering
  \includegraphics[width=\textwidth]{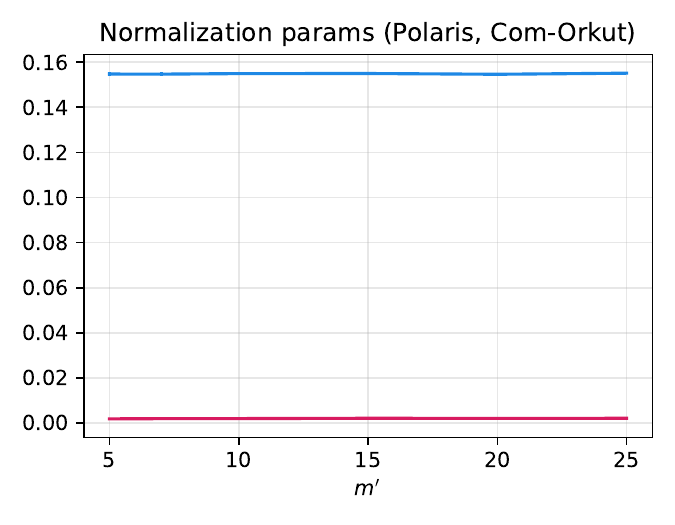} 
\end{subfigure}
\caption{ 
Normalization parameters of the function $M_v$ for the node $v$ with highest degree, under the Polaris model.
}
\Description{
Normalization parameters of the function $M_v$ for the node $v$ with highest degree, under the Polaris model.}
\label{fig:normalizationparameterspolappendix}
\end{figure*}
\fi

\ifextversion
\begin{figure*}[ht]
\begin{subfigure}{.175\textwidth}
  \centering
  \includegraphics[width=\textwidth]{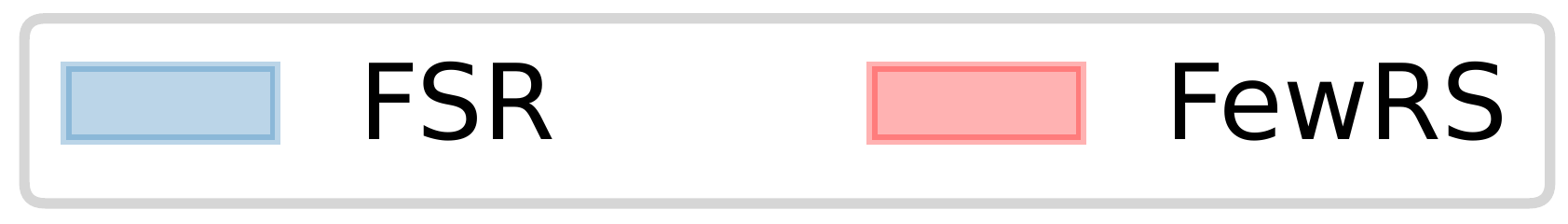}
\end{subfigure} \\
\begin{subfigure}{.28\textwidth}
  \centering
  \includegraphics[width=\textwidth]{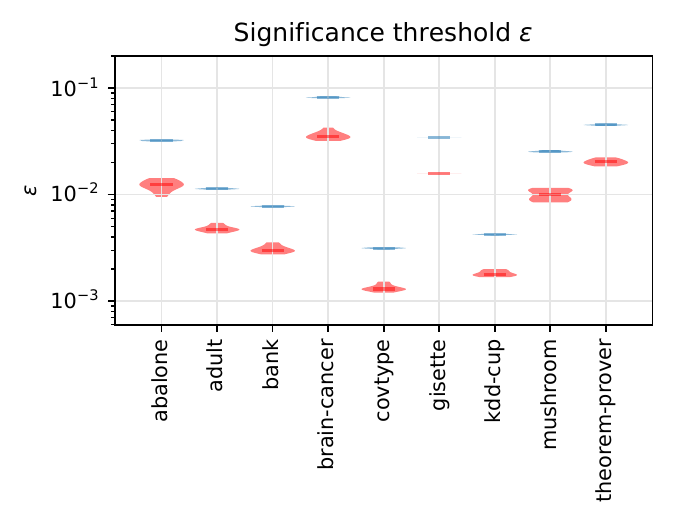} 
  %\caption{}
\end{subfigure}
\begin{subfigure}{.28\textwidth}
  \centering
  \includegraphics[width=\textwidth]{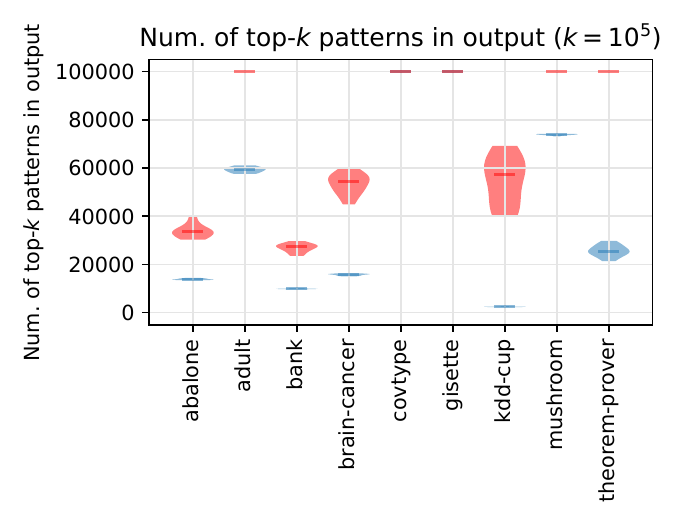} 
    %\caption{}
\end{subfigure}
\begin{subfigure}{.28\textwidth}
  \centering
  \includegraphics[width=\textwidth]{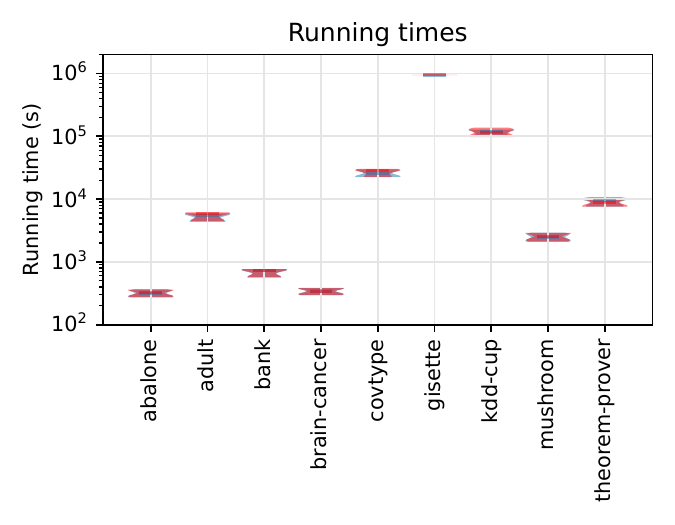} 
    %\caption{}
\end{subfigure}
\caption{ 
Comparison of \algname\ with FSR in terms of significance threshold, number of significant patterns reported in output, and running time. 
}
\Description{
Comparison of \algname\ with FSR in terms of significance threshold, number of significant patterns reported in output, and running time. }
\label{fig:comparisonfsrlarge}
\end{figure*}

\else

\begin{figure}[ht]
\begin{subfigure}{.16\textwidth}
  \centering
  \includegraphics[width=\textwidth]{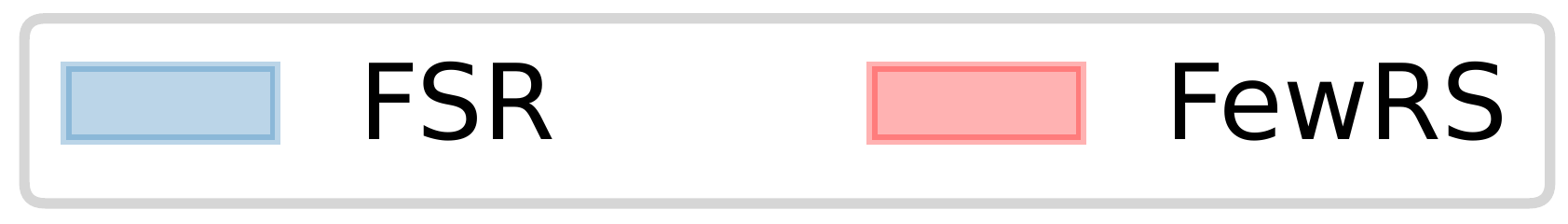}
\end{subfigure} \\
\begin{subfigure}{.33\textwidth}
  \centering
  \includegraphics[width=\textwidth]{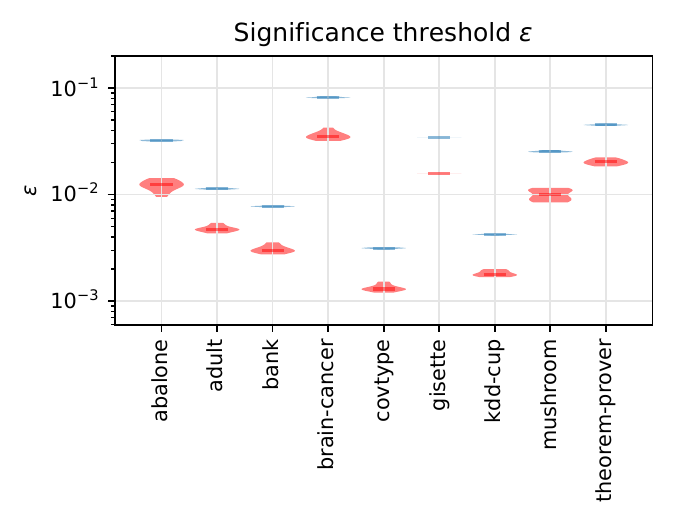} 
  %\caption{}
\end{subfigure}
\begin{subfigure}{.33\textwidth}
  \centering
  \includegraphics[width=\textwidth]{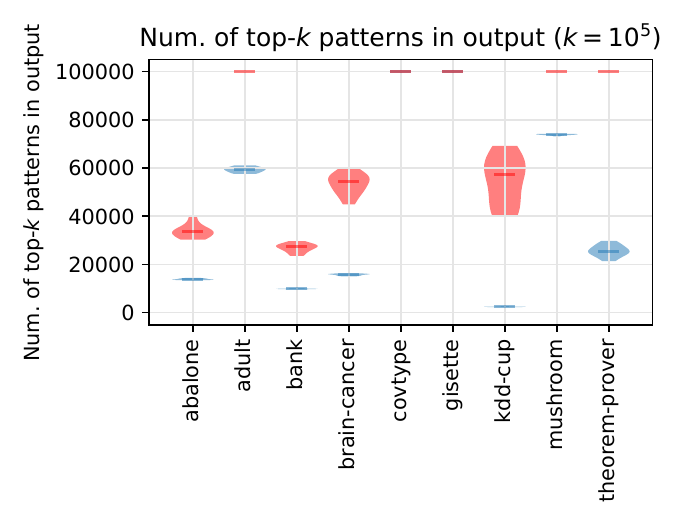} 
    %\caption{}
\end{subfigure}
\caption{ 
Comparison of \algname\ with FSR in terms of significance threshold and number of significant patterns reported in output (running time are equal for both methods, see~\cite{fewrsextended}). 
}
\Description{
Comparison of \algname\ with FSR in terms of significance threshold and number of significant patterns reported in output (running time are equal for both methods, see~\cite{fewrsextended}). }
\label{fig:comparisonfsrsmall}
\end{figure}
\fi